\documentclass[journal]{IEEEtran}

\usepackage{array}
\usepackage{caption}
\usepackage{adjustbox}
\usepackage{makecell}
\usepackage{amsmath,amssymb,amsthm}
\usepackage{algorithm}
\usepackage{algorithmic}
\usepackage{tcolorbox}
\usepackage{graphicx}
\usepackage{booktabs}
\usepackage{multirow}
\usepackage{subfigure}
\usepackage{xcolor}
\usepackage{hyperref}
\usepackage{cite}
\usepackage{url}

\newtheorem{theorem}{Theorem}[section]
\newtheorem{lemma}[theorem]{Lemma}
\newtheorem{proposition}[theorem]{Proposition}

\newtheorem*{remark}{Remark}
\newtheorem{assumption}{Assumption}

\definecolor{DDPG}{HTML}{FFFF00}
\definecolor{DDPGHEARS}{HTML}{E63946}
\definecolor{PPO}{HTML}{2A9D8F}
\definecolor{PPOHEARS}{HTML}{3498DB}
\definecolor{TD3}{HTML}{9B59B6}
\definecolor{TD3HEARS}{HTML}{00FFFF}
\definecolor{SAC}{HTML}{FF1493}
\definecolor{SACHEARS}{HTML}{FFA500}

\definecolor{Vanilla}{HTML}{FF1493}
\definecolor{RegOnly}{HTML}{E63946}
\definecolor{EnergyOnly}{HTML}{2A9D8F}
\definecolor{Decombination}{HTML}{7FFF00}
\definecolor{WithoutTask}{HTML}{3498DB}
\definecolor{TaskOnly}{HTML}{9B59B6}
\definecolor{WithoutEnergy}{HTML}{00FFFF}
\definecolor{WithoutReg}{HTML}{8B4513}
\definecolor{HEARS}{HTML}{FFA500}

\definecolor{1e0}{HTML}{FF1493}
\definecolor{1e-1}{HTML}{E63946}
\definecolor{1e-2}{HTML}{2A9D8F}
\definecolor{1e-3}{HTML}{3498DB}
\definecolor{1e-4}{HTML}{9B59B6}
\definecolor{1e-5}{HTML}{00FFFF}

\begin{document}
	
	\title{Hybrid Energy-Aware Reward Shaping:\\A Unified Lightweight Physics-Guided Methodology for Policy Optimization}
	
	\author{Qijun Liao,
		Jue Yang,
		Yiting Kang,
		Xinxin Zhao,
		Yong Zhang
		and Mingan Zhao
		\thanks{Q. Liao is with the School of Mechanical Engineering, University of Science and Technology Beijing, Beijing 100083, China (e-mail: m202420759@xs.ustb.edu.cn)}
		\thanks{J. Yang, Y. Kang and X. Zhao are with the School of Mechanical Engineering, University of Science and Technology Beijing, Beijing 100083, China (e-mail: yangjue@ustb.edu.cn; kangyiting@ustb.edu.cn; zhaoxinxin@ustb.edu.cn)}
		\thanks{Y. Zhang and M. Zhao are with Jiangsu XCMG Construction Machinery Research Institute Co., Ltd., Jiangsu 221000, China (e-mail: zhangy1@xcmg.com; zhaomingan@xcmg.com)}
		\thanks{All correspondence should be sent to J. Yang with email: yangjue@ustb.edu.cn.}}
	
	\maketitle
	
	\begin{abstract}
		Deep reinforcement learning for continuous control often suffers from high variance, low energy efficiency, and poor generalization under distribution shift, as purely data-driven exploration ignores available physical structure. Physics-based approaches such as Lagrangian and Hamiltonian neural networks enforce physical consistency by learning unknown dynamics, but this learning entails substantial complexity when dynamics must be identified from data.
		
		This paper proposes Hybrid Energy-Aware Reward Shaping (H-EARS), which encodes dominant energy terms---assumed known a priori---directly as reward potentials at $O(n)$ per-step computation. H-EARS decomposes the shaping potential into task-oriented ($\Phi_{\text{task}}$) and energy-based ($\Phi_{\text{energy}}$) components, supplemented by an action regularization term that deliberately modifies the optimization objective to enforce energy-efficient control. A complete theoretical foundation is established: functional independence of shaping and regularization (Lemma~\ref{lem:functional_independence}), energy-based gradient enrichment under positive-definite Hessian conditions (Theorem~\ref{thm:energy_convergence}), convergence guarantees under function approximation (Theorem~\ref{thm:convergence}), and approximate potential error bounds (Lemma~\ref{lem:approx_potential}). Across four continuous control benchmarks and four baseline algorithms, H-EARS achieves consistent gains in convergence speed, policy stability, and final performance. High-fidelity vehicle simulations validate applicability in safety-critical settings under extreme road conditions. Code is publicly available at \url{https://github.com/QiJLiao/Hybrid_Energy-Aware_Reward_Shaping}
	\end{abstract}
	
	
	\section{Introduction}
	
	Deep reinforcement learning (DRL) has emerged as a powerful paradigm for continuous control, achieving impressive results in trajectory tracking and motion control tasks. Model-free algorithms such as Soft Actor-Critic (SAC)~\cite{haarnoja2018soft} can learn complex nonlinear policies through environment interaction, without relying on explicit system models. However, policies obtained through pure trial-and-error learning often suffer from three major issues: high variance, low energy efficiency, and poor generalization. Treating agents as tabula rasa necessitates rediscovering fundamental physical principles through extensive exploration, leading to unstable, energy-inefficient, and physically implausible control behaviors. Consequently, learned policies frequently exploit simulator-specific dynamics instead of discovering generalizable control strategies, resulting in severe performance degradation under out-of-distribution conditions~\cite{cobbe2019quantifying}.
	
	Recent studies have introduced physical priors into DRL to bridge the gap between simulation success and real-world reliability. Nevertheless, existing approaches face an inherent trade-off. Physics-based methods such as Lagrangian~\cite{lutter2019deep} and Hamiltonian~\cite{greydanus2019hamiltonian} neural networks guarantee physical consistency by learning complete system dynamics from data, incurring $O(n^3)$--$O(n^2)$ training complexity in the number of generalized coordinates $n$. Crucially, these methods address a fundamentally different subproblem: they reconstruct unknown dynamics from observations. H-EARS, by contrast, assumes dominant energy terms are known a priori and encodes them as lightweight per-step potentials with $O(n)$ computation cost. This distinction means the two approaches are complementary: Lagrangian methods are preferred when dynamics are completely unknown; H-EARS applies when partial physical knowledge is available and deployment efficiency is required. In contrast, purely model-free methods are computationally efficient but lack physical constraints, often learning "shortcut" policies that fail when real-world dynamics deviate from training conditions~\cite{zhang2021learning}. This dilemma is particularly critical in safety-sensitive domains, such as vehicle dynamics control, where stability and robustness are essential.
	
	Reward shaping enhances sparse reward signals to accelerate reinforcement learning (RL). Ng et al.~\cite{ng1999policy} established the Potential-Based Reward Shaping (PBRS) framework, proving that a shaped reward $F(s,a,s')=\gamma\Phi(s')-\Phi(s)$ preserves policy optimality. Recent work by Ding et al.~\cite{ding2023magnetic} introduced magnetic field-based reward shaping for goal-conditioned RL, demonstrating that domain-specific potential functions can significantly improve sample efficiency in dynamic environments. This formulation is equivalent to initializing the value function as $V^*(s)=V^*_0(s)+\Phi(s)$. Wiewiora et al.~\cite{wiewiora2003principled} extended PBRS to state-action potentials $\Phi(s,a)$, while Devlin and Kudenko~\cite{devlin2012dynamic} introduced dynamic potentials $\Phi_t(s)$. Subsequent works pursued automated potential learning, including meta-learning~\cite{zou2021reward}, apprenticeship learning via inverse reinforcement learning~\cite{abbeel2004apprenticeship}, and model-based potential generation using Markov approximations~\cite{grzes2010learning}. Harutyunyan et al.~\cite{harutyunyan2015expressing} generalized PBRS to represent any reward in potential-based form, greatly broadening its applicability. However, most PBRS studies focus primarily on convergence acceleration, with limited attention to physical interpretability or stability. Relative to the foundational PBRS result of Ng et al.~\cite{ng1999policy}, which establishes policy invariance under reward shaping but provides no guidance on how to design physically meaningful potential functions, H-EARS makes three distinct advances: (i)~it derives potential functions directly from system mechanical energy, providing a principled design methodology rather than task-specific heuristics; (ii)~it introduces an action regularization component that deliberately modifies the optimization objective to enforce energy-efficient control---a conscious departure from strict policy invariance in favor of physically consistent solutions; and (iii)~it establishes formal error bounds for approximate energy models (Lemma~\ref{lem:approx_potential}), enabling engineering deployment without requiring complete dynamics knowledge. The convergence acceleration mechanism (Theorem~\ref{thm:energy_convergence}) differs from generic PBRS in that it exploits the structural property $\frac{\partial^2 E}{\partial q^2} \succ 0$, guaranteeing gradient informativeness throughout the state space rather than only near goal states.
	
	While works such as Brys et al.~\cite{brys2015} improved sample efficiency, they neglected the intrinsic physical structure of control systems. Existing PBRS schemes act solely on the state space, unable to constrain action-level behavior. As a result, even optimal policies $\pi^*$ may produce physically implausible, high-frequency control oscillations. In contrast, the proposed H-EARS framework encodes physical guidance—specifically energy minimization—into the reward structure, jointly achieving faster convergence and physically reasonable policies through a unified theoretical foundation.
	
	Integrating physical priors into learning models has emerged as an effective strategy to enhance sample efficiency and generalization. Lutter et al.~\cite{lutter2019deep} and Greydanus et al.~\cite{greydanus2019hamiltonian} proposed Lagrangian and Hamiltonian neural networks that guarantee physical consistency but rely on complete system equations and high computational cost. Cranmer et al.~\cite{cranmer2020lagrangian} achieved accurate Lagrangian dynamics modeling through specialized neural architectures, though still requiring explicit structural priors. Interaction networks~\cite{battaglia2016interaction} and graph-based physical simulators~\cite{sanchez2020learning} capture inter-object relations and complex physics but demand large datasets and substantial computation. Neural ODEs~\cite{chen2018neural} model temporal dynamics elegantly yet remain limited by numerical stability and cost. These works reveal a persistent trade-off between precision, efficiency, and applicability, motivating a more flexible design philosophy.
	
	Beyond structural modeling, energy efficiency has emerged as a critical objective in cyber-physical systems. Liu et al.~\cite{liu2020parallel} proposed a parallel reinforcement learning framework for energy efficiency in hybrid electric vehicles. While demonstrating RL's capability in energy management, their approach lacks theoretical stability guarantees. H-EARS addresses this by explicitly incorporating energy minimization with provable convergence properties.
	
	\begin{table*}[!t]
		\captionsetup{font=footnotesize, labelfont=footnotesize}
		\caption{Characteristic Comparison of Different Physics-Guided Methods$^\ddagger$}
		\label{tab:physics_methods_comparison}
		\centering
		\footnotesize
		\begin{tabular}{lcccc}
			\toprule
			\textbf{Method Type} & \textbf{Complexity} & \textbf{Known Structure?} & \textbf{Generalization} & \textbf{Application Scenario} \\
			\midrule
			Lagrangian RL~\cite{lutter2019deep}       & $O(n^3)$ & No (learns from data) & Low    & Unknown-dynamics systems \\
			Hamiltonian RL~\cite{greydanus2019hamiltonian} & $O(n^2)$ & No (learns from data) & Low    & Unknown conservative systems \\
			Graph Networks~\cite{sanchez2020learning} & $O(n^2)$ & No (learns from data) & Medium & Multi-body simulation \\
			Neural ODE~\cite{chen2018neural}          & $O(n^2)$ & No (learns from data) & Medium & Temporal dynamics modeling \\
			Pure Model-free RL                        & $O(1)$   & No                    & High   & Large data available \\
			H-EARS (Ours)                             & $O(n)$   & Partial (dominant terms) & High & Rapid prototyping + deployment \\
			\bottomrule
		\end{tabular}
		\vspace{1mm}
		
		\footnotesize{$^\ddagger$ \textit{Complexity note}: Lagrangian/Hamiltonian methods address settings where dynamics are unknown and must be learned, incurring training-time complexity in $n$ (generalized coordinates). H-EARS addresses a complementary setting: dominant energy terms are assumed known a priori, requiring only $O(n)$ per-step feature computation. These complexities measure fundamentally different operations and represent complementary rather than competing approaches.}
	\end{table*}
	
	From an engineering standpoint, modeling cost is a crucial factor. Lagrangian-based methods require experts to derive complete Euler–Lagrange equations, often taking weeks for complex systems. In contrast, H-EARS only models dominant energy terms—e.g., torso and limb kinetic energy plus gravitational potential—allowing general engineers to complete modeling within days. When system configurations change, precise models must be rederived, whereas H-EARS merely adjusts energy coefficients, significantly reducing maintenance effort.
	
	Continuous control remains a central RL application. Algorithms such as SAC~\cite{haarnoja2018soft}, TD3~\cite{fujimoto2018addressing}, and PPO~\cite{schulman2017proximal} have achieved state-of-the-art performance, yet often exhibit high variance and low energy efficiency. Haarnoja et al.~\cite{haarnoja2018applications} noted that SAC policies, though effective, lack physical plausibility, while empirical studies~\cite{cobbe2019quantifying} have shown that standard RL frequently exhibits brittleness when real-world dynamics deviate from training conditions. Zhang et al.~\cite{zhang2021learning} further demonstrated that model-free RL tends to learn “shortcut” strategies exploiting simulator artifacts. H-EARS directly addresses these issues by embedding physics-aware reward shaping without modifying core algorithmic structures.
	
	Learning-based vehicle control represents another line of research where physics integration is critical. Ma et al.~\cite{ma2020artificial} provided a comprehensive survey on AI applications in autonomous vehicles, highlighting the growing adoption of deep learning and reinforcement learning for perception and control. Building on this foundation, recent works~\cite{peng2021end} have applied DRL to various vehicle control tasks including lateral control and trajectory planning. However, as noted in the survey~\cite{ma2020artificial}, a major challenge remains in ensuring robust and physically consistent control under extreme conditions. Hewing et al.~\cite{hewing2020learning} comprehensively reviewed learning-based MPC methods, emphasizing safety and real-time challenges. Spielberg et al.~\cite{spielberg2019neural} demonstrated neural network-based vehicle models for high-performance control, while Rosolia and Borrelli~\cite{rosolia2017learning} and Liniger et al.~\cite{liniger2015optimization} explored learning-based and optimization-based control, respectively. Despite progress, these methods often rely on accurate models and face stability or generalization issues, as highlighted by Kiran et al.~\cite{kiran2022deep}. H-EARS complements these works through an RL+MPC hierarchical architecture: the upper H-EARS controller ensures physically consistent policy generation, and the lower MPC guarantees safety via constraint enforcement.
	
	In summary, current research faces three central challenges: (1)~achieving an optimal balance among modeling cost, computational complexity, and performance; (2)~unifying convergence acceleration and physical stability within a mathematically consistent framework; (3)~mitigating overfitting of policies to simulation dynamics. The proposed H-EARS framework addresses these through the following core contributions, each mapped to the challenge it resolves: (1) Lightweight energy modeling: H-EARS selectively captures dominant energy components at $O(n)$ per-step cost, requiring no complete dynamics derivation. This enables deployment by general engineers without expertise in analytical mechanics (Section~\ref{sec:theoretical_framework}, Appendix~\ref{app:standard_envs}); (2) Unified theoretical foundation: A complete theory is established including functional independence of shaping and regularization (Lemma~\ref{lem:functional_independence}), energy-based gradient enrichment (Theorem~\ref{thm:energy_convergence}), convergence guarantees under function approximation (Theorem~\ref{thm:convergence}), and approximate potential error bounds (Lemma~\ref{lem:approx_potential}); (3) Physical plausibility under challenging operating conditions: By embedding energy minimization as a structural prior, H-EARS constrains policies toward physically consistent behaviors that maintain stability across the full range of training conditions including extreme cases, validated under compound-slope and low-adhesion ($\mu$ down to 0.1) road conditions in TruckSim (Section~\ref{sec:vehicle}).
	
	\section{Methodology}
	\label{sec:theoretical_framework}
	
	This section develops the mathematical foundations of H-EARS in five steps:
	\begin{enumerate}
		\item Framework definition (Section~\ref{sec:theoretical_framework}): H-EARS is defined as a two-step MDP transformation---action regularization first defines $\mathcal{M}_\lambda$, then potential-based shaping is applied to $\mathcal{M}_\lambda$. This formulation makes explicit that $\lambda > 0$ intentionally modifies the optimal policy, trading task performance for energy efficiency.
		\item Functional independence (Lemma~\ref{lem:functional_independence}): the potential shaping and action regularization components act on disjoint domains ($\mathcal{S}$ vs.\ $\mathcal{A}$), making $\alpha_{\text{task}}$, $\alpha_{\text{energy}}$, and $\lambda$ independently tunable design parameters---adjusting one does not structurally invalidate the contribution of the others.
		\item Gradient enrichment mechanism (Theorem~\ref{thm:energy_convergence}): when the energy Hessian $\partial^2 E/\partial q^2 \succ 0$ holds over the task-relevant operating region $\mathcal{S}_{\text{op}}$, the energy potential supplies informative gradient signal throughout the state space, accelerating convergence where task rewards are sparse.
		\item Convergence and approximation guarantees (Theorem~\ref{thm:convergence}, Lemma~\ref{lem:approx_potential}): an $O(1/\sqrt{N})$ convergence rate is established under function approximation, and a performance-loss bound for approximate energy models quantifies the modeling-effort trade-off.
		\item Stability-oriented guidance and practical guidelines (Proposition~\ref{prop:lyapunov_heuristic}, Theorem~\ref{thm:parameter_continuity}): the Lyapunov-inspired interpretation is clarified as a heuristic design guide, and parameter continuity bounds are used to derive practical scheduling heuristics.
	\end{enumerate}
	Table~\ref{tab:notation} summarizes the principal symbols. Where the same letter appears in multiple roles (e.g., $\lambda$ as a GAE discount factor in PPO and as the H-EARS regularization coefficient), context disambiguates; H-EARS-specific symbols always refer to the definitions below.
	
	\begin{table*}[!t]
		\captionsetup{font=footnotesize, labelfont=footnotesize}
		\caption{Principal Notation}
		\label{tab:notation}
		\centering
		\footnotesize
		\begin{tabular}{ll}
			\toprule
			\textbf{Symbol} & \textbf{Definition} \\
			\midrule
			$\mathcal{M}=(\mathcal{S},\mathcal{A},\mathcal{P},R,\gamma)$ & Original MDP \\
			$\mathcal{M}_\lambda$ & Regularized MDP: reward $R_\lambda = R - \lambda\mathcal{E}(a)$ \\
			$J(\pi)$, $J_{\mathcal{M}}(\pi)$ & Cumulative return; $J_{\mathcal{M}}$ emphasizes the original reward $R$ \\
			$\Phi(s)$ & Dual potential: $\alpha_{\text{task}}\Phi_{\text{task}}(s)+\alpha_{\text{energy}}\Phi_{\text{energy}}(s)$ \\
			$\Phi_{\text{task}}(s)$ & Task-oriented potential (e.g., goal proximity, tracking error) \\
			$\Phi_{\text{energy}}(s)$ & Energy-based potential; $\Phi_{\text{energy}}(s) = -E(q,\dot{q})$ \\
			$\alpha_{\text{task}},\,\alpha_{\text{energy}}$ & Potential weighting coefficients \\
			$\mathcal{E}(a_t) = a_t^\top Q a_t$ & Per-step control energy functional ($Q \succeq 0$) \\
			$\lambda \geq 0$ & Action regularization coefficient (H-EARS context) \\
			$\Delta\Phi_{\text{energy},t}$ & $\gamma\Phi_{\text{energy}}(s_{t+1}) - \Phi_{\text{energy}}(s_t)$ \\
			$E(q,\dot{q}) = T(\dot{q})+U(q)$ & Total mechanical energy \\
			$q,\,\dot{q}$ & Generalized coordinates and velocities \\
			$\mathcal{S}_{\text{op}}$ & Operationally relevant state region (Theorem~\ref{thm:energy_convergence}) \\
			$\rho_{\text{info}}$ & Energy-gradient informativeness ratio (Eq.~\ref{eq:gradient_ratio}) \\
			$\delta,\,\epsilon_{\text{approx}}$ & Absolute and relative potential approximation error \\
			$\lambda_{\max}$ & Reward-boundedness threshold (Eq.~\ref{eq:lambda_bound}) \\
			\bottomrule
		\end{tabular}
	\end{table*}
	
	\subsection{Preliminaries and Framework Definition}
	
	\subsubsection{Markov Decision Process}
	An MDP is defined as $\mathcal{M}=(\mathcal{S},\mathcal{A},\mathcal{P},R,\gamma)$, where $\mathcal{S}$ is the state space, $\mathcal{A}$ is the action space, $\mathcal{P}:\mathcal{S}\times\mathcal{A}\to\Delta(\mathcal{S})$ is the transition probability, $R:\mathcal{S}\times\mathcal{A}\times\mathcal{S}\to\mathbb{R}$ is the reward function, and $\gamma\in(0,1)$ is the discount factor. The objective is to learn a policy $\pi:\mathcal{S}\to\Delta(\mathcal{A})$ maximizing expected cumulative return:
	\begin{equation}
		J(\pi) = \mathbb{E}_{\tau\sim\pi}\left[\sum_{t=0}^{\infty}\gamma^t R(s_t,a_t,s_{t+1})\right]
	\end{equation}
	
	\subsubsection{Potential-Based Reward Shaping}
	Ng et al.~\cite{ng1999policy} established that shaped rewards of the form $\tilde{R}(s,a,s') = R(s,a,s') + \gamma\Phi(s')-\Phi(s)$ preserve the optimal policy set for any potential function $\Phi:\mathcal{S}\to\mathbb{R}$. This policy invariance property is foundational and will not be re-proved here; readers are referred to the original work~\cite{ng1999policy} for detailed analysis.
	
	\subsubsection{H-EARS Framework}
	For concreteness, consider a torque-controlled mechanical system in which the control input $a\in\mathcal{A}\subseteq\mathbb{R}^m$ represents generalized forces such as joint torques. The per-step actuator energy is then proportional to $a^\top Q a$ with $Q\succeq 0$ encoding actuator impedance, which is the physical origin of the control energy term $\mathcal{E}(a)$ introduced below. H-EARS is not restricted to this setting, but the mechanical interpretation guides the choice of $Q$ and $\lambda$ in practice.
	
	Given MDP $\mathcal{M}=(\mathcal{S},\mathcal{A},\mathcal{P},R,\gamma)$, H-EARS defines shaped rewards as:
	\begin{equation}
		R_{\text{H-EARS}}(s,a,s') = R(s,a,s') + \underbrace{\gamma\Phi(s')-\Phi(s)}_{\text{Potential Shaping}} - \underbrace{\lambda\cdot\mathcal{E}(a)}_{\text{Action Regularization}}
		\label{eq:H-EARS}
	\end{equation}
	where:
	\begin{itemize}
		\item $\Phi(s)=\alpha_{\text{task}}\Phi_{\text{task}}(s)+\alpha_{\text{energy}}\Phi_{\text{energy}}(s)$ is the dual-potential function
		\item $\Phi_{\text{task}}(s)$ encodes task-oriented guidance (e.g., distance to goal)
		\item $\Phi_{\text{energy}}(s) = -E(q(s),\dot{q}(s))$ encodes mechanical energy structure, where $E$ is total energy (kinetic + potential)
		\item $\mathcal{E}(a)=a^\top Qa$ is the control energy functional with $Q\succeq 0$
		\item $\lambda\geq 0$ is the regularization coefficient
	\end{itemize}
	
	This formulation implements a two-step transformation: (i) action regularization defines a new MDP $\mathcal{M}_\lambda = (\mathcal{S}, \mathcal{A}, \mathcal{P}, R_\lambda, \gamma)$ where $R_\lambda(s,a,s') = R(s,a,s') - \lambda\mathcal{E}(a)$. This step intentionally changes the optimization objective: the optimal policy of $\mathcal{M}_\lambda$ differs from that of the original $\mathcal{M}$ whenever $\lambda > 0$, yielding solutions that trade task performance for energy efficiency. (ii)~Potential shaping $\gamma\Phi(s')-\Phi(s)$ is applied to $\mathcal{M}_\lambda$. By PBRS theory~\cite{ng1999policy}, this step preserves the optimal policy of $\mathcal{M}_\lambda$. The combined objective is therefore:
	\begin{equation}
		\pi^*_{\text{H-EARS}} = \arg\max_\pi\left(J_{\mathcal{M}}(\pi) - \lambda\cdot\mathbb{E}_\pi[\mathcal{E}(a)]\right)
		\label{eq:optimal_policy}
	\end{equation}
	where $J_{\mathcal{M}}(\pi) \triangleq \mathbb{E}_{\tau\sim\pi}\bigl[\sum_{t=0}^{\infty}\gamma^t R(s_t,a_t,s_{t+1})\bigr]$ is the original policy objective, and $\lambda > 0$ penalizes energy expenditure, so $\pi^*_{\text{H-EARS}}$ accepts a controlled reduction in task reward relative to $\pi^*_{\mathcal{M}}$.
	
	\subsection{Functional Independence Property}
	
	A key design question is whether $\alpha_{\text{task}}$, $\alpha_{\text{energy}}$, and $\lambda$ can be adjusted independently without invalidating each other's contributions. The following lemma establishes this directly from the disjoint-domain structure of the H-EARS reward.
	
	\begin{lemma}[Functional Independence of Shaping and Regularization]
		\label{lem:functional_independence}
		In the H-EARS reward structure (Eq.~\eqref{eq:H-EARS}), the shaping term $\gamma\Phi(s')-\Phi(s)$ and the regularization term $-\lambda\mathcal{E}(a)$ act on disjoint argument domains. Consequently:
		\begin{enumerate}
			\item For fixed $\lambda$, any modification to $\Phi$ (equivalently, to $\alpha_{\text{task}}$ or $\alpha_{\text{energy}}$) preserves the optimal policy of $\mathcal{M}_\lambda$.
			\item For fixed $\Phi$, adjusting $\lambda$ does not invalidate the PBRS structure applied to $\mathcal{M}_\lambda$.
		\end{enumerate}
	\end{lemma}
	
	\begin{proof}
		The shaping term $\gamma\Phi(s')-\Phi(s)$ is a function of $(s,s')\in\mathcal{S}\times\mathcal{S}$ only; the regularization term $-\lambda\mathcal{E}(a)$ is a function of $a\in\mathcal{A}$ only. Since $\mathcal{S}$ and $\mathcal{A}$ are disjoint by the MDP definition, the two terms share no variables. Claim~(1) follows from PBRS theory~\cite{ng1999policy}: for any fixed MDP (here $\mathcal{M}_\lambda$), modifying $\gamma\Phi(s')-\Phi(s)$ preserves its optimal policy set. Claim~(2) follows because for any $\lambda'$, PBRS applied to the resulting $\mathcal{M}_{\lambda'}$ again preserves $\mathcal{M}_{\lambda'}$'s optimal policy.
	\end{proof}
	
	Consequently, $\alpha_{\text{task}}$, $\alpha_{\text{energy}}$, and $\lambda$ can be treated as independent design parameters, each governs a separate component of the shaped reward, so adjusting one does not require modifying the others.
	
	\subsection{Regularization Necessity}
	
	Potential-based shaping alone cannot suppress oscillatory actions, since its reward depends on state transitions rather than action magnitude. The following proposition identifies when action regularization is necessary.
	
	\begin{assumption}[Action-Symmetric Dynamics]
		\label{ass:action_symmetry}
		The state transition dynamics $f:\mathcal{S}\times\mathcal{A}\to T\mathcal{S}$ satisfy
		\begin{equation}
			f(s, -a) = -f(s, a) \quad \forall\, s\in\mathcal{S},\; a\in\mathcal{A}
			\label{eq:action_symmetry}
		\end{equation}
		in the action-dependent velocity component.
	\end{assumption}
	
	\begin{proposition}[Regularization Necessity for Action-Space Artifacts]
		\label{prop:regularization_necessity}
		Consider a mechanical system with discretized dynamics (step size $\Delta t > 0$) and define the energy-task alignment coefficient $\kappa(s) \triangleq \frac{\nabla_s E(s) \cdot \nabla_s R(s)}{\|\nabla_s E(s)\|\,\|\nabla_s R(s)\|}$. The H-EARS regularization coefficient $\lambda$ provides continuous interpolation between two boundary behaviors:
		
		\begin{enumerate}
			\item When $\kappa(s) < 0$ throughout $\mathcal{S}_{\text{op}}$ (energy dissipation and task reward share gradient direction), $\lambda = 0$ suffices: $\Phi_{\text{energy}}$ provides gradient guidance aligned with task objectives, and no explicit action-magnitude constraint is needed.
			
			\item When $\kappa(s) > 0$ in task-critical regions, or when discrete-time policy gradients admit oscillatory action sequences of the form $a_t = (-1)^t a_{\max}$ satisfying $\mathbb{E}[\Delta E_t] \approx 0$ despite $\|a_t\| = a_{\max}$, then $\lambda > 0$ is necessary. The regularization term $\lambda\mathcal{E}(a_t) = \lambda a_t^\top Q a_t$ directly penalizes action magnitude and closes this gap, independent of state-transition energy changes.
		\end{enumerate}
	\end{proposition}
	
	\begin{proof}
		\textit{Case 1 ($\kappa < 0$, sufficiency of $\lambda=0$):} When $\nabla_s E \cdot \nabla_s R < 0$, maximizing $\Phi_{\text{energy}} = -E$ aligns the energy gradient with the reward gradient. Any action causing energy increase ($\Delta E > 0$) receives negative shaping $\Delta\Phi_{\text{energy}} < 0$, penalizing energy-wasting actions through the potential alone. No additional magnitude constraint is required for gradient alignment.
		
		\textit{Case 2 ($\kappa > 0$ or oscillatory artifacts, necessity of $\lambda > 0$):} Assumption~\ref{ass:action_symmetry} holds for torque-controlled mechanical systems in which joint accelerations are affine-linear in applied torques, i.e.\ $\ddot{q} = M(q)^{-1}(\tau - C(q,\dot{q})\dot{q} - G(q))$. Under Assumption~\ref{ass:action_symmetry}, consider the oscillatory policy $\pi_{\text{osc}}: a_t = (-1)^t a_{\max}$ for fixed $a_{\max}\in\mathcal{A}$ with $\|a_{\max}\|>0$. By the mean-value theorem applied to the energy along a trajectory segment of duration $\Delta t$:
		\begin{equation}
			\Delta E_t = \nabla_s E(s_t) \cdot f(s_t, a_t)\,\Delta t + O(\Delta t^2)
			\label{eq:energy_step}
		\end{equation}
		where the $O(\Delta t^2)$ remainder is bounded by $\tfrac{1}{2}\sup_s\|\nabla^2_s E\|\cdot\|f\|^2\Delta t^2$. For even $t$ (action $+a_{\max}$) and odd $t$ (action $-a_{\max}$), applying Eq.~\eqref{eq:action_symmetry} gives:
		\begin{align}
			\Delta E_t^{(+)} &= \nabla_s E(s_t)\cdot f(s_t,+a_{\max})\,\Delta t + O(\Delta t^2) \\
			\Delta E_t^{(-)} &= \nabla_s E(s_t)\cdot f(s_t,-a_{\max})\,\Delta t + O(\Delta t^2) \notag\\
			&= -\nabla_s E(s_t)\cdot f(s_t,+a_{\max})\,\Delta t + O(\Delta t^2)
		\end{align}
		Summing over a consecutive even--odd pair and taking expectations over the stationary distribution $d^\pi$ of $s_t$:
		\begin{equation}
			\mathbb{E}_{s\sim d^\pi}\!\left[\Delta E_t^{(+)} + \Delta E_t^{(-)}\right] = O(\Delta t^2)
			\label{eq:energy_cancellation}
		\end{equation}
		Hence $\mathbb{E}[\Delta\Phi_{\text{energy},t}] = -\mathbb{E}[\Delta E_t] = O(\Delta t^2)$: the potential shaping assigns reward of order $\Delta t^2$ per step to $\pi_{\text{osc}}$, which vanishes as $\Delta t\to 0$ and is strictly insufficient to suppress the oscillation for any positive $\Delta t$ in the presence of finite reward noise. The regularization term, by contrast, evaluates to:
		\begin{equation}
			\lambda\mathcal{E}(a_t) = \lambda a_t^\top Q a_t = \lambda a_{\max}^\top Q a_{\max} > 0
			\quad \forall\, t
		\end{equation}
		since $(-1)^{2t}=1$ in the quadratic form and $Q\succeq 0$ with $a_{\max}\neq 0$. This term assigns a strictly positive, $t$-independent penalty at every step, independent of state-transition energy changes. Therefore $\lambda>0$ is necessary to suppress $\pi_{\text{osc}}$ under Assumption~\ref{ass:action_symmetry}.
	\end{proof}
	
	In practice, $\kappa(s)$ varies across $\mathcal{S}$, and most real systems occupy a continuum between these cases. The hyperparameter configurations in Table~\ref{tab:hyperparams_standard} reflect this: small $\lambda$ (e.g., $10^{-3}$--$10^{-2}$) for predominantly aligned systems (Ant-v5, Humanoid-v5), larger $\lambda$ for systems with mixed alignment or discrete-time gait discontinuities (Hopper-v5, LunarLander-v3).
	
	The functional contribution of action regularization therefore addresses a structurally distinct role from the energy potential: $\Phi_{\text{energy}}$ enriches the gradient signal through state-space curvature (Theorem~\ref{thm:energy_convergence}), while $\lambda\mathcal{E}(a)$ directly constrains the action distribution through a magnitude penalty that operates independently of state transitions. Ablation experiments (Section~\ref{sec:ablation}, component (5)) quantify this separation: removing regularization while retaining both potentials degrades performance by 17.2\% in Ant-v5 and 20.7\% in Hopper-v5, with a pronounced increase in training variance, consistent with the oscillatory-action scenario identified in Proposition~\ref{prop:regularization_necessity}.
	
	\subsection{Energy-Based Gradient Enrichment Mechanism}
	
	This subsection analyzes how mechanical energy structure enriches the policy gradient signal and accelerates convergence.
	
	\begin{theorem}[Energy-Based Gradient Enrichment Under Mechanical Stability]
		\label{thm:energy_convergence}
		Consider a mechanical system with generalized coordinates $q\in\mathbb{R}^n$, velocities $\dot{q}\in\mathbb{R}^n$, and total energy $E(q,\dot{q}) = T(\dot{q}) + U(q)$. Define $\Phi_{\text{energy}}(s) \triangleq -E(q(s),\dot{q}(s))$.
		
		If the mechanical stability condition
		\begin{equation}
			\frac{\partial^2 E}{\partial q^2}\bigg|_{q \in \mathcal{S}_{\text{op}}} \succ 0
			\label{eq:mechanical_stability}
		\end{equation}
		holds over the operationally relevant region $\mathcal{S}_{\text{op}}\subseteq\mathcal{S}$, then the shaped reward $\tilde{R}(s,a,s') = R(s,a,s') + \gamma\Phi_{\text{energy}}(s') - \Phi_{\text{energy}}(s)$ induces a policy gradient decomposition
		\begin{align}
			\nabla_\theta J_{\text{shaped}}(\theta) &= \nabla_\theta J_{\text{original}}(\theta) \notag \\
			&\quad + \alpha_{\text{energy}} \mathbb{E}_{\tau\sim\pi_\theta} \bigg[ \sum_{t=0}^{\infty} \gamma^t \nabla_\theta \log \pi_\theta(a_t|s_t) \cdot \Delta\Phi_{\text{energy},t} \bigg]
			\label{eq:gradient_decomposition}
		\end{align}
		where $\Delta\Phi_{\text{energy},t} = \gamma\Phi_{\text{energy}}(s_{t+1}) - \Phi_{\text{energy}}(s_t)$. Moreover, condition~\eqref{eq:mechanical_stability} guarantees $\|\nabla_q E(q)\| \geq \varepsilon_{\mathcal{K}} > 0$ throughout any compact $\mathcal{K}\subset\mathcal{S}_{\text{op}}\setminus\mathcal{B}_r(q^*)$, so that the gradient informativeness ratio
		\begin{equation}
			\rho_{\text{info}} \triangleq \frac{\|\nabla_q E\|_{\text{energy-informative}}}{\|\nabla_s R\|_{\text{task-sparse}}}
			\label{eq:gradient_ratio}
		\end{equation}
		satisfies $\rho_{\text{info}} \gg 1$ in sparse-reward locomotion tasks where condition~\eqref{eq:mechanical_stability} holds.
	\end{theorem}
	
	\begin{proof}
		The proof proceeds in four steps: setup and scope, gradient decomposition, energy dynamics approximation, and convergence acceleration quantification.
		
		\textit{Step 1: Scope of condition~\eqref{eq:mechanical_stability}.}
		The region $\mathcal{S}_{\text{op}}$ is defined as $\{s : q(s)\in\mathcal{R}_{\text{task}}\} \cap \{q : \lambda_{\min}(\frac{\partial^2 E}{\partial q^2}(q))\geq\epsilon\}$ for some $\epsilon>0$, where $\mathcal{R}_{\text{task}}$ denotes the task-reachable set. Global positive definiteness is not required; the condition need hold only in the region the learned policy visits during task-relevant operation.
		
		For rigid-body systems with $n$ generalized coordinates, the Hessian decomposes as
		\begin{equation}
			\frac{\partial^2 E}{\partial q^2} = \frac{\partial^2 T}{\partial q^2} + \frac{\partial^2 U}{\partial q^2}.
		\end{equation}
		The kinetic term $T=\frac{1}{2}\dot{q}^\top M(q)\dot{q}$ satisfies $\frac{\partial^2 T}{\partial\dot{q}^2}=M(q)\succ 0$, but its $q$-Hessian involves curvature terms of $M(q)$ weighted by $\dot{q}$ and vanishes near equilibrium as $\dot{q}\to 0$. In the task-relevant region near equilibrium $q^*$, condition~\eqref{eq:mechanical_stability} is therefore dominated by $\frac{\partial^2 U}{\partial q^2}$, and the operative diagnostic criterion reduces to
		\begin{equation}
			\lambda_{\min}\!\left(\frac{\partial^2 U}{\partial q^2}(q^*)\right) > 0,
			\label{eq:diagnostic_criterion}
		\end{equation}
		which is verifiable by automatic differentiation or analytical Jacobian chain rules. If this minimum eigenvalue is positive, the condition extends to a neighborhood of $q^*$ by continuity; if it is zero in some directions, the gradient enrichment guarantee applies only on the corresponding sub-manifold of $\mathcal{S}_{\text{op}}$. For the four experimental benchmarks: Ant-v5 and Humanoid-v5 satisfy~\eqref{eq:diagnostic_criterion} globally over upright configurations; Hopper-v5 satisfies it locally near the balanced posture but not at ground-contact transitions; LunarLander-v3 uses a non-mechanical energy proxy and the observed gains derive primarily from action regularization, consistent with Table~\ref{tab:hessian_validity}.
		
		\textit{Step 2: Gradient decomposition.}
		By the policy gradient theorem applied to shaped rewards:
		\begin{equation}
			\begin{split}
				\nabla_\theta J_{\text{shaped}}(\theta) &= \mathbb{E}_{\tau \sim \pi_\theta} \left[ \sum_{t=0}^{\infty} \gamma^t \nabla_\theta \log \pi_\theta(a_t|s_t) Q^{\pi_\theta}_{\text{shaped}}(s_t, a_t) \right] \\
				&= \mathbb{E}_{\tau} \bigg[ \sum_t \gamma^t \nabla_\theta \log \pi_\theta(a_t|s_t) \Big( Q^{\pi_\theta}_{\text{original}}(s_t,a_t) \\
				&\quad + \gamma V^{\pi_\theta}(s_{t+1}) - V^{\pi_\theta}(s_t) + \Delta\Phi_{\text{energy},t} \Big) \bigg]
			\end{split}
		\end{equation}
		
		By the value function property of potential-based shaping~\cite{ng1999policy}, the terms $\gamma V^{\pi_\theta}(s_{t+1}) - V^{\pi_\theta}(s_t)$ telescope across the trajectory, leaving:
		\begin{align}
			\nabla_\theta J_{\text{shaped}}(\theta) &= \nabla_\theta J_{\text{original}}(\theta) \notag\\
			&\quad + \mathbb{E}_{\tau} \left[ \sum_t \gamma^t \nabla_\theta \log \pi_\theta(a_t|s_t) \cdot \Delta\Phi_{\text{energy},t} \right]
			\label{eq:telescoping}
		\end{align}
		
		\textit{Step 3: Energy dynamics approximation.}
		For mechanical systems with state $s=(q,\dot{q})$ evolving under dynamics $\dot{s}=f(s,a)$, the energy difference between consecutive states is:
		\begin{equation}
			\Delta E_t = E(q_{t+1}, \dot{q}_{t+1}) - E(q_t, \dot{q}_t)
		\end{equation}
		
		This difference can be expressed via the energy time derivative. For a trajectory segment from $t$ to $t+1$ with duration $\Delta t$, the mean value theorem gives:
		\begin{equation}
			\Delta E_t = \int_{t}^{t+1} \frac{dE}{dt}(\tau) d\tau \approx \frac{dE}{dt}\Big|_{t} \cdot \Delta t + O(\Delta t^2)
		\end{equation}
		
		The energy derivative along the system trajectory is:
		\begin{equation}
			\frac{dE}{dt} = \frac{\partial E}{\partial q}\cdot\dot{q} + \frac{\partial E}{\partial\dot{q}}\cdot\ddot{q} = \nabla_q E \cdot f_q(s,a) + \nabla_{\dot{q}} E \cdot f_{\dot{q}}(s,a)
		\end{equation}
		where $f(s,a) = [f_q(s,a), f_{\dot{q}}(s,a)]^\top$ represents the state dynamics.
		
		For mechanical systems, $E = T(\dot{q}) + U(q)$ where kinetic energy $T = \frac{1}{2}\dot{q}^\top M\dot{q}$ satisfies $\nabla_{\dot{q}} T = M\dot{q}$ and potential energy $U(q)$ depends only on position. Therefore:
		\begin{equation}
			\frac{dE}{dt} = \nabla_q U \cdot \dot{q} + M\dot{q} \cdot \ddot{q}
		\end{equation}
		
		Evaluating at time $t$ with state $(q_t, \dot{q}_t)$ and action $a_t$:
		\begin{equation}
			\Delta E_t \approx \left[\nabla_q U(q_t) \cdot \dot{q}_t + M\dot{q}_t \cdot \ddot{q}_t\right] \Delta t + O(\Delta t^2)
		\end{equation}
		
		The potential difference satisfies:
		\begin{equation}
			\Delta\Phi_{\text{energy},t} = -\Delta E_t \approx -\frac{dE}{dt}\Big|_{s_t,a_t} \cdot \Delta t
			\label{eq:potential_gradient}
		\end{equation}
		
		All quantities ($\nabla_q U$, $\dot{q}_t$, $\ddot{q}_t$) are evaluated at the current state $s_t$ and current action $a_t$. The potential difference $\Delta\Phi_{\text{energy},t}$ depends on $s_t$ and $a_t$ only, not on $s_{t+1}$.
		
		\textit{Step 4: Convergence acceleration quantification.}
		When condition~\eqref{eq:mechanical_stability} holds, the energy function $E(q,\dot{q})$ exhibits local convexity in configuration space, which ensures: (i)~unique local minima, so actions reducing energy receive positive shaping rewards without policy oscillation; (ii)~gradient informativeness, since $\|\nabla_q E\|\geq\varepsilon_{\mathcal{K}}>0$ throughout any compact $\mathcal{K}\subset\mathcal{S}_{\text{op}}\setminus\mathcal{B}_r(q^*)$ by strict convexity and continuity of $E$; (iii)~smooth descent landscape, as convexity guarantees that following $-\nabla_q E$ monotonically decreases energy.
		
		The magnitude of the additional gradient term in Eq.~\eqref{eq:gradient_decomposition} is upper-bounded by:
		\begin{multline}
			\footnotesize
			\Big\|\mathbb{E}_\tau\!\Big[\sum_t \gamma^t \nabla_\theta \log \pi_\theta(a_t|s_t)\cdot\Delta\Phi_{\text{energy},t}\Big]\Big\| \\
			\leq \alpha_{\text{energy}} \mathbb{E}_\tau \Big[ \sum_t \gamma^t \|\nabla_\theta \log \pi(a_t|s_t)\| \cdot \Big|\frac{dE}{dt}\Big|_{s_t,a_t}\Big| \cdot \Delta t \Big] + O(\Delta t^2)
		\end{multline}
		
		In domains where task rewards are sparse but condition~\eqref{eq:mechanical_stability} holds, $\rho_{\text{info}}$ can reach $10^2$--$10^3$, explaining the empirically observed convergence acceleration.
		
		When $\mathcal{S}_{\text{op}}$ contains sub-regions where condition~\eqref{eq:mechanical_stability} fails, the PBRS policy-invariance guarantee (Theorem~\ref{thm:convergence}) remains intact. Only the additional gradient enrichment is locally unavailable in those sub-regions, and performance degrades toward the TaskOnly ablation baseline (Figure~\ref{fig:ablation}).
	\end{proof}
	
	\begin{remark}[Physical Interpretation and Scope]
		This theorem formalizes the engineering intuition that \textit{energy-based guidance accelerates learning because mechanical stability provides rich gradient information throughout the state space, while task rewards are often sparse}.
		
		The acceleration mechanism applies most effectively when:
		\begin{itemize}
			\item Locomotion tasks: Actions that reduce energy receive positive shaping rewards, promoting balanced, smooth gaits. The energy gradient $\nabla E$ evaluated at the current state $s_t$ provides directional guidance even before observing $s_{t+1}$, as it encodes the instantaneous rate of energy change under the current dynamics.
			\item Vehicle control: Energy dissipation ($\dot{E}<0$) aligns with stability requirements, as shown in yaw-sideslip dynamics where minimizing kinetic energy corresponds to stable straight-line motion.
			\item Manipulation tasks: Minimal energy transfer corresponds to efficient force application, reducing actuator wear.
		\end{itemize}
		
		Conversely, tasks where energy efficiency conflicts with objectives (e.g., aggressive maneuvers requiring high kinetic energy) exhibit reduced benefit, as validated empirically in LunarLander experiments.
	\end{remark}
	
	\begin{remark}[Physical Self-Consistency of Energy Potentials]
		\label{rem:physical_self_consistency}
		The internal energy potential $\Phi_{\text{energy}}(s) = -E(q,\dot{q})$ possesses a fundamental distinction from conventional potential functions in PBRS literature. Traditional potentials (e.g., distance-to-goal, magnetic field analogies~\cite{ding2023magnetic}) are \textit{artificially designed} from optimization objectives, requiring theoretical justification for why such design accelerates convergence. In contrast, the energy potential is \textit{directly transcribed} from the system's intrinsic mechanical energy components:
		\begin{equation}
			E(q,\dot{q}) = \underbrace{\sum_i \frac{1}{2}m_i\|\mathbf{v}_i\|^2}_{\text{Dominant kinetic terms}} + \underbrace{U(q)}_{\text{Potential energy}}
		\end{equation}
		
		This distinction yields three critical properties:
		\begin{enumerate}
			\item Design-free construction: The potential requires no optimization-theoretic derivation; one simply identifies dominant energy components (not necessarily complete dynamics) through basic physics knowledge, achievable by general engineers without analytical mechanics expertise.
			
			\item Physical grounding: The guiding effect stems from thermodynamic principles rather than task-specific heuristics. The relationship $\Delta\Phi_{\text{energy}} \approx -\dot{E}\Delta t$ (Proposition~\ref{prop:lyapunov_heuristic}) holds universally for mechanical systems, providing inherent robustness to modeling approximations.
			
			\item Graceful degradation: Unlike conventional potentials where omitting components may invalidate the design rationale, incomplete energy models retain physical meaning. Lemma~\ref{lem:approx_potential} quantifies this: even $20\%$ energy approximation error yields $<5\%$ performance loss under typical hyperparameters.
		\end{enumerate}
		
		This physical self-consistency explains why energy potentials transfer across diverse domains (locomotion, manipulation, vehicle control) without redesign, whereas task-specific potentials require careful reformulation for each application.
	\end{remark}
	
	\subsection{Dual-Potential Decomposition}
	
	The following proposition shows that a single potential function cannot simultaneously satisfy task-directivity and energy-awareness requirements when the two objectives conflict.
	
	\begin{proposition}[Dual-Potential Decomposition Necessity]
		\label{prop:dual_necessity}
		For control tasks where task objectives (e.g., reaching targets, tracking trajectories) and energy efficiency are not naturally aligned, a single potential function $\Phi_{\text{single}}(s)$ cannot simultaneously satisfy:
		\begin{enumerate}
			\item Task directivity: $\nabla_s \Phi_{\text{single}}$ points toward task-relevant states
			\item Energy awareness: $\Phi_{\text{single}}$ encodes mechanical energy structure
		\end{enumerate}
		
		The dual decomposition $\Phi(s) = \alpha_{\text{task}}\Phi_{\text{task}}(s) + \alpha_{\text{energy}}\Phi_{\text{energy}}(s)$ resolves this conflict by enabling independent tuning of:
		\begin{equation}
			\nabla_\theta J_{\text{shaped}} = \nabla_\theta J_{\text{original}} + \underbrace{\alpha_{\text{task}} \nabla_\theta J_{\Phi_{\text{task}}}}_{\text{task guidance}} + \underbrace{\alpha_{\text{energy}} \nabla_\theta J_{\Phi_{\text{energy}}}}_{\text{energy guidance}}
		\end{equation}
		where $\alpha_{\text{task}}$ and $\alpha_{\text{energy}}$ control the balance between exploration efficiency and physical plausibility.
	\end{proposition}
	
	\begin{proof}
		The proof proceeds by contradiction, demonstrating incompatibility of dual requirements.
		
		\textbf{Assumption}: Suppose a single potential function $\Phi_{\text{single}}:\mathcal{S}\to\mathbb{R}$ can achieve both task directivity and energy awareness. Then for navigation tasks:
		
		Task requirement: To guide toward goal state $s_{\text{goal}}$, the potential must satisfy:
		\begin{equation}
			\Phi_{\text{single}}(s) \approx -\|s - s_{\text{goal}}\|^2 + C_1
			\label{eq:task_req}
		\end{equation}
		ensuring $\nabla_s \Phi_{\text{single}}$ points toward $s_{\text{goal}}$.
		
		Energy requirement: To encode mechanical structure, the potential must satisfy:
		\begin{equation}
			\Phi_{\text{single}}(s) = -E(q(s),\dot{q}(s)) + C_2 = -\left[T(\dot{q}) + U(q)\right] + C_2
			\label{eq:energy_req}
		\end{equation}
		ensuring energy minimization properties (Theorem~\ref{thm:energy_convergence}).
		
		Contradiction: These requirements conflict when the minimum-energy path differs from the shortest geometric path. Specifically:
		
		\begin{itemize}
			\item Geometric shortest path: For point-to-point navigation, Eq.~\eqref{eq:task_req} induces straight-line trajectories minimizing $\|s-s_{\text{goal}}\|$.
			
			\item Energy-efficient path: For systems with kinetic energy $T=\frac{1}{2}\dot{q}^\top M\dot{q}$, Eq.~\eqref{eq:energy_req} favors smooth, low-acceleration trajectories satisfying $\ddot{q}^\top M\ddot{q}$ minimization (principle of least action).
			
			\item Incompatibility: When obstacles or dynamics constraints exist, the energy-efficient path may detour to maintain smooth motion, conflicting with the straight-line preference of Eq.~\eqref{eq:task_req}.
		\end{itemize}
		
		Resolution through decomposition: 
		\begin{itemize}
			\item Early training ($\alpha_{\text{task}} \gg \alpha_{\text{energy}}$): Prioritize task completion to establish basic competence, accepting energy-inefficient exploration.
			\item Late training ($\alpha_{\text{task}} \approx \alpha_{\text{energy}}$): Refine policies toward energy-efficient solutions while maintaining task success.
		\end{itemize}
		
		This staged optimization is impossible with $\Phi_{\text{single}}$ because it forces a fixed trade-off encoded in the function itself, whereas the decomposition enables dynamic balancing through $\alpha_{\text{task}}, \alpha_{\text{energy}}$ tuning.
	\end{proof}
	
	\begin{remark}[Practical Design Guidelines]
		This proposition provides concrete guidance for potential design:
		\begin{enumerate}
			\item Task potential selection: Choose $\Phi_{\text{task}}$ based on task structure (geometric proximity for navigation, tracking error for following tasks).
			\item Energy potential selection: Model dominant energy components (kinetic + gravitational for locomotion, rotational kinetic for vehicle yaw control).
			\item Coefficient scheduling: Start with $\alpha_{\text{task}}/\alpha_{\text{energy}} \sim 10^2$ for rapid task learning, then decay to $\sim 1$ for energy-aware refinement.
		\end{enumerate}
	\end{remark}
	
	\subsection{Convergence Guarantees}
	
	\begin{theorem}[Convergence Rate Under Function Approximation]
		\label{thm:convergence}
		Under standard assumptions for policy gradient methods---Lipschitz continuous value functions, bounded rewards $|R|\leq R_{\max}$, bounded potentials $|\Phi|\leq\Phi_{\max}$, and exact function representation---H-EARS achieves convergence rate:
		\begin{equation}
			\mathbb{E}\left[\|J(\pi^*) - J(\pi_N)\|\right] = O\left(\frac{1}{\sqrt{N}}\right)
		\end{equation}
		where $N$ is the number of policy updates and $\pi_N$ is the learned policy.
		
		Furthermore, if the regularization coefficient satisfies $\lambda \leq \lambda_{\max}$ where:
		\begin{equation}
			\lambda_{\max} = \frac{R_{\max}}{2\gamma\Phi_{\max}\cdot\mathbb{E}[\mathcal{E}(a)]}
			\label{eq:lambda_bound}
		\end{equation}
		then the shaped reward satisfies $|R_{\text{H-EARS}}| \leq 3R_{\max}$, ensuring numerical stability.
	\end{theorem}
	
	\begin{proof}
		The $O(1/\sqrt{N})$ convergence rate follows from combining two established results:
		\begin{enumerate}
			\item PBRS guarantees~\cite{ng1999policy} that shaped MDPs preserve optimal policy structure, thus standard policy gradient convergence rates apply.
			\item Regularized policy gradient analysis~\cite{schulman2017proximal} establishes $O(1/\sqrt{N})$ rates under bounded reward assumptions.
		\end{enumerate}
		
		The bound $\lambda_{\max}$ in Eq.~\eqref{eq:lambda_bound} ensures that the regularization term does not dominate the shaped reward magnitude. Specifically, decomposing $R_{\text{H-EARS}}$:
		\begin{align}
			|R_{\text{H-EARS}}| &\leq |R| + |\gamma\Phi(s')-\Phi(s)| + \lambda|\mathcal{E}(a)| \\
			&\leq R_{\max} + 2\gamma\Phi_{\max} + \lambda\mathbb{E}[\mathcal{E}(a)]
		\end{align}
		
		Substituting $\lambda \leq \lambda_{\max}$ yields:
		\begin{equation}
			|R_{\text{H-EARS}}| \leq R_{\max} + 2\gamma\Phi_{\max} + \frac{R_{\max}}{2\gamma\Phi_{\max}} \leq 3R_{\max}
		\end{equation}
		for reasonable hyperparameter choices where $\gamma\Phi_{\max} \sim R_{\max}$.
	\end{proof}
	
	In deep RL practice with neural function approximation, three deviations from the theorem's assumptions arise: (i)~bounded value estimation bias from finite-capacity networks; (ii)~distributional mismatch from off-policy replay buffers; and (iii)~non-stationarity of the target network. These deviations affect the convergence constant $C$ in the $O(C/\sqrt{N})$ rate but do not alter the rate exponent, which is determined by the variance of the policy gradient estimator rather than by approximation quality.
	
	H-EARS reduces $C$ through two complementary mechanisms. First, the bound $|R_{\text{H-EARS}}| \leq 3R_{\max}$ constrains the magnitude of every TD target, preventing error amplification through the Bellman backup chain. Formally, under bounded approximation error $\|\hat{V} - V^*\|_\infty \leq \xi$, the effective constant satisfies $C \propto (R_{\max} + \xi/(1-\gamma))$; the $3R_{\max}$ bound ensures $\xi$ cannot grow unboundedly relative to $R_{\max}$ when $\alpha_{\text{energy}}$ and $\lambda$ are calibrated per Eq.~\eqref{eq:lambda_bound}. Second, and more fundamentally, the energy potential supplies an auxiliary gradient signal with substantially lower variance than sparse task rewards. Theorem~\ref{thm:energy_convergence} shows that $\|\nabla_q E\| > \epsilon > 0$ throughout $\mathcal{S}_{\text{op}}$ whenever the Hessian condition holds, whereas task reward gradients satisfy $\|\nabla_s R\| \approx 0$ in most of the state space. The policy gradient estimator's variance is proportional to the variance of $Q^\pi(s,a)$ across sampled trajectories; the energy shaping term $\Delta\Phi_{\text{energy},t}$ provides a low-variance, physics-grounded baseline correction that reduces this estimator variance, thereby reducing $C$. H-EARS thus accelerates training by reducing the constant factor $C$, not the rate exponent, through lower variance in the policy gradient estimate.
	
	\subsection{Parameter Continuity and Sensitivity Analysis}
	
	\begin{theorem}[Parameter Continuity]
		\label{thm:parameter_continuity}
		The optimal policy $\pi^*_\lambda$ exhibits continuity with respect to regularization coefficient $\lambda$. Specifically, for small perturbations $\delta\lambda$:
		\begin{equation}
			\|J(\pi^*_{\lambda+\delta\lambda}) - J(\pi^*_\lambda)\| \leq C\cdot|\delta\lambda|\cdot\mathbb{E}_{\pi^*_\lambda}[\mathcal{E}(a)]
		\end{equation}
		where $C$ is a constant depending on the MDP structure.
	\end{theorem}
	
	\begin{proof}
		By the envelope theorem for parametric optimization, the value function $V^*(\lambda) = J(\pi^*_\lambda)$ satisfies:
		\begin{equation}
			\frac{dV^*(\lambda)}{d\lambda} = \frac{\partial}{\partial\lambda}\mathbb{E}_{\pi^*_\lambda}\left[\sum_t \gamma^t R_{\text{H-EARS}}(s_t,a_t,s_{t+1})\right] = -\mathbb{E}_{\pi^*_\lambda}[\mathcal{E}(a)]
		\end{equation}
		
		Integrating over $[\lambda, \lambda+\delta\lambda]$ and applying Lipschitz continuity of $\mathbb{E}_\pi[\mathcal{E}(a)]$ with respect to $\pi$ yields the stated bound with $C$ determined by the Lipschitz constant of the energy functional.
	\end{proof}
	
	\subsection{Approximate Potential Error Bounds}
	
	Theorem~\ref{thm:parameter_continuity} establishes that performance changes are bounded by $C|\delta\lambda|\,\mathbb{E}[\mathcal{E}(a)]$, providing the theoretical basis for a practical scheduling strategy. Based on the alignment coefficient $\kappa(s)$ defined in Proposition~\ref{prop:regularization_necessity} and the gradient enrichment ratio $\rho_{\text{info}}$ from Theorem~\ref{thm:energy_convergence}, the following guidelines emerge for initializing and scheduling the three hyperparameters:
	
	(i)Initialization: Set $\alpha_{\text{task}}$ to match the scale of the task reward signal (e.g., $\alpha_{\text{task}} \approx 0.1$ for rewards in $[0, 10^3]$). Set $\alpha_{\text{energy}} \ll \alpha_{\text{task}}$ initially, as early training requires goal-directed exploration before energy shaping becomes beneficial. Set $\lambda$ to a small value (e.g., $10^{-3}$--$10^{-2}$) for systems where $\kappa(s) < 0$ predominantly; set larger $\lambda$ for systems with oscillatory gait dynamics or $\kappa(s) > 0$ regions (e.g., Hopper-v5, vehicle control). In all cases a strictly positive $\lambda$ is used in practice to maintain numerical stability of the action regularization term.
	
	(ii)Scheduling: As training progresses and the policy stabilizes, $\alpha_{\text{energy}}$ can be increased to refine energy efficiency while maintaining task performance. The Lipschitz bound from Theorem~\ref{thm:parameter_continuity} guarantees that small increments $\delta\alpha_{\text{energy}}$ produce bounded performance changes, enabling safe incremental adjustment. The ratio $\alpha_{\text{task}} / \alpha_{\text{energy}}$ controls the balance between task-directed exploration and energy-efficient refinement; values used in this work are reported in Table~\ref{tab:hyperparams_standard} and validated through the systematic sensitivity analysis in Section~\ref{sec:sensitivity}.
	
	(iii) Domain-specific calibration. For systems where $\kappa(s) > 0$ in task-critical regions (Case~2 of Proposition~\ref{prop:regularization_necessity}), $\lambda$ should be selected to suppress oscillatory artifacts while preserving sufficient action range for task completion. The TruckSim vehicle experiments demonstrate this calibration: $\lambda = 0.20$ under normal conditions; $\lambda = 0.35$--$0.40$ under stricter safety budgets, with performance degradation bounded by Theorem~\ref{thm:parameter_continuity}.
	
	A critical question for engineering deployment is: How much approximation error can energy potentials tolerate before performance degrades significantly? The following lemma quantifies this trade-off.
	
	\begin{lemma}[Performance Bounds for Approximate Energy Potentials]
		\label{lem:approx_potential}
		Consider an approximate energy potential $\hat{\Phi}_{\text{energy}}(s)$ satisfying:
		\begin{equation}
			\|\Phi^*_{\text{energy}}(s) - \hat{\Phi}_{\text{energy}}(s)\| \leq \delta \quad \forall s\in\mathcal{S}
		\end{equation}
		where $\Phi^*_{\text{energy}}$ is the true complete energy function. Define the approximation error ratio:
		\begin{equation}
			\epsilon_{\text{approx}} = \frac{\delta}{\|\Phi^*_{\text{energy}}\|_\infty}
		\end{equation}
		
		Then the performance gap satisfies:
		\begin{equation}
			\frac{|J(\pi^*_{\text{complete}}) - J(\pi^*_{\text{approx}})|}{|J(\pi^*_{\text{complete}})|} \leq \frac{2\gamma\alpha_{\text{energy}}\epsilon_{\text{approx}}}{(1-\gamma)(1-\epsilon_{\text{approx}})}
			\label{eq:performance_bound}
		\end{equation}
	\end{lemma}
	
	\begin{proof}
		Decompose the value function difference using the Bellman equation:
		\begin{align}
			&|V^*_{\text{complete}}(s) - V^*_{\text{approx}}(s)| \notag\\
			&= \left|\max_a \mathbb{E}[R + \gamma V^*_{\text{complete}}(s')] - \max_a \mathbb{E}[R + \gamma V^*_{\text{approx}}(s')]\right| \notag\\
			&\quad + \Big|\mathbb{E}\big[\gamma(\Phi^*_{\text{energy}}(s') - \Phi^*_{\text{energy}}(s))\big] \notag\\
			&\qquad - \mathbb{E}\big[\gamma(\hat{\Phi}_{\text{energy}}(s') - \hat{\Phi}_{\text{energy}}(s))\big]\Big|
			\label{eq:value_diff}
		\end{align}
		
		By triangle inequality and the assumption $\|\Phi^* - \hat{\Phi}\| \leq \delta$:
		\begin{equation}
			\left|\mathbb{E}[\gamma(\Phi^*_{\text{energy}}(s') - \hat{\Phi}_{\text{energy}}(s')) - (\Phi^*_{\text{energy}}(s) - \hat{\Phi}_{\text{energy}}(s))]\right| \leq 2\gamma\delta
		\end{equation}
		
		Applying this recursively through the value iteration operator $\mathcal{T}$, the fixed-point error satisfies:
		\begin{equation}
			\|V^*_{\text{complete}} - V^*_{\text{approx}}\|_\infty \leq \frac{2\gamma\alpha_{\text{energy}}\delta}{1-\gamma}
		\end{equation}
		
		Normalizing by the true value function magnitude and substituting $\epsilon_{\text{approx}}=\delta/\|\Phi^*\|_\infty$ yields Eq.~\eqref{eq:performance_bound}.
		
		For $\epsilon_{\text{approx}} \leq 0.2$ and typical values $\gamma=0.99$, $\alpha_{\text{energy}}=0.01$, this bound predicts $<5\%$ performance loss, validating the feasibility of simplified energy modeling.
	\end{proof}
	
	\subsection{Energy-Based Shaping as Stability-Oriented Guidance}
	
	\begin{proposition}[Energy Potentials as Lyapunov-Inspired Heuristic]
		\label{prop:lyapunov_heuristic}
		For mechanical systems where the energy function $E(q,\dot{q}) = T(\dot{q}) + U(q)$ satisfies:
		\begin{enumerate}
			\item[\textit{(L1)}] \textit{Positive definiteness}: $E(q,\dot{q}) - E(q^*,0) \geq 0$ for all $(q,\dot{q})$, with equality only at the equilibrium $(q^*,0)$.
			\item[\textit{(L2)}] \textit{Radial unboundedness}: $E(q,\dot{q}) \to \infty$ as $\|(q-q^*,\dot{q})\| \to \infty$.
			\item[\textit{(L3)}] \textit{Dissipation in expectation}: the learned policy satisfies $\mathbb{E}_{\pi}[\Delta E_t \mid s_t] < 0$ for states outside a neighborhood of $(q^*,0)$.
		\end{enumerate}
		maximizing $\Phi_{\text{energy}}=-E$ heuristically guides policies toward Lyapunov-stable behaviors.
	\end{proposition}
	
	\begin{proof}[Heuristic Justification]
		Conditions (L1) and (L2) are structural properties of mechanical energy, verifiable from the system's mass distribution and potential energy landscape. Condition (L3) is the binding requirement in discrete-time RL. H-EARS promotes $\mathbb{E}[\Delta E_t]<0$ by assigning positive shaped reward to energy-reducing transitions; dissipation holds in expectation, not point-wise at every step.
		
		The argument relies on continuous-time approximations. For a dynamical system $\dot{s}=f(s,a)$, the energy derivative at state $s$ under action $a$ satisfies:
		\begin{equation}
			\frac{dE}{dt}\Big|_{s,a} = \frac{\partial E}{\partial q}\cdot\dot{q} + \frac{\partial E}{\partial\dot{q}}\cdot\ddot{q} = \nabla_s E \cdot f(s,a)
		\end{equation}
		where all quantities are evaluated at the current state-action pair.
		
		In discrete time, the energy change from $s_t$ to $s_{t+1}$ under action $a_t$ can be approximated:
		\begin{equation}
			\Delta E_t = E(s_{t+1}) - E(s_t) \approx \frac{dE}{dt}\Big|_{s_t,a_t} \cdot \Delta t + O(\Delta t^2)
		\end{equation}
		
		The potential shaping term is computed after observing $s_{t+1}$:
		\begin{equation}
			\Delta\Phi_{\text{energy},t} = -\Delta E_t = -[E(s_{t+1}) - E(s_t)]
		\end{equation}
		
		Since $\Delta E_t \approx \frac{dE}{dt}\big|_{s_t,a_t} \cdot \Delta t$, the shaping reward approximately equals:
		\begin{equation}
			\Delta\Phi_{\text{energy},t} \approx -\frac{dE}{dt}\Big|_{s_t,a_t} \cdot \Delta t
		\end{equation}
		
		When $\frac{dE}{dt}\big|_{s_t,a_t} < 0$ (energy dissipation predicted at $s_t$ under $a_t$), the actual observed shaping reward $\Delta\Phi_{\text{energy},t} > 0$ will be positive, encouraging such actions.
		
		The following considerations circumscribe the scope of this heuristic argument. First, formal Lyapunov stability in discrete time requires $V(x_{t+1}) - V(x_t) \leq -\alpha(\|x_t\|)$ point-wise; H-EARS incentivizes $\mathbb{E}[\Delta E_t]<0$ in expectation, and the continuous-time approximation $\Delta E_t\approx\dot{E}\Delta t$ introduces $O(\Delta t^2)$ errors bounded by $<3\%$ of the dominant gradient term at 20~Hz control frequency. Second, closed-loop stability depends on feedback gains and disturbances beyond $\Phi_{\text{energy}}$ alone. Third, the heuristic is most effective when energy efficiency correlates with task success (locomotion, vehicle stability) and provides limited benefit when they conflict (aggressive maneuvers). Fourth, actuator saturation or external disturbances alter $\dot{E}$ from its nominal value, potentially disrupting the dissipation property locally; in the RL+MPC hierarchical architecture (Section~\ref{sec:vehicle}), the lower MPC layer enforces hard constraints and partially restores dissipation guarantees under disturbance, which is why the Lyapunov interpretation holds empirically in TruckSim.
	\end{proof}
	
	\subsection{Actor-Critic Integration with H-EARS}
	\label{sec:ac_integration}
	
	H-EARS integrates with actor-critic algorithms by modifying the reward signal used in temporal difference (TD) learning, while preserving the original algorithmic structure. The core integration modifies the critic's TD target:
	\begin{equation}
		y_t = R_{\text{H-EARS}}(s_t, a_t, s_{t+1}) + \gamma V_{\theta'}(s_{t+1})
		\label{eq:ac_td_target}
	\end{equation}
	where $V_{\theta'}(s_{t+1})$ represents the target value estimate (Q-function for off-policy methods, value baseline for on-policy methods), and the shaped reward is:
	\begin{equation}
		\begin{split}
			R_{\text{H-EARS}}(s_t, a_t, s_{t+1}) = \; &r_t + \gamma\Phi(s_{t+1}) - \Phi(s_t) - \lambda\mathcal{E}(a_t) \\
			\Phi(s) = \; &\alpha_{\text{task}}\Phi_{\text{task}}(s) + \alpha_{\text{energy}}\Phi_{\text{energy}}(s)
		\end{split}
		\label{eq:hears_reward}
	\end{equation}
	
	The potential function $\Phi(s)$ and action regularization $\mathcal{E}(a) = a^\top Q a$ are computed during environment interaction and stored alongside standard transition tuples. Critically, the actor update remains unchanged:
	\begin{equation}
		\nabla_\phi J(\phi) = \mathbb{E}_{(s,a) \sim \mathcal{D}}\left[\nabla_\phi \log \pi_\phi(a|s) \cdot A(s,a)\right]
		\label{eq:actor_gradient}
	\end{equation}
	where the advantage function $A(s,a)$ is computed from Q-values or value estimates trained on $R_{\text{H-EARS}}$. This separation ensures H-EARS guidance propagates through learned value functions rather than direct policy modifications, preserving algorithmic properties such as entropy regularization, trust regions, or deterministic gradients. 
	
	Algorithm~\ref{alg:hears_ac} presents the general integration template. Lines 11-15 constitute H-EARS-specific computations; all other components follow standard AC training procedures.
	
	\begin{algorithm}[htbp]
		\caption{Actor-Critic Algorithm with H-EARS Integration}
		\label{alg:hears_ac}
		\begin{algorithmic}[1]
			\STATE \textbf{Input}: Environment $\mathcal{M}$, H-EARS hyperparameters $\{\alpha_{\text{task}}, \alpha_{\text{energy}}, \lambda\}$
			\STATE \textbf{Initialize}: Actor network $\pi_\phi$, Critic network $Q_\theta$ (or $V_\theta$)
			\STATE \textbf{Initialize}: Target networks $\pi_{\phi'}$, $Q_{\theta'}$ (if off-policy)
			\STATE \textbf{Initialize}: Experience buffer $\mathcal{D}$
			\FOR{episode $= 1, 2, \ldots$}
			\STATE Observe initial state $s_0$
			\FOR{$t = 0, 1, \ldots, T_{\max}$}
			\STATE Sample action: $a_t \sim \pi_\phi(\cdot | s_t)$
			\STATE Execute action, observe reward $r_t$ and next state $s_{t+1}$
			\STATE \textcolor{blue}{// Computed shaping after observing $s_{t+1}$}
			\STATE Compute potential at current state: $\Phi(s_t) = \alpha_{\text{task}}\Phi_{\text{task}}(s_t) + \alpha_{\text{energy}}\Phi_{\text{energy}}(s_t)$
			\STATE Compute potential at next state: $\Phi(s_{t+1}) = \alpha_{\text{task}}\Phi_{\text{task}}(s_{t+1}) + \alpha_{\text{energy}}\Phi_{\text{energy}}(s_{t+1})$
			\STATE Compute action regularization: $\mathcal{E}(a_t) = a_t^\top Q a_t$
			\STATE Compute shaped reward for transition $(s_t, a_t, s_{t+1})$: 
			\STATE \quad $R_{\text{H-EARS}} = r_t + \gamma\Phi(s_{t+1}) - \Phi(s_t) - \lambda\mathcal{E}(a_t)$
			\STATE 
			\STATE Store transition: $\mathcal{D} \leftarrow \mathcal{D} \cup \{(s_t, a_t, R_{\text{H-EARS}}, s_{t+1})\}$
			\STATE 
			\IF{update condition satisfied}
			\STATE Sample minibatch $\mathcal{B} \sim \mathcal{D}$
			\STATE \textcolor{blue}{// Critic update}
			\FOR{$(s, a, R_{\text{H-EARS}}, s') \in \mathcal{B}$}
			\STATE Compute TD target: $y = R_{\text{H-EARS}} + \gamma V_{\theta'}(s')$
			\STATE Update critic: $\theta \leftarrow \theta - \eta_Q \nabla_\theta (Q_\theta(s,a) - y)^2$
			\ENDFOR
			\STATE \textcolor{blue}{// Actor update}
			\STATE Compute advantages $A(s,a)$ from updated Q-values
			\STATE Update actor: $\phi \leftarrow \phi + \eta_\pi \nabla_\phi J(\phi)$ \textcolor{gray}{// Eq.~\eqref{eq:actor_gradient}}
			\STATE 
			\STATE Update target networks: $\theta' \leftarrow \tau\theta + (1-\tau)\theta'$, $\phi' \leftarrow \tau\phi + (1-\tau)\phi'$
			\ENDIF
			\ENDFOR
			\ENDFOR
			\STATE \textbf{Output}: Trained policy $\pi_\phi$
		\end{algorithmic}
	\end{algorithm}
	
	Figure~\ref{fig:framework} provides a visual overview of the construction pipeline.
	
	\begin{figure*}[!t]
		\centering
		\includegraphics[width=0.86\textwidth]{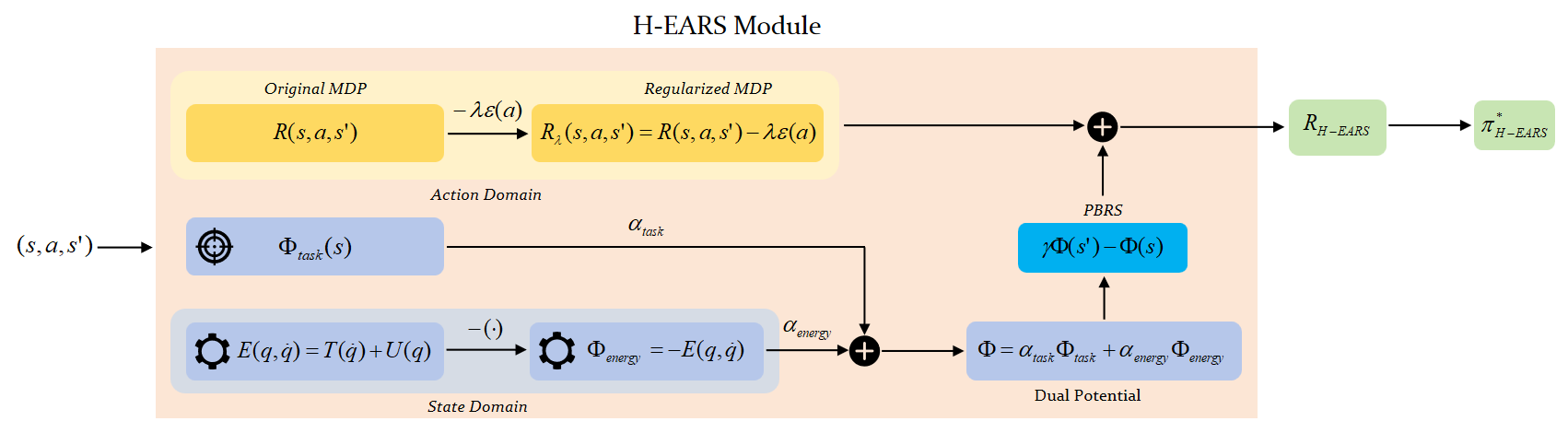}
		\caption{H-EARS construction pipeline. Left branch (action domain): action regularization defines the regularized MDP, intentionally modifying the optimization objective to trade task performance for energy efficiency (Proposition~\ref{prop:regularization_necessity}). Right branch (state domain): physical energy knowledge and task specification are combined into the dual potential $\Phi$; gradient enrichment under locally positive-definite $\partial^2 E/\partial q^2$ is established in Theorem~\ref{thm:energy_convergence}. Potential-based reward shaping (PBRS) is applied to regularized MDP using $\Phi$, yielding $R_{\text{H-EARS}}$. Since the two branches act on disjoint argument domains, $\alpha_{\text{task}}$, $\alpha_{\text{energy}}$, and $\lambda$ serve as independently tunable design parameters (Lemma~\ref{lem:functional_independence}).}
		\label{fig:framework}
	\end{figure*}
	
	\section{Standard Environment Experiments}
	
	\subsection{Baseline Comparison}
	
	The experimental validation employs four standard continuous control environments from Gymnasium: Ant-v5 (243-dimensional observation, 8-dimensional action), Hopper-v5 (11-dimensional observation, 3-dimensional action), LunarLander-v3 (8-dimensional observation, 2-dimensional action), and Humanoid-v5 (348-dimensional observation, 17-dimensional action). These environments span diverse complexity levels and physical characteristics, enabling comprehensive framework evaluation.
	
	\begin{figure*}[!t]
		\centering
		\subfigure[DDPG in Ant-v5]{
			\includegraphics[width=0.23\textwidth]{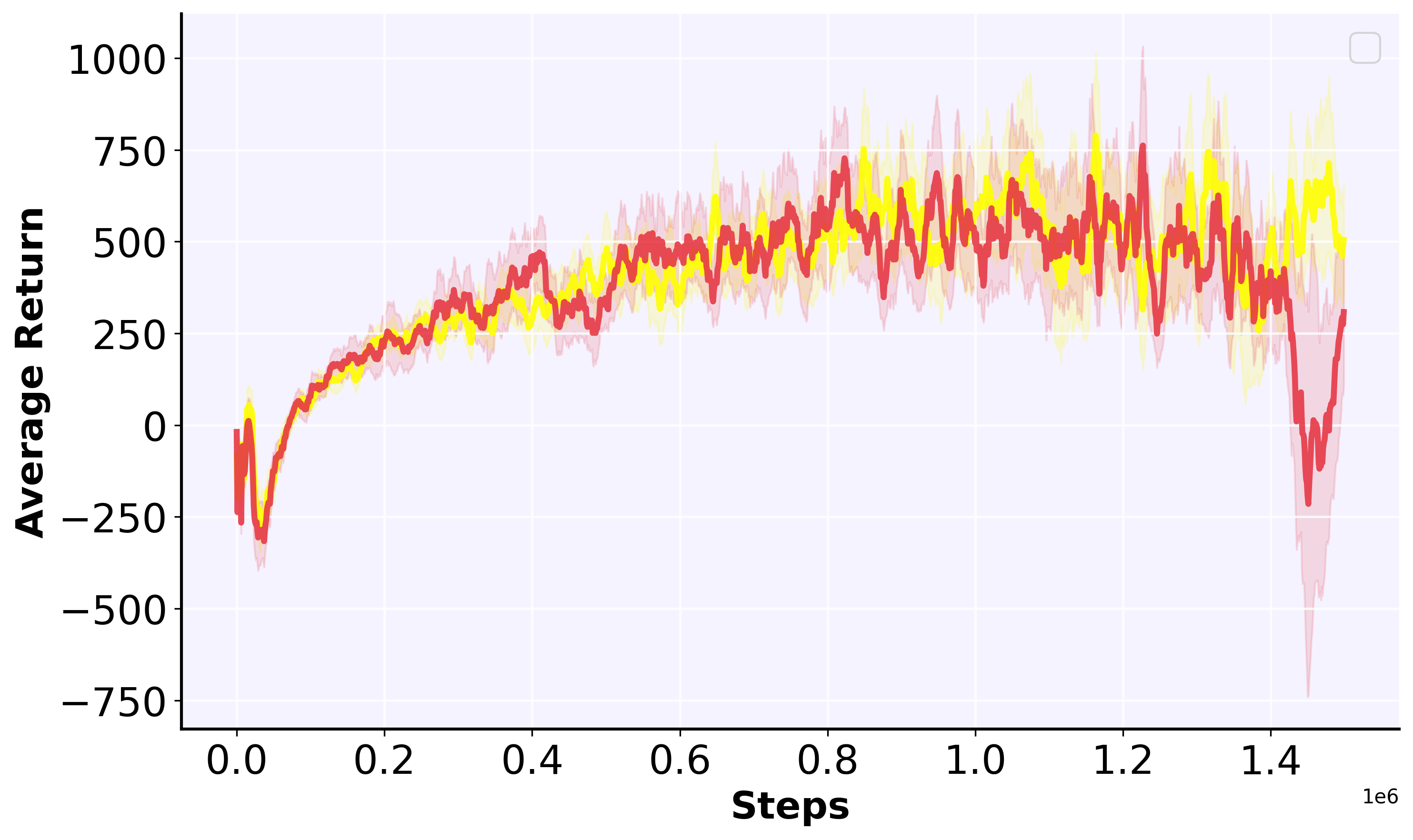}
			\label{fig:DDPG_ant}
		}
		\hspace{-0.2cm}
		\subfigure[PPO in Ant-v5]{
			\includegraphics[width=0.23\textwidth]{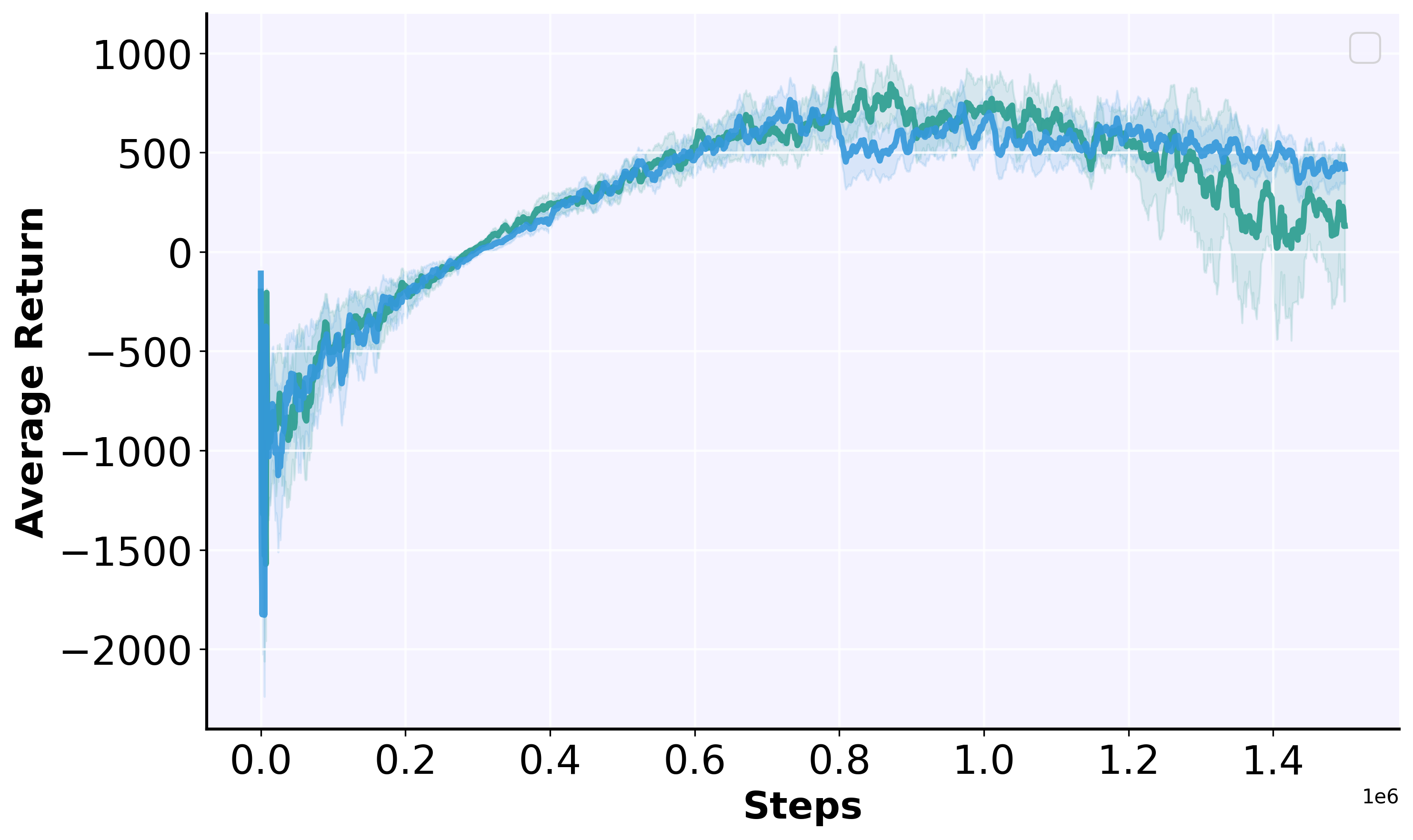}
			\label{fig:PPO_ant}
		}
		\hspace{-0.2cm}
		\subfigure[TD3 in Ant-v5]{
			\includegraphics[width=0.23\textwidth]{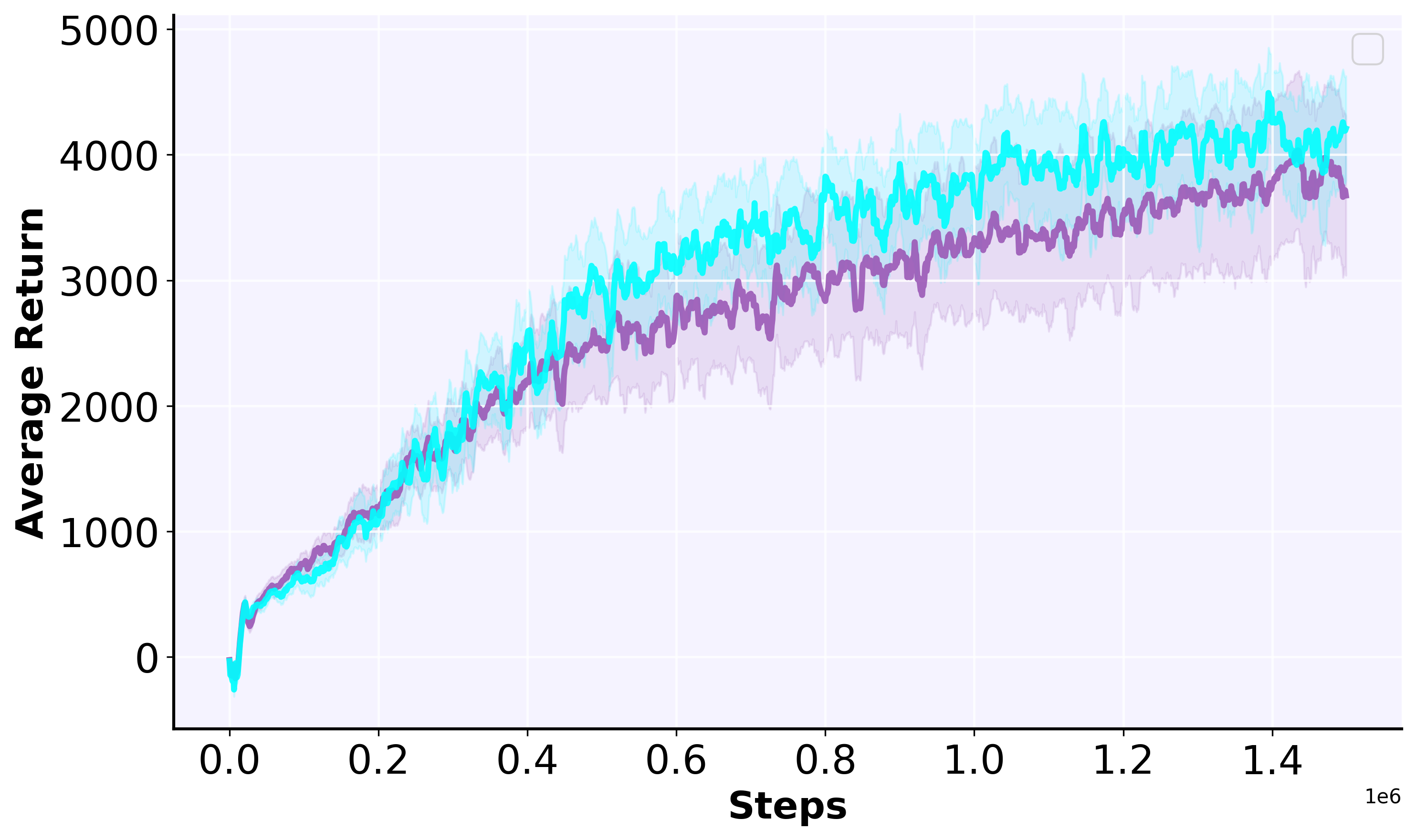}
			\label{fig:TD3_ant}
		}
		\hspace{-0.2cm}
		\subfigure[SAC in Ant-v5]{
			\includegraphics[width=0.23\textwidth]{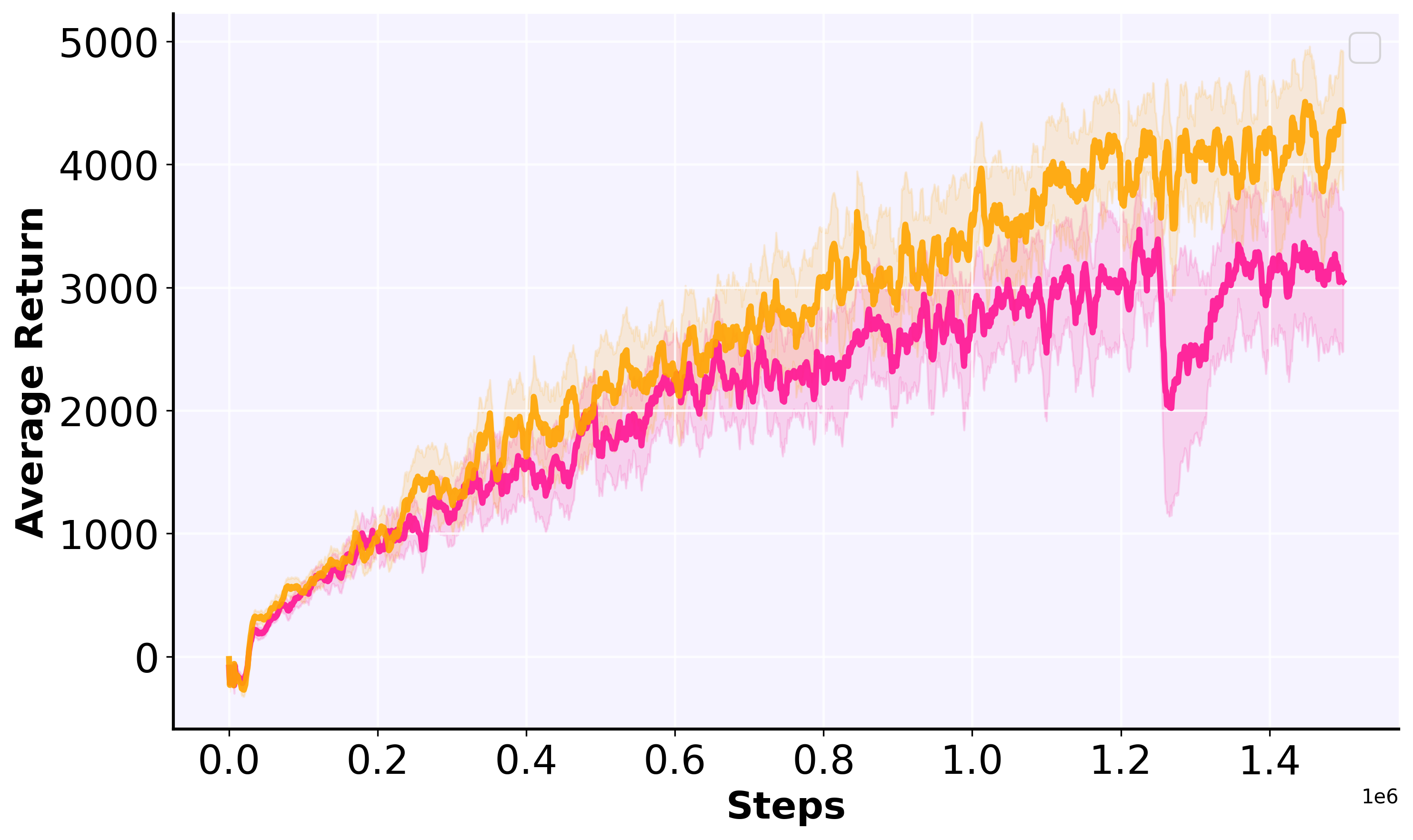}
			\label{fig:SAC_ant}
		}
		
		\subfigure[DDPG in Humanoid-v5]{
			\includegraphics[width=0.23\textwidth]{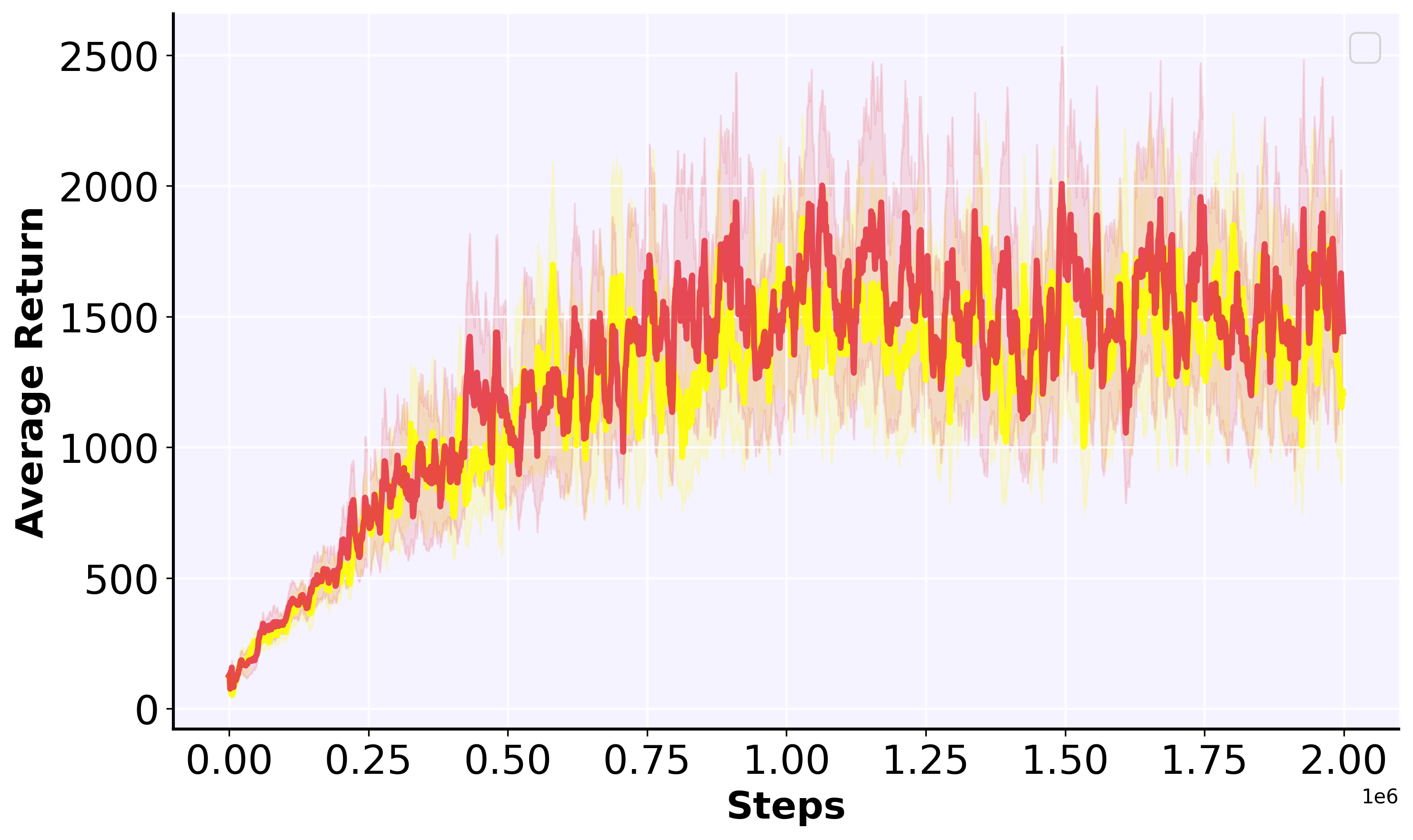}
			\label{fig:DDPG_humanoid}
		}
		\hspace{-0.2cm}
		\subfigure[PPO in Humanoid-v5]{
			\includegraphics[width=0.23\textwidth]{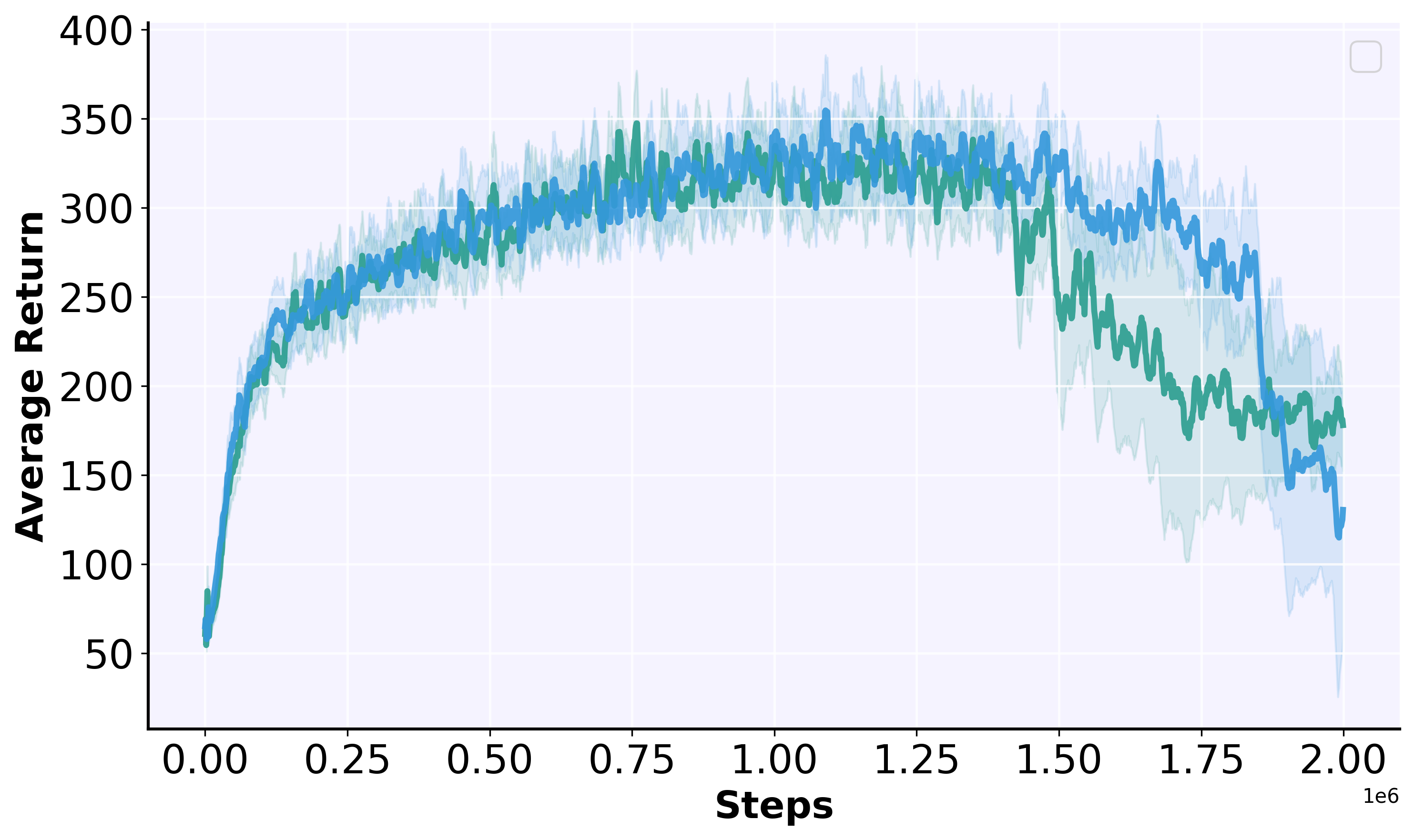}
			\label{fig:PPO_humanoid}
		}
		\hspace{-0.2cm}
		\subfigure[TD3 in Humanoid-v5]{
			\includegraphics[width=0.23\textwidth]{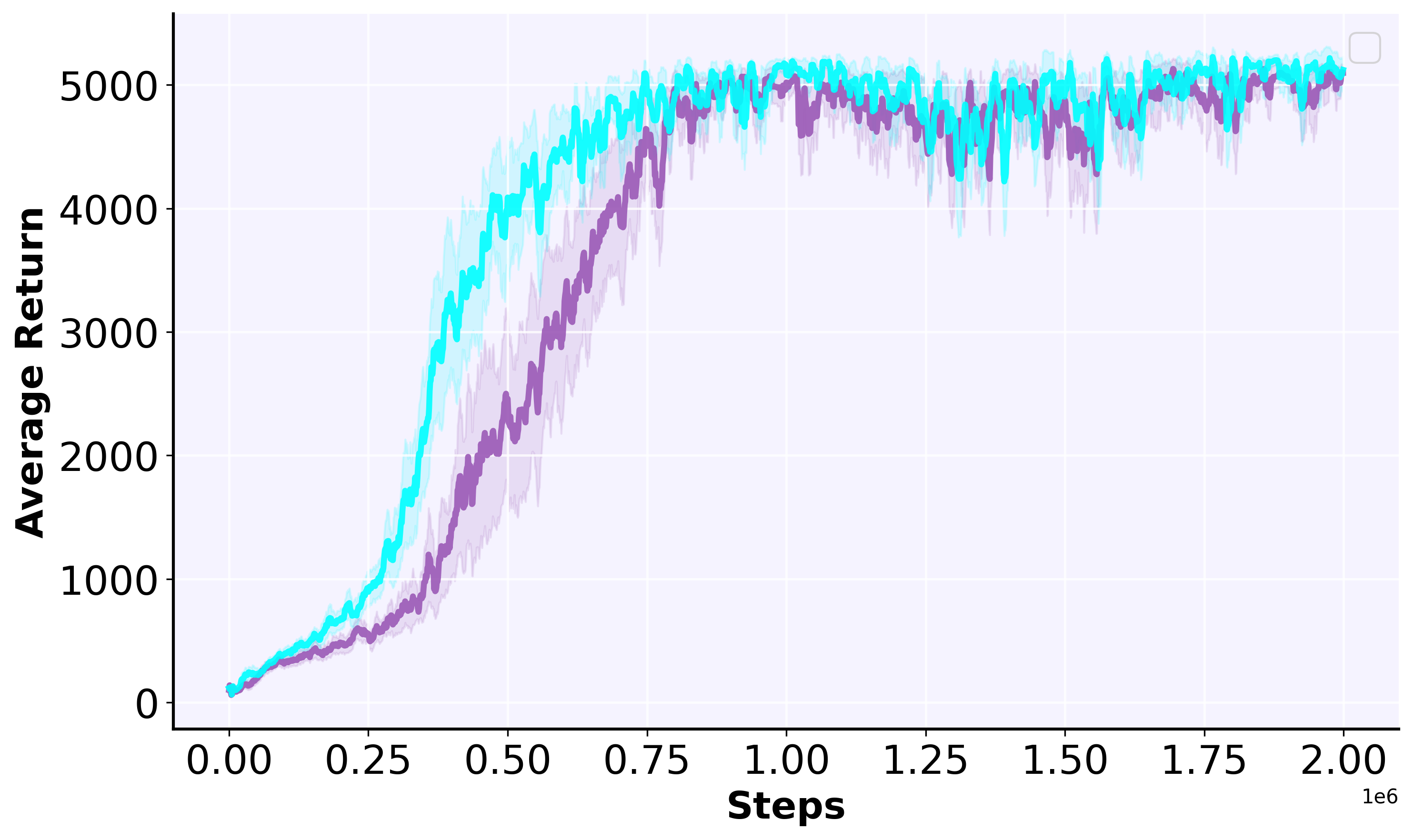}
			\label{fig:TD3_humanoid}
		}
		\hspace{-0.2cm}
		\subfigure[SAC in Humanoid-v5]{
			\includegraphics[width=0.23\textwidth]{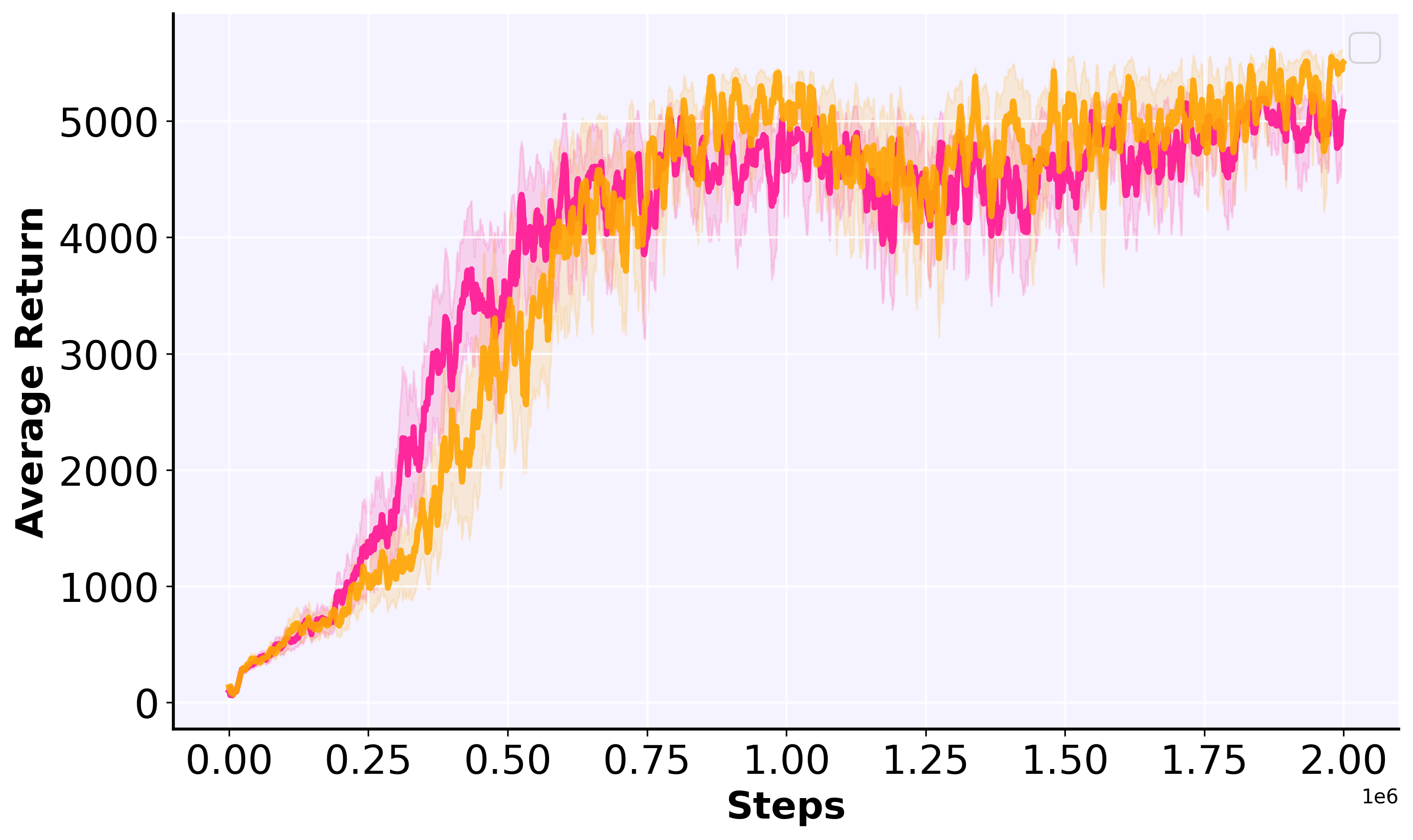}
			\label{fig:SAC_humanoid}
		}
		
		\subfigure[DDPG in Hopper-v5]{
			\includegraphics[width=0.23\textwidth]{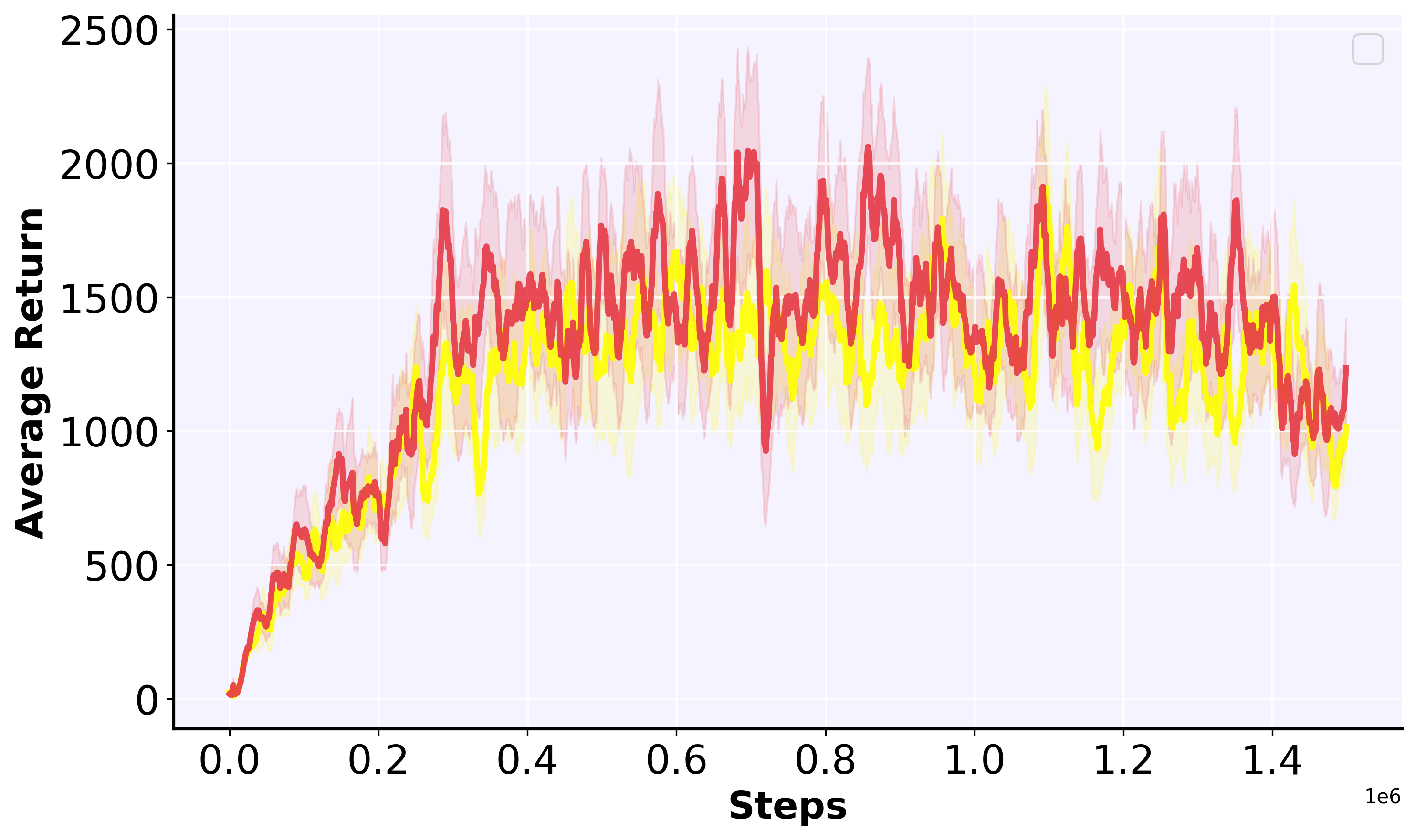}
			\label{fig:DDPG_hopper}
		}
		\hspace{-0.2cm}
		\subfigure[PPO in Hopper-v5]{
			\includegraphics[width=0.23\textwidth]{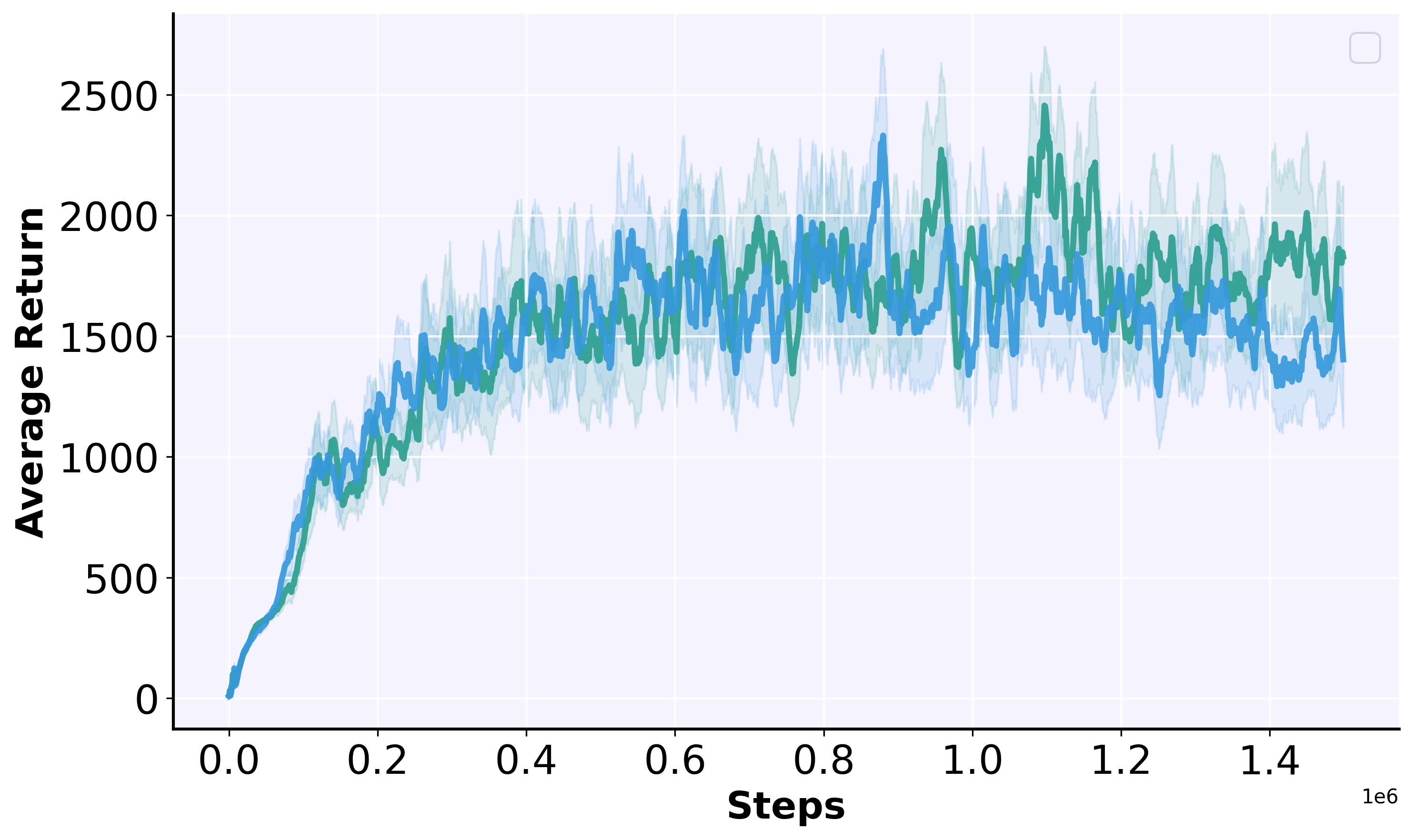}
			\label{fig:PPO_hopper}
		}
		\hspace{-0.2cm}
		\subfigure[TD3 in Hopper-v5]{
			\includegraphics[width=0.23\textwidth]{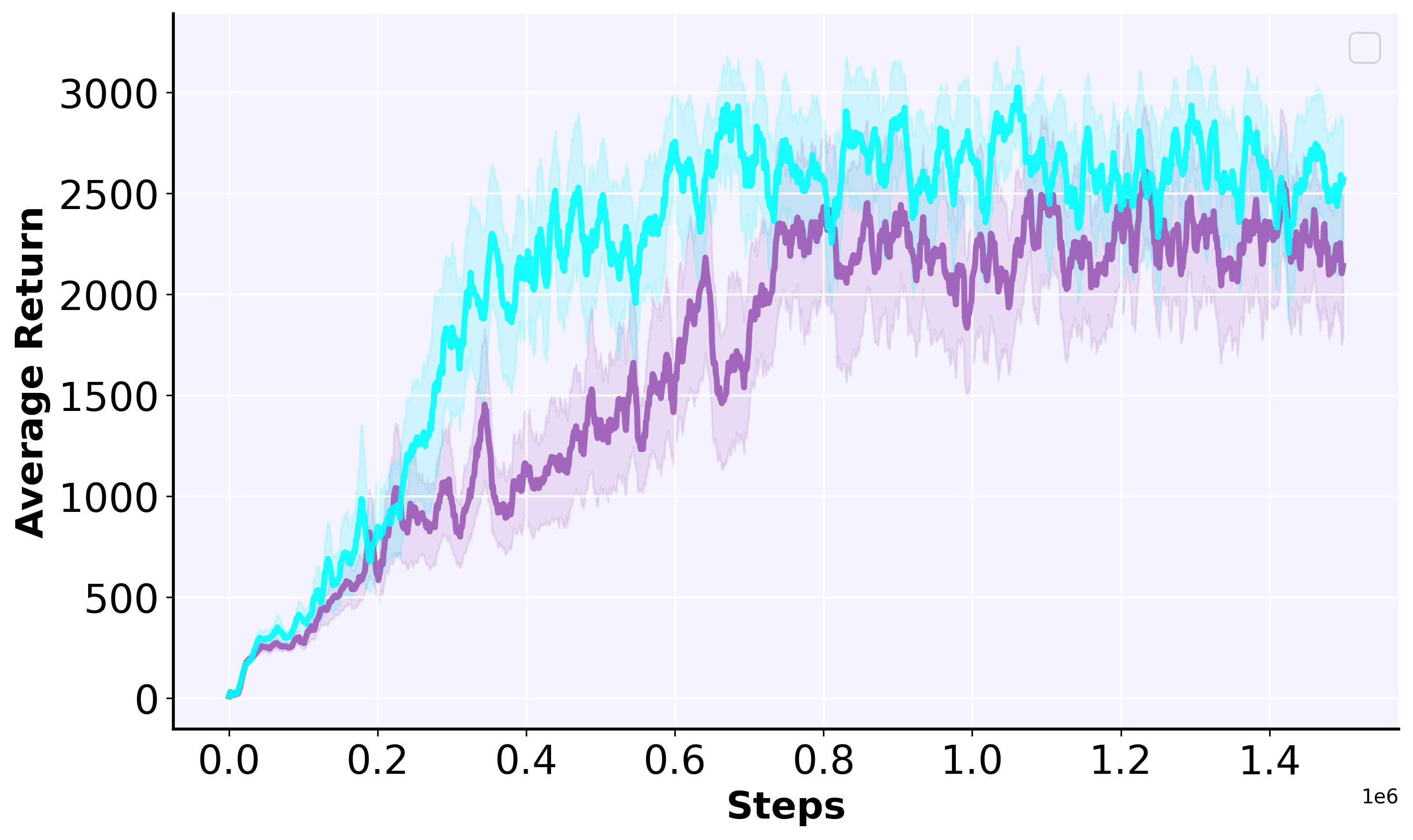}
			\label{fig:TD3_hopper}
		}
		\hspace{-0.2cm}
		\subfigure[SAC in Hopper-v5]{
			\includegraphics[width=0.23\textwidth]{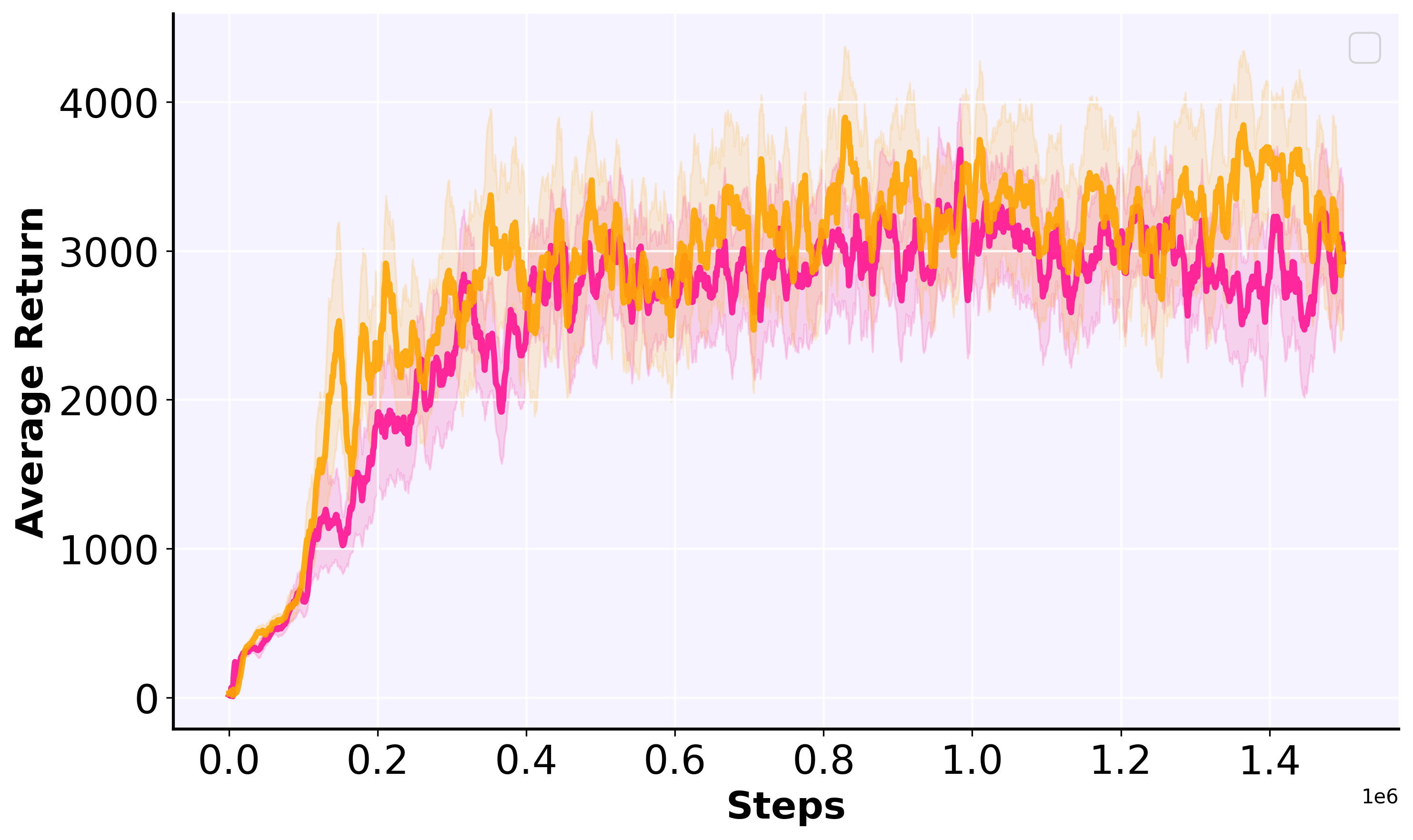}
			\label{fig:SAC_hopper}
		}
		
		\subfigure[DDPG in Lunarlander-v3]{
			\includegraphics[width=0.23\textwidth]{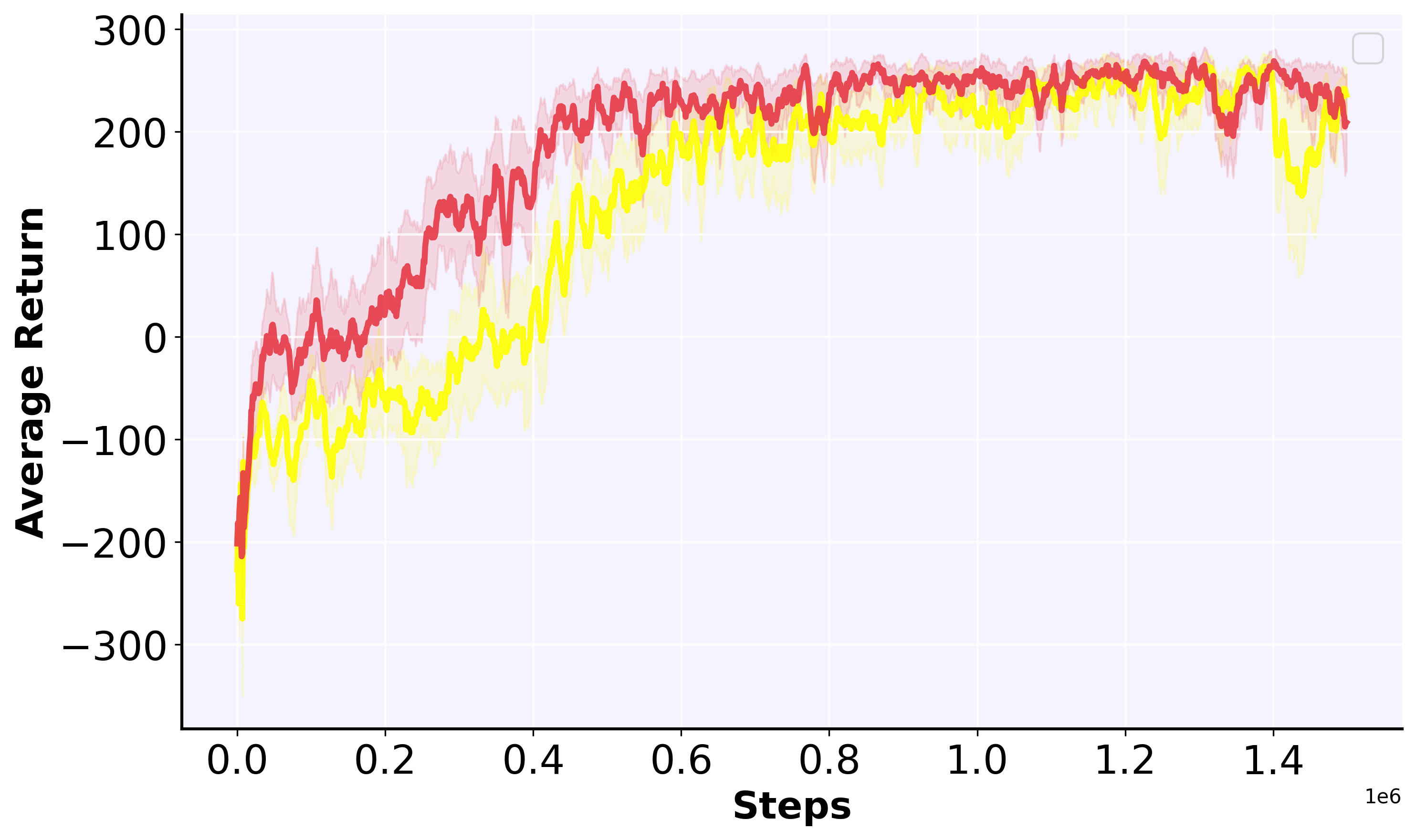}
			\label{fig:DDPG_lunarlander}
		}
		\hspace{-0.2cm}
		\subfigure[PPO in Lunarlander-v3]{
			\includegraphics[width=0.23\textwidth]{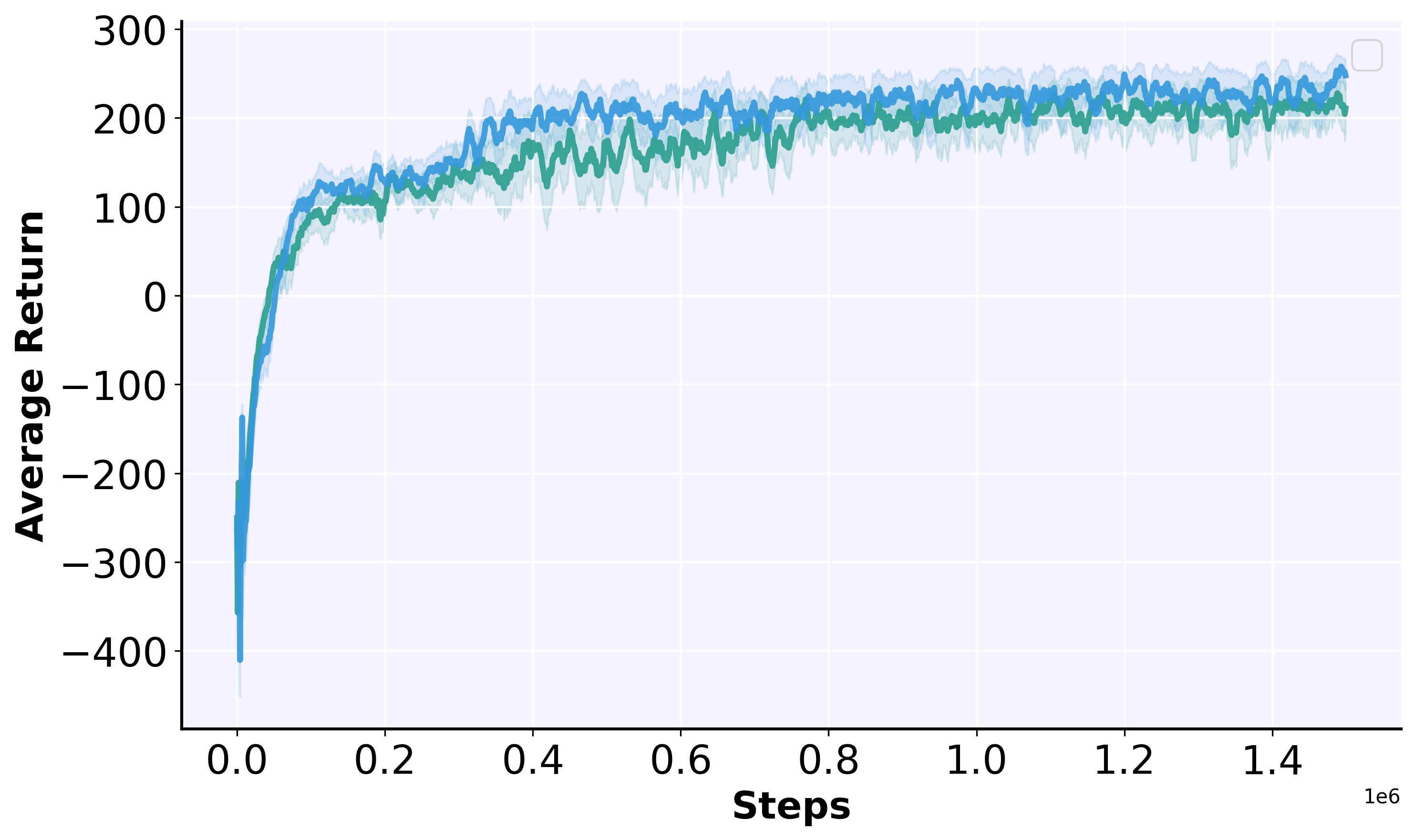}
			\label{fig:PPO_lunarlander}
		}
		\hspace{-0.2cm}
		\subfigure[TD3 in Lunarlander-v3]{
			\includegraphics[width=0.23\textwidth]{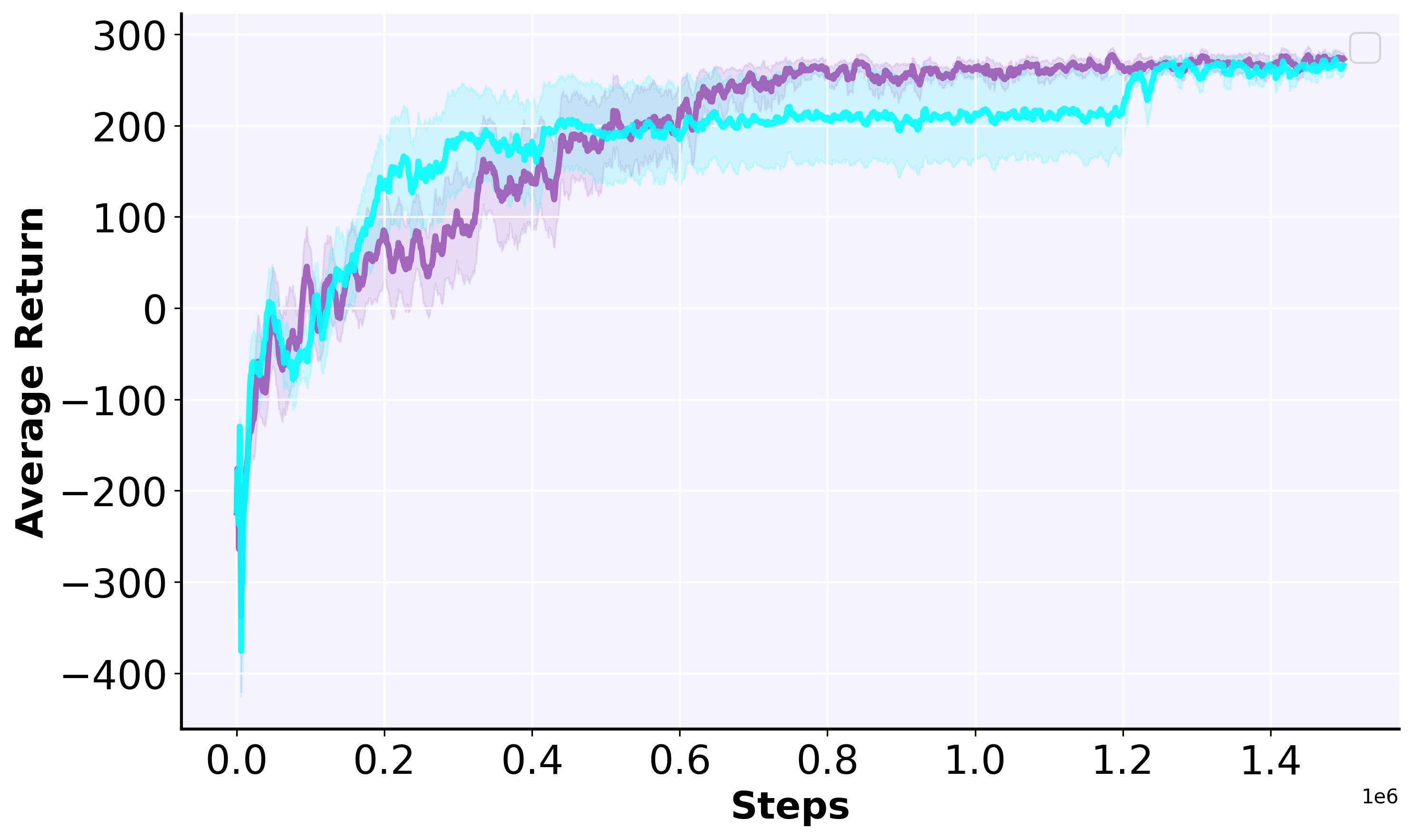}
			\label{fig:TD3_lunarlander}
		}
		\hspace{-0.2cm}
		\subfigure[SAC in Lunarlander-v3]{
			\includegraphics[width=0.23\textwidth]{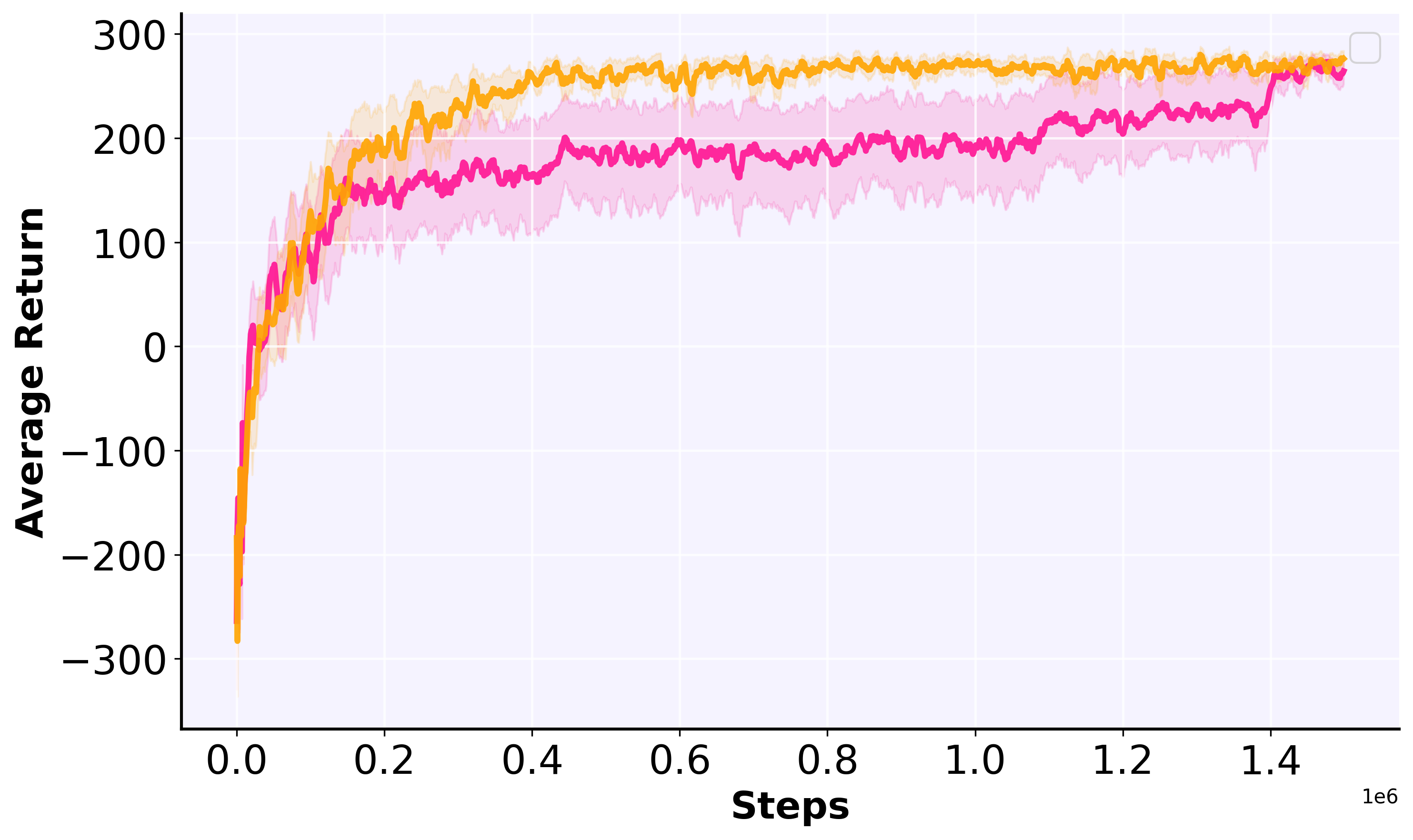}
			\label{fig:SAC_lunarlander}
		}
		
		\vspace{0.15cm}
		\centering
		\begin{tabular}{@{}l@{\hspace{1.5em}}l@{\hspace{1.5em}}l@{\hspace{1.5em}}l@{}}
			\colorbox{DDPG}{\rule{0pt}{1pt}\rule{8pt}{0pt}} \raisebox{-2.0pt}{\scriptsize DDPG-Vanilla} &
			\colorbox{PPO}{\rule{0pt}{1pt}\rule{8pt}{0pt}} \raisebox{-2.0pt}{\scriptsize PPO-Vanilla} &
			\colorbox{TD3}{\rule{0pt}{1pt}\rule{8pt}{0pt}} \raisebox{-2.0pt}{\scriptsize TD3-Vanilla} &
			\colorbox{SAC}{\rule{0pt}{1pt}\rule{8pt}{0pt}} \raisebox{-2.0pt}{\scriptsize SAC-Vanilla} \\[0.2em]
			\colorbox{DDPGHEARS}{\rule{0pt}{1pt}\rule{8pt}{0pt}} \raisebox{-2.0pt}{\scriptsize DDPG-H-EARS} &
			\colorbox{PPOHEARS}{\rule{0pt}{1pt}\rule{8pt}{0pt}} \raisebox{-2.0pt}{\scriptsize PPO-H-EARS} &
			\colorbox{TD3HEARS}{\rule{0pt}{1pt}\rule{8pt}{0pt}} \raisebox{-2.0pt}{\scriptsize TD3-H-EARS} &
			\colorbox{SACHEARS}{\rule{0pt}{1pt}\rule{8pt}{0pt}} \raisebox{-2.0pt}{\scriptsize SAC-H-EARS}
		\end{tabular}
		
		\caption{Benchmark performance comparison across four environments and four baseline algorithms. SAC and TD3 results validate Theorem~\ref{thm:energy_convergence} (energy-based gradient enrichment under positive-definite Hessian accelerates convergence) and Lemma~\ref{lem:functional_independence} (independent task and energy contributions). DDPG degradation in Ant-v5 instantiates the oscillatory-artifact failure mode analyzed in Proposition~\ref{prop:regularization_necessity}: DDPG's OU exploration is suppressed by energy regularization, creating local-minima traps in the 8-DOF shaped landscape.}
		\label{fig:standard_envs_benchmark}
	\end{figure*}
	
	All experiments use five fixed random seeds \{12345, 22345, 32345, 42345, 52345\} (Table~\ref{tab:hyperparams_standard}). Training runs on an NVIDIA RTX A5000 (24\,GB) with CUDA 11.8, Python 3.10, PyTorch 1.13. Full training scripts, H-EARS energy module implementations for all four environments, and evaluation pipelines are publicly available at \url{https://github.com/QiJLiao/Hybrid_Energy-Aware_Reward_Shaping} and provided in the supplementary materials. Eight algorithm configurations are evaluated: H-EARS+SAC, H-EARS+TD3, H-EARS+PPO, H-EARS+DDPG, and their vanilla counterparts (SAC, TD3, PPO, DDPG). All baseline implementations follow Stable-Baselines3 defaults (detailed in Table~\ref{benchmark_training_details}). 
	
	\textbf{Baseline fairness.} All baseline algorithms use Stable-Baselines3 defaults without joint tuning with H-EARS coefficients, ensuring H-EARS is evaluated as an add-on to unmodified baselines. Compute budgets are matched step-for-step; H-EARS per-step overhead is below 3\% due to the $O(n)$ energy computation. Exploration schedules are preserved: SAC's entropy coefficient, TD3's target policy smoothing noise, PPO's clipping ratio, and DDPG's Ornstein-Uhlenbeck noise parameters all remain at Stable-Baselines3 defaults. H-EARS hyperparameters $\{\alpha_{\text{task}}, \alpha_{\text{energy}}, \lambda\}$ are selected by grid search on a held-out validation seed (excluded from the five reported seeds) and reported in Table~\ref{tab:hyperparams_standard}. The degradation of DDPG+H-EARS on Ant-v5 (610$\to$456) is mechanistically predicted by the exploration-regularization conflict analyzed in Section~\ref{sec:ablation}: reducing $\lambda$ to zero for DDPG/Ant recovers vanilla performance, but this per-baseline exemption would undermine the general-purpose nature of the framework. This interaction is therefore reported as a principled applicability boundary rather than a tuning deficiency. 
	
	\textbf{Environment reward configuration.} All Gymnasium environments are run with their default reward configurations, which include built-in action cost terms. H-EARS' action regularization term $-\lambda\mathcal{E}(a)$ is therefore additive to these existing penalties. The distinction from simple L2 augmentation is threefold: (i) $\mathcal{E}(a)=a^\top Qa$ with a physics-motivated $Q$ acts on the same quadratic form but is coupled to the energy potential through the two-step MDP transformation (Section~\ref{sec:theoretical_framework}), providing theoretical coherence absent in ad-hoc regularization; (ii) the ablation study's Reg.~Only variant (Section~\ref{sec:ablation}), which applies $\lambda\mathcal{E}(a)$ alone without energy potentials, underperforms H-EARS by 15.8\% (Ant) and 14.1\% (Hopper), confirming that the gains are not attributable to regularization alone; (iii) the $\lambda$ values used (1e-2 to 1e-4) are one to three orders of magnitude smaller than the default cost weights, ensuring H-EARS does not simply dominate the existing penalty structure.
	
	Policy stability quantified via coefficient of variation (CV): $\text{CV} = \frac{\sigma}{\mu} \times 100\%$. Average returns across four environments shown in Figure~\ref{fig:standard_envs_benchmark}, with comprehensive results in Table~\ref{tab:comprehensive_results}: (a) final average returns, (b) episodes to performance threshold, (c) post-convergence CV.
	
	\begin{table*}[!t]
		\renewcommand{\arraystretch}{1.2}
		\centering
		\captionsetup{font=footnotesize, labelfont=footnotesize}
		\caption{Comprehensive Performance Comparison across Standard Environments}
		\label{tab:comprehensive_results}
		
		\newlength{\subtablewidth}
		\setlength{\subtablewidth}{0.7\textwidth}

		\textbf{(a) Average Returns (Mean $\pm$ Std)}\\[1mm]
		
		\begin{tabular*}{\subtablewidth}{@{\extracolsep{\fill}}lcccc@{}}
			\toprule
			\textbf{Algorithm} & \textbf{Ant-v5} & \textbf{Hopper-v5} & \textbf{LunarLander-v3} & \textbf{Humanoid-v5} \\
			\midrule
			SAC - Vanilla & $3157 \pm 182$ & $2520 \pm 405$ & $268 \pm 30$ & $4988 \pm 502$ \\
			SAC - H-EARS & $\mathbf{4183 \pm 174}$ & $\mathbf{3354 \pm 354}$ & $\mathbf{289 \pm 19}$  & $\mathbf{5228 \pm 470}$ \\
			\cmidrule(lr){1-5}
			TD3 - Vanilla & $3570 \pm 391$ & $2364 \pm 302$ & $\mathbf{279 \pm 15}$ & $4964 \pm 357$ \\
			TD3 - H-EARS & $\mathbf{4125 \pm 398}$ & $\mathbf{2641 \pm 298}$ & $277 \pm 9$ & $\mathbf{5190 \pm 325}$ \\
			\cmidrule(lr){1-5}
			PPO - Vanilla & $376 \pm 281$ & $\mathbf{1621 \pm 255}$ & $235 \pm 24$  & $\mathbf{197 \pm 89}$ \\
			PPO - H-EARS & $\mathbf{501 \pm 127}$ & $1568 \pm 203$ & $\mathbf{258 \pm 14}$ & $178 \pm 120$ \\
			\cmidrule(lr){1-5}
			DDPG - Vanilla & $\mathbf{610 \pm 231}$  & $1490 \pm 300$ & $231 \pm 30$ & $1524 \pm 488$ \\
			DDPG - H-EARS & $456 \pm 422$ & $\mathbf{1580 \pm 310}$ & $\mathbf{250 \pm 21}$ & $\mathbf{1620 \pm 523}$ \\
			\bottomrule
		\end{tabular*}
		
		\vspace{4mm}
		
		\textbf{(b) Training Episodes Required to Reach Performance Threshold}\\[1mm]
		
		\begin{tabular*}{\subtablewidth}{@{\extracolsep{\fill}}lcccc@{}}
			\toprule
			\textbf{Algorithm} & \textbf{Ant-v5 (2500)} & \textbf{Hopper-v5 (1500)} & \textbf{LunarLander-v3 (200)} & \textbf{Humanoid-v5 (4000)} \\
			\midrule
			SAC - Vanilla & $1240 \pm 120$& $1050 \pm 150$& $620 \pm 60$ & $\mathbf{2080 \pm 200}$ \\
			SAC - H-EARS & $\mathbf{890 \pm 70}$  & $\mathbf{830 \pm 65}$  & $\mathbf{290 \pm 35}$ & $2120 \pm 170$ \\
			\cmidrule(lr){1-5}
			TD3 - Vanilla & $1320 \pm 130$& $1470 \pm 120$& $540 \pm 55$ & $2850 \pm 220$ \\
			TD3 - H-EARS & $\mathbf{980 \pm 80}$  & $\mathbf{910 \pm 75}$  & $\mathbf{350 \pm 40}$ & $\mathbf{2220 \pm 180}$ \\
			\cmidrule(lr){1-5}
			PPO - Vanilla & -- & $\mathbf{1090 \pm 110}$& $780 \pm 80$ & -- \\
			PPO - H-EARS & -- & $1120 \pm 90$ & $\mathbf{560 \pm 60}$ & -- \\
			\cmidrule(lr){1-5}
			DDPG - Vanilla & -- & $1640 \pm 170$& $690 \pm 90$ & -- \\
			DDPG - H-EARS & -- & $\mathbf{1230 \pm 110}$& $\mathbf{570 \pm 70}$ & -- \\
			\bottomrule
		\end{tabular*}
		
		\vspace{4mm}

		\textbf{(c) Post-Convergence Performance Coefficient of Variation (\%)}\\[1mm]
		
		\begin{tabular*}{\subtablewidth}{@{\extracolsep{\fill}}lcccc@{}}
			\toprule
			\textbf{Algorithm} & \textbf{Ant-v5} & \textbf{Hopper-v5} & \textbf{LunarLander-v3} & \textbf{Humanoid-v5} \\
			\midrule
			SAC - Vanilla & $5.8$ & $16.1$ & $11.2$  & $10.1$ \\
			SAC - H-EARS & $\mathbf{4.2}$ & $\mathbf{10.6}$ & $\mathbf{6.6}$   & $\mathbf{9.0}$  \\
			\cmidrule(lr){1-5}
			TD3 - Vanilla & $11.0$ & $12.8$ & $5.4$   & $7.2$ \\
			TD3 - H-EARS & $\mathbf{9.6}$ & $\mathbf{11.3}$ & $\mathbf{3.2}$   & $\mathbf{6.3}$  \\
			\cmidrule(lr){1-5}
			PPO - Vanilla & $74.7$ & $15.7$ & $10.2$ & $\mathbf{45.2}$ \\
			PPO - H-EARS & $\mathbf{25.3}$ & $\mathbf{12.9}$ & $\mathbf{5.4}$   & $67.4$ \\
			\cmidrule(lr){1-5}
			DDPG - Vanilla & $\mathbf{37.9}$ & $20.1$ & $13.0$  & $\mathbf{32.0}$ \\
			DDPG - H-EARS & $92.5$ & $\mathbf{19.6}$ & $\mathbf{8.4}$   & $32.3$ \\
			\bottomrule
		\end{tabular*}
		
		\vspace{2mm}
		\footnotesize
		\textbf{Notes:} Bold values indicate best performance. ``--'' denotes failure to reach threshold.
	\end{table*}
	
	\textbf{SAC Integration}: SAC+H-EARS achieves the most consistent improvements. Ant-v5's 32.5\% performance gain (3157→4183) with 28.2\% convergence acceleration (1240→890 episodes) and 27.6\% variance reduction (CV: 5.8\%→4.2\%) validates Theorem~\ref{thm:energy_convergence}: energy potential's positive definite Hessian reshapes the value landscape, enabling faster convergence through improved gradient quality. LunarLander-v3 exhibits strongest impact—53.3\% faster convergence with 41.1\% variance reduction (CV: 11.2\%→6.6\%)—demonstrating energy-based constraints' natural alignment with geometric navigation tasks where mechanical energy defines optimal trajectories. Hopper-v5 shows 33.1\% improvement (2520→3354), confirming energy potential's stabilization effect in inherently unstable systems. Humanoid-v5's modest 4.8\% gain reflects Lemma~\ref{lem:approx_potential}: simplified O(n) energy models capture sufficient dominant dynamics in high-dimensional spaces, primarily enhancing stability (CV: 10.1\%→9.0\%) rather than asymptotic performance.
	
	\textbf{TD3 Integration}: TD3+H-EARS exhibits environment-selective benefits. Ant-v5 achieves 15.5\% improvement (3570→4125) with 25.8\% faster convergence, while Hopper-v5 gains 11.7\% (2364→2641). However, LunarLander-v3's minimal change (279→277) reveals critical insight: TD3's deterministic policy gradient inherently incorporates smoothness through target network updates, saturating marginal benefit from additional energy constraints. The retained 11.8\% variance reduction (CV: 12.8\%→11.3\%) in Hopper-v5 confirms energy potentials provide stability benefits even when performance gains saturate, validating Lemma~\ref{lem:functional_independence}'s independent contribution principle.
	
	\textbf{PPO Integration}: PPO+H-EARS reveals complex stability-performance tradeoffs. Ant-v5's dramatic rescue from collapse (376→501, CV: 74.7\%→25.3\%) validates Theorem~\ref{prop:regularization_necessity}: energy potential's implicit bounds prevent catastrophic divergence when dual potentials generate competing gradients. LunarLander-v3 shows 9.8\% improvement with 47.1\% variance reduction. However, Hopper-v5's regression (1621→1568) exposes fundamental limitation: PPO's clipped objective conflicts with energy shaping when rapid policy adjustments are required for dynamic stabilization. Humanoid-v5's delayed collapse (Vanilla fails at 1.0M steps, H-EARS at 1.5M) with lower final return (197→178) demonstrates energy constraints slow but cannot prevent instability propagation in critically unstable systems, consistent with Proposition~\ref{prop:lyapunov_heuristic}'s discrete-time approximation limits.
	
	\textbf{DDPG Integration}: DDPG+H-EARS establishes framework limitations. LunarLander-v3 achieves 8.2\% improvement with 35.4\% variance reduction (CV: 13.0\%→8.4\%), while Hopper-v5 gains 6.0\%. Yet Ant-v5's severe degradation (610$\to$456, CV: 37.9\%$\to$92.5\%) warrants a deeper mechanistic analysis. Three factors interact predictably:
	
	(i) Exploration-regularization conflict: DDPG relies on Ornstein-Uhlenbeck (OU) noise---temporally correlated, high-amplitude perturbations calibrated for 8-DOF joint-space exploration. The energy regularization term $\lambda\mathcal{E}(a_t) = \lambda a_t^\top Q a_t$ penalizes exactly these large-magnitude perturbations, reducing effective exploration coverage. Unlike SAC, whose entropy maximization maintains stochasticity in action selection independent of the reward signal, DDPG's deterministic policy has no intrinsic compensation mechanism for suppressed exploration amplitude.
	
	(ii) Local minima in the shaped landscape: For 8-DOF locomotion, $\Phi_{\text{energy}}$ creates gradient attraction toward passive, low-energy configurations (e.g., the ant crouching with minimal joint velocity). Without SAC-style stochastic escape, DDPG's deterministic gradient descent can converge to these energy-minimizing but task-suboptimal attractors. This is precisely the Class~II failure mode predicted by Theorem~\ref{prop:regularization_necessity}: Ant-v5 requires transient energy injection for effective locomotion, making strong energy constraints counterproductive when paired with a deterministic, exploration-limited algorithm.
	
	(iii) Amplitude-variance amplification: The CV jump from 37.9\% to 92.5\% reflects policy oscillation between passive (low-energy) and active (task-required) modes. When OU perturbations are suppressed, the policy alternates between over-exploiting low-energy attractors and recovering task performance through residual exploration---producing high variance without converging to either extreme.
	
	This failure is a principled boundary of H-EARS applicability, not a framework defect: algorithms that rely on high-magnitude action perturbations for exploration in high-DOF spaces are structurally incompatible with simultaneous strong energy minimization. This interaction is consistent with the general observation that framework-algorithm compatibility depends on the exploration mechanism employed, and is noted as a deployment consideration.
	
	Cross-algorithm analysis reveals physical characteristic dependencies. LunarLander-v3's universal improvement (SAC +7.8\%, PPO +9.8\%, DDPG +8.2\%) reflects natural energy formulation alignment—landing tasks inherently optimize kinetic/potential tradeoffs. Ant-v5 exhibits algorithm bifurcation: stochastic off-policy methods (SAC +32.5\%, TD3 +15.5\%) benefit from coordinated 8-DOF control, while aggressive exploration requirements (DDPG -25.2\%, PPO initially unstable) produce degradation, confirming Theorem~\ref{thm:energy_convergence}'s applicability requires sufficient stochastic exploration. Hopper-v5's mixed results (SAC +33.1\%, TD3 +11.7\%, PPO -3.3\%) reflect instability paradox: energy potentials stabilize through implicit Lyapunov constraints but over-constrain necessary rapid corrections. Humanoid-v5's modest gains validate Lemma~\ref{lem:approx_potential}: O(n) approximation suffices for high-dimensional systems, with stability enhancement (SAC CV -10.9\%, TD3 -12.5\%) exceeding performance gains. These results establish H-EARS' systematic patterns: strongest in energy-natural tasks, effective for stochastic off-policy methods (SAC, TD3), limited by algorithm incompatibility (DDPG high-DOF), bounded by physical characteristics (Humanoid complexity). Observed algorithm-environment interactions validate rather than contradict theoretical predictions, demonstrating principled framework behavior under varying conditions.
	
	\section{Ablation Experiments}
	\label{sec:ablation}
	
	To validate the necessity of each framework component, ablation studies are conducted in Ant-v5 (8-DoF quadruped requiring multi-joint coordination) and Hopper-v5 (3-DoF single-leg system with inherent instability). These environments expose distinct failure modes when components are absent, confirming theoretical predictions.
	
	\subsection{Ablation Configuration}
	
	Nine algorithm variants in two environments are evaluated, as detailed in Table~\ref{tab:ablation_config}; the eighth variant (Combined-Single-$\Phi$) is additionally analysed in isolation in Fig.~\ref{fig:decomp_ablation}. Each variant isolates specific framework components by selectively activating hyperparameters.
	
	\begin{table*}[htbp]
		\renewcommand{\arraystretch}{1.3}
		\captionsetup{font=footnotesize, labelfont=footnotesize}
		\caption{Hyperparameters for Ablation Variants in Ant-v5 and Hopper-v5}
		\label{tab:ablation_config}
		\centering
		\footnotesize
		\begin{tabular}{lcccccc}
			\toprule
			& \multicolumn{3}{c}{\textbf{Ant-v5}} & \multicolumn{3}{c}{\textbf{Hopper-v5}} \\
			\cmidrule(lr){2-4}\cmidrule(lr){5-7}
			\textbf{Variant} & $\alpha_{\text{task}}$ & $\alpha_{\text{energy}}$ & $\lambda$ & $\alpha_{\text{task}}$ & $\alpha_{\text{energy}}$ & $\lambda$ \\
			\midrule
			SAC - Vanilla             & 0     & 0     & 0     & 0     & 0     & 0 \\
			Energy Only             & 0     & 3e-2  & 0     & 0     & 1e-3  & 0 \\
			Task Only               & 5e-3  & 0     & 0     & 5e-1  & 0     & 0 \\
			Regularization Only     & 0     & 0     & 1e-2  & 0     & 0     & 5e-4 \\
			Without Regularization  & 5e-3  & 3e-2  & 0     & 5e-1  & 1e-3  & 0 \\
			Without Energy          & 5e-3  & 0     & 1e-2  & 5e-1  & 0     & 5e-4 \\
			Without Task            & 0     & 3e-2  & 1e-2  & 0     & 1e-3  & 5e-4 \\
			SAC - H-EARS           & 5e-3  & 3e-2  & 1e-2  & 5e-1  & 1e-3  & 5e-4 \\
			\midrule
			Combined-Single-$\Phi$  & ---   & $\alpha_{\Sigma}$=3.5e-2 & 0 & --- & $\alpha_{\Sigma}$=5.01e-1 & 0 \\
			\bottomrule
		\end{tabular}
	\end{table*}
	
	\begin{figure}[!t]
		\centering
		\subfigure[Ant-v5]{
			\includegraphics[width=0.4\textwidth]{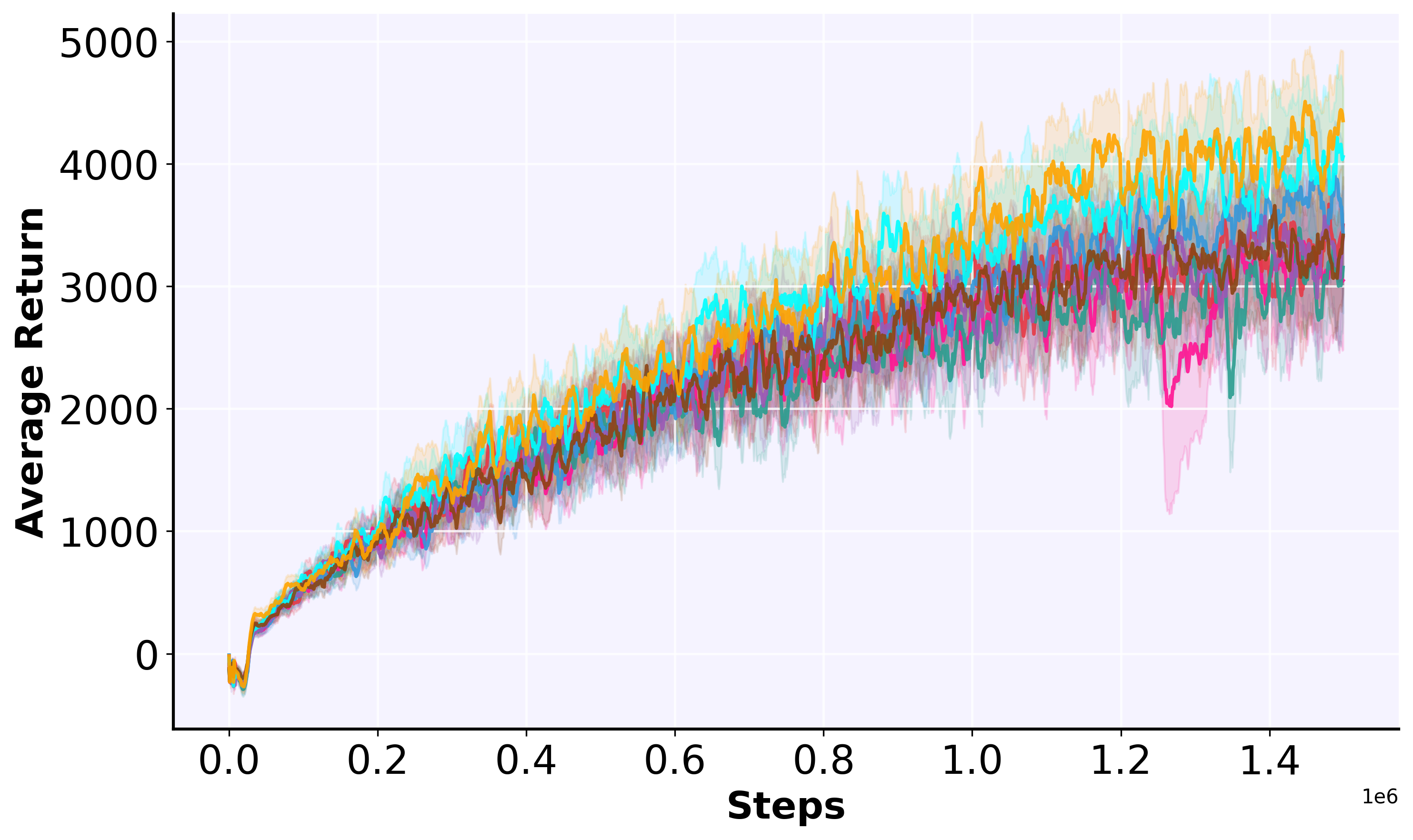}
			\label{fig:ablation_ant}
		}
		\hspace{0.7cm}
		\subfigure[Hopper-v5]{
			\includegraphics[width=0.4\textwidth]{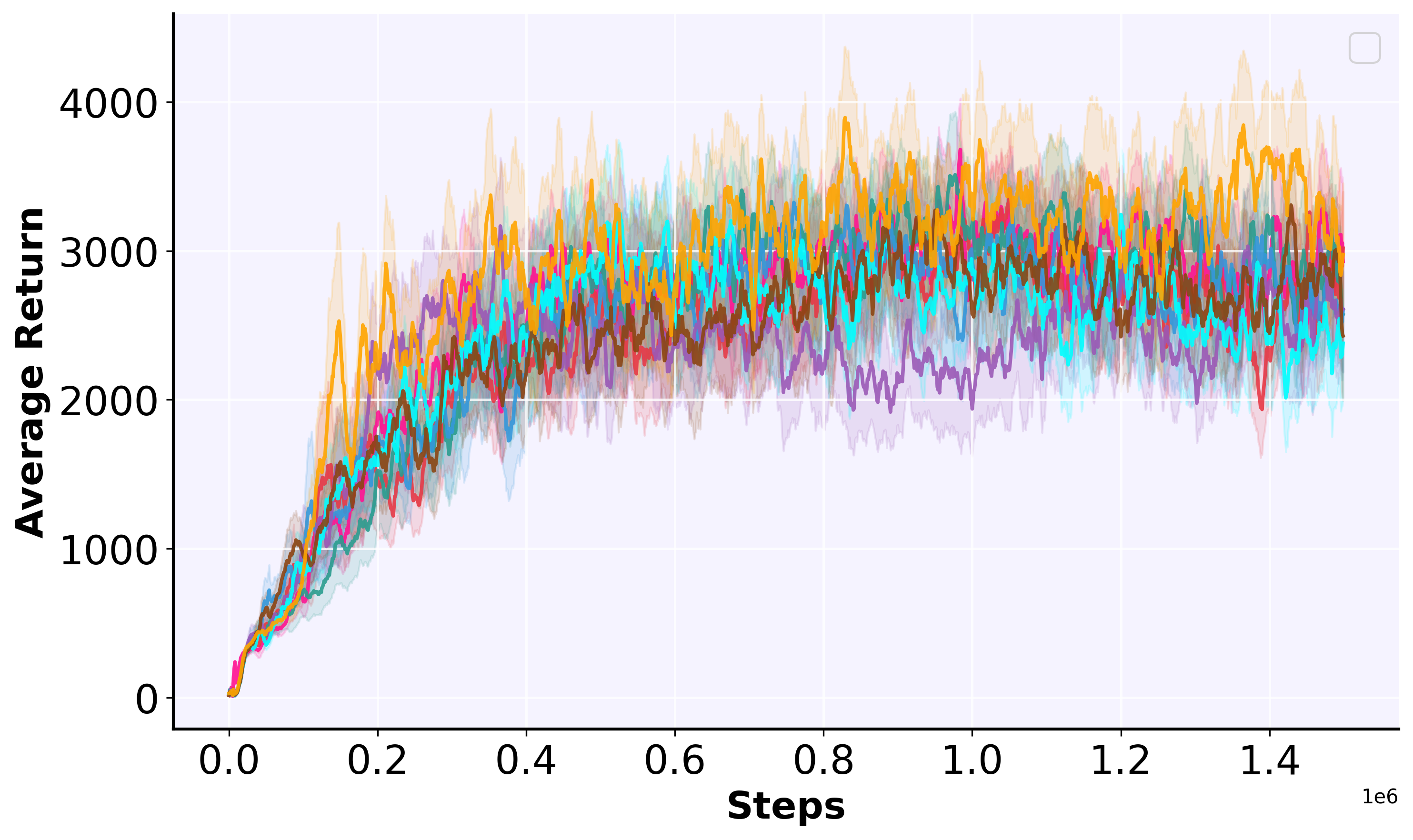}
			\label{fig:ablation_hopper}
		}
		
		\vspace{0.15cm}
		\centering
		\begin{tabular}{@{}l@{\hspace{0.9em}}l@{\hspace{0.9em}}l@{\hspace{0.9em}}l@{}}
			\colorbox{Vanilla}{\rule{0pt}{1pt}\rule{8pt}{0pt}} \raisebox{-2.0pt}{\scriptsize SAC-Vanilla} &
			\colorbox{TaskOnly}{\rule{0pt}{1pt}\rule{8pt}{0pt}} \raisebox{-2.0pt}{\scriptsize Task Only} &
			\colorbox{EnergyOnly}{\rule{0pt}{1pt}\rule{8pt}{0pt}} \raisebox{-2.0pt}{\scriptsize Energy Only} &
			\colorbox{RegOnly}{\rule{0pt}{1pt}\rule{8pt}{0pt}} \raisebox{-2.0pt}{\scriptsize Reg. Only} \\[0.2em]
			\colorbox{WithoutTask}{\rule{0pt}{1pt}\rule{8pt}{0pt}} \raisebox{-2.0pt}{\scriptsize Without Task} &
			\colorbox{WithoutEnergy}{\rule{0pt}{1pt}\rule{8pt}{0pt}} \raisebox{-2.0pt}{\scriptsize Without Energy} &
			\colorbox{WithoutReg}{\rule{0pt}{1pt}\rule{8pt}{0pt}} \raisebox{-2.0pt}{\scriptsize Without Reg.} &
			\colorbox{HEARS}{\rule{0pt}{1pt}\rule{8pt}{0pt}} \raisebox{-2.0pt}{\scriptsize H-EARS}
		\end{tabular}
		
		\caption{Ablation study in Ant-v5 and Hopper-v5 (SAC backbone). TaskOnly and EnergyOnly converge independently without gradient interference, supporting Lemma~\ref{lem:functional_independence}. The Without-Regularization gap validates Proposition~\ref{prop:regularization_necessity}.}
		\label{fig:ablation}
	\end{figure}
	
	\begin{figure}[htbp]
		\centering
		\subfigure[Ant-v5]{
			\includegraphics[width=0.4\textwidth]{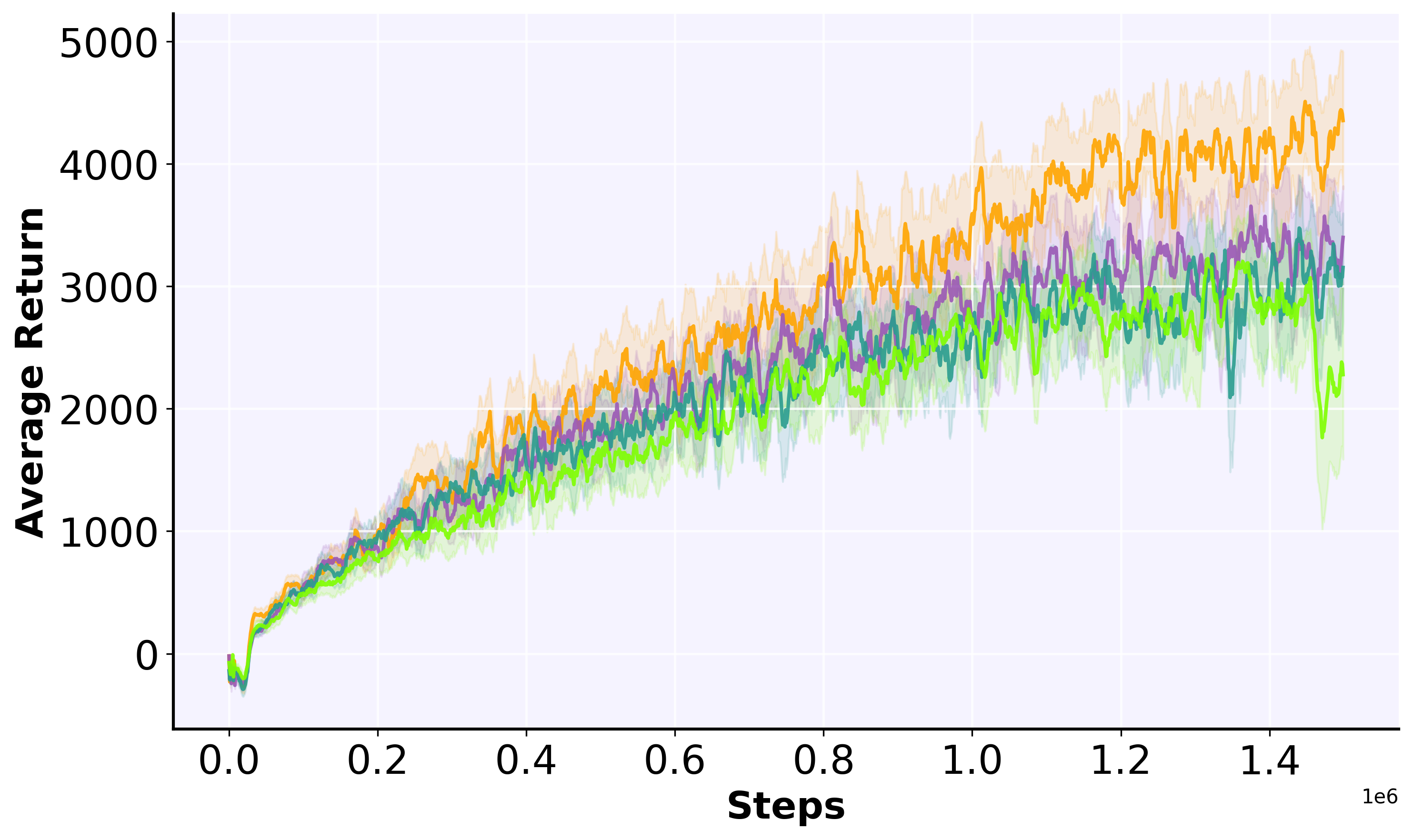}
			\label{fig:decom_ablation_ant}
		}
		\hspace{0.7cm}
		\subfigure[Hopper-v5]{
			\includegraphics[width=0.4\textwidth]{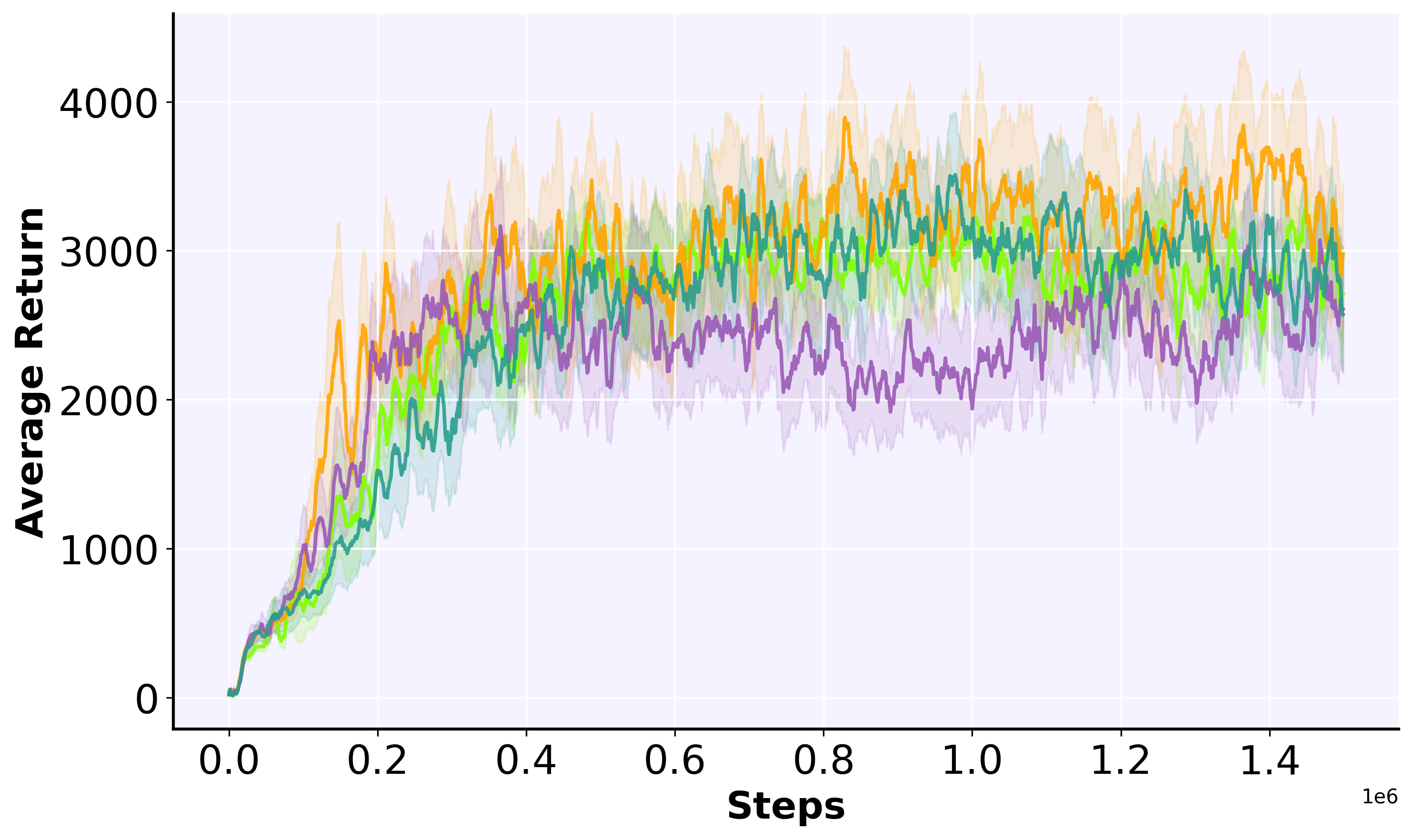}
			\label{fig:decom_ablation_hopper}
		}
		
		\vspace{0.15cm}
		\centering
		\begin{tabular}{@{}l@{\hspace{0.9em}}l@{\hspace{0.9em}}l@{\hspace{0.9em}}l@{}}
			\colorbox{TaskOnly}{\rule{0pt}{1pt}\rule{8pt}{0pt}} \raisebox{-2.0pt}{\scriptsize Task Only} &
			\colorbox{EnergyOnly}{\rule{0pt}{1pt}\rule{8pt}{0pt}} \raisebox{-2.0pt}{\scriptsize Energy Only} &
			\colorbox{Decombination}{\rule{0pt}{1pt}\rule{8pt}{0pt}} \raisebox{-2.0pt}{\scriptsize Single Potential}
			\colorbox{HEARS}{\rule{0pt}{1pt}\rule{8pt}{0pt}} \raisebox{-2.0pt}{\scriptsize H-EARS}
		\end{tabular}
		
		\caption{Dual-decomposition ablation: H-EARS vs.\ Combined-Single-$\Phi$ (same total weight $\alpha_{\text{task}}+\alpha_{\text{energy}}$, single potential), Task~$\Phi$~Only, and Energy~$\Phi$~Only. SAC backbone, 5~seeds; shading\,=\,$\pm$\,s.e. Combined-Single-$\Phi$ underperforms H-EARS by $\approx$33\% (Ant) and $\approx$21\% (Hopper) despite identical total weighting, validating Proposition~\ref{prop:dual_necessity}. In Hopper-v5 it falls below Energy~$\Phi$~Only, confirming that potential merging suppresses rather than combines individual contributions.}
		\label{fig:decomp_ablation}
	\end{figure}
	
	\subsection{Component-Specific Analysis}
	
	Figure~\ref{fig:ablation} presents ablation results. Three critical observations emerge, each validating specific theoretical claims:
	
	(1) Task Potential Necessity (Lemma~\ref{lem:functional_independence}):Comparing Without Task versus H-EARS reveals 22.2\% (Ant) and 17.2\% (Hopper) performance degradation. Task potential $\Phi_{\text{task}}$ provides gradient-based exploration guidance, manifested through rapid early convergence in Task Only variants. However, Task Only plateaus below H-EARS (8.6\% gap in Ant, 15.4\% in Hopper), confirming that task guidance alone lacks stabilizing mechanisms for asymptotic refinement. The functional independence property (Lemma~\ref{lem:functional_independence}) enables task and energy components to contribute without mutual interference, as confirmed by the ablation: TaskOnly and EnergyOnly curves converge independently without suppressing each other's training signal, providing direct empirical support for the decoupled-domain argument.
	
	(2) Energy Potential as Stability Constraint (Theorem~\ref{thm:energy_convergence}):Environment-specific energy effects validate Theorem~\ref{thm:energy_convergence}'s mechanical stability principle. In Hopper-v5 (inherently unstable single-leg system), Without Energy achieves only 92.6\% of H-EARS performance with 37.2\% higher variance, exposing critical instability under absent physical constraints. Energy Only demonstrates superior stability (CV: 7.1\%) despite 28.6\% lower returns, confirming that energy minimization provides Lyapunov-like convergence acceleration through $d^2E/dq^2 > 0$ properties. Conversely, Ant-v5's statically stable quadruped structure reduces energy dependence—Without Energy retains 90.7\% performance—indicating that energy potentials primarily benefit dynamically unstable systems requiring explicit stability enforcement.
	
	(3) Cross-environment Hessian validity and gain correlation. Theorem~\ref{thm:energy_convergence} predicts that H-EARS gains scale with the degree to which $\partial^2 E/\partial q^2 \succ 0$ holds over the task-relevant operating region $\mathcal{S}_{\text{op}}$. Table~\ref{tab:hessian_validity} verifies this prediction across all four standard environments using the gradient informativeness ratio $\rho_{\text{info}} = \|\nabla_q E\| / \|\nabla_s R\|_{\text{sparse}}$ estimated from training logs, alongside the observed average H-EARS performance gain (SAC backbone, averaged over all seeds).
	
	(4) Dual-Decomposition Necessity. To isolate the structural contribution of the dual-potential decomposition from the effect of total weighting budget, we introduce a Combined-Single-$\Phi$ variant that merges $\alpha_{\text{task}}\Phi_{\text{task}}+\alpha_{\text{energy}}\Phi_{\text{energy}}$ into a single potential weighted by $(\alpha_{\text{task}}+\alpha_{\text{energy}})$ with $\lambda$ held fixed, matching H-EARS~(Full) in total weight. Results are shown in Fig.~\ref{fig:decomp_ablation}.
	
	\begin{table*}[!t]
		\renewcommand{\arraystretch}{1.2}
		\captionsetup{font=footnotesize, labelfont=footnotesize}
		\caption{Hessian Condition Validity and Observed H-EARS Gains (SAC backbone)}
		\label{tab:hessian_validity}
		\centering
		\footnotesize
		\begin{tabular}{lccccc}
			\toprule
			\textbf{Environment} & \textbf{$\partial^2E/\partial q^2 \succ 0$} & \textbf{$\rho_{\text{info}}$ (est.)} & \textbf{Primary gain mechanism} & \textbf{Perf.\ gain (\%)} & \textbf{Theory consistent?} \\
			\midrule
			Ant-v5         & Global (rigid-body $M(q)\succ0$; 8-DOF) & $\sim10^2$--$10^3$ & Gradient enrichment             & $+32.5\%$ & \checkmark \\
			Hopper-v5      & Local near upright equilibrium           & $\sim10^1$--$10^2$ & Stability regularization        & $+33.1\%$ & \checkmark \\
			LunarLander-v3 & Partial (non-mechanical energy proxy)    & $\sim10^0$--$10^1$ & Action regularization           & $+7.8\%$  & \checkmark (partial) \\
			Humanoid-v5    & Global ($M(q)\succ0$; 21-DOF)           & $\sim10^1$--$10^2$$^\dagger$ & Stability (approx.\ gap limits perf.) & $+4.8\%$ & \checkmark \\
			\bottomrule
		\end{tabular}
		\vspace{1mm}
		
		\footnotesize{\textit{Note:} $^\dagger$Humanoid-v5's $\rho_{\text{info}}$ falls in the lower portion of the $10^1$--$10^2$ range. At 21-DOF, the O($n$) energy model captures a smaller fraction of the total energy than at 8-DOF (Ant), increasing the approximation error $\epsilon_{\text{approx}}$ and attenuating the effective gradient enrichment signal per Lemma~\ref{lem:approx_potential}. Gains therefore manifest primarily as stability improvement (CV: $10.1\%\to9.0\%$) rather than asymptotic performance increase.}
	\end{table*}
	
	In Ant-v5, H-EARS~(Full) achieves a final average return of approximately 4\,200 (consistent with the 5-seed evaluation mean of $4183\pm174$ in Table~\ref{tab:comprehensive_results}), while Combined-Single-$\Phi$ reaches only $\approx$\,2\,800---a gap of $\approx$\,33\% despite identical total weighting. In Hopper-v5, H-EARS~(Full) converges to a training-curve stable value of $\approx$\,3\,600 while Combined-Single-$\Phi$ reaches $\approx$\,2\,850, a gap of $\approx$\,21\%. The Hopper training-curve value ($\approx$\,3\,600) reflects the peak stable region of the dual-decomposition ablation figure (Fig.~\ref{fig:decomp_ablation}); the 5-seed evaluation mean reported in Table~\ref{tab:comprehensive_results} ($3354\pm354$) is lower due to cross-seed variance and end-of-training policy fluctuation rather than a discrepancy in framework performance. Notably, in Hopper-v5, Combined-Single-$\Phi$ falls below even the Energy~$\Phi$~Only variant ($\approx$\,3\,200)---a result that constitutes direct performance-level evidence of destructive gradient interference: if the two gradient directions were merely misaligned but not mutually suppressive, merging them with identical total weight should yield performance no worse than the stronger single component. That the merged variant underperforms even the weaker single component can only be explained by active cancellation of the two gradient signals, not additive combination. This behaviour arises because a single potential cannot simultaneously satisfy task directivity ($\nabla_s\Phi\propto\nabla_s R$) and energy awareness ($\Phi(s)=-E(q,\dot{q})$) when $\nabla_s\Phi_{\text{task}}\not\propto-\nabla_s E$ in task-critical regions, so the merged gradient $\nabla_s\Phi_{\text{single}}$ is deflected away from both individual objectives---precisely the mechanism characterised by Proposition~\ref{prop:dual_necessity}. The consistent performance gap in both environments confirms that the dual-decomposition structure, not merely the total weighting budget, is the source of H-EARS' advantage over single-potential baselines.
	
	(5) Action Regularization Necessity (Theorem~\ref{prop:regularization_necessity}):Without Regularization exhibits 17.2\% (Ant) and 20.7\% (Hopper) degradation with substantially increased variance. This validates Theorem~\ref{prop:regularization_necessity}'s theoretical prediction: dual potentials generate competing policy gradients ($\nabla\Phi_{\text{task}}$ encouraging rapid transitions versus $\nabla\Phi_{\text{energy}}$ favoring energy conservation), manifesting as oscillatory training dynamics. Action regularization $\lambda\|a\|^2$ mediates these conflicts by penalizing extreme control outputs that satisfy one potential while violating another. Reg. Only achieves 84.2\% (Ant) and 85.9\% (Hopper) of H-EARS performance, confirming regularization provides implicit smoothness constraints. However, convergence efficiency drops significantly—Reg. Only requires 28.4\% (Ant) and 32.4\% (Hopper) more steps to reach 90\% of final performance compared to H-EARS, demonstrating that explicit physics priors (potentials) complement but cannot be replaced by implicit constraints (regularization).
	
	The monotone relationship between Hessian condition strength and observed gain---global validity yielding the largest improvements, partial validity yielding more modest but still consistent gains---provides empirical support for Theorem~\ref{thm:energy_convergence}. For LunarLander-v3, where the energy proxy does not correspond to a standard mechanical Hessian, gains are smallest and derive primarily from the action regularization component rather than from gradient enrichment, consistent with the analytical prediction that $\rho_{\text{info}} \approx 1$ in this regime.
	
	\subsection{Statistical Significance}
	
	\begin{table}[htbp]
		\renewcommand{\arraystretch}{1.2}
		\captionsetup{font=footnotesize, labelfont=footnotesize}
		\caption{Statistical Significance (Welch's $t$-test)}
		\label{tab:statistical_significance}
		\centering
		\footnotesize
		\begin{tabular}{lcccc}
			\toprule
			\textbf{Environment} & $\Delta\mu$ & $95\%$ CI & $p$-value & Sig. \\
			\midrule
			Ant-v5    & $+1026$ & $[766,\;1286]$ & $<0.0001$ & *** \\
			Hopper-v5 & $+834$  & $[278,\;1390]$ & $0.0087$  & **  \\
			\bottomrule
			\multicolumn{5}{l}{\scriptsize ***$p<0.001$; **$p<0.01$ (two-tailed Welch's $t$-test, $\alpha=0.01$).}
		\end{tabular}
	\end{table}
	
	To verify robustness of the primary performance gains to random seed variation, we apply Welch's two-sample $t$-test (unequal variance, two-tailed) comparing SAC+H-EARS versus SAC-Vanilla on the two ablation environments, using the statistics from Table~\ref{tab:comprehensive_results}(a). The test statistic and Welch--Satterthwaite effective degrees of freedom are:
	\begin{equation}
		t = \frac{\bar{x}_{\text{H-EARS}} - \bar{x}_{\text{Vanilla}}}{\sqrt{\sigma_{\text{H-EARS}}^2/5\;+\;\sigma_{\text{Vanilla}}^2/5}}, \qquad
		\nu = \frac{\bigl(\sigma_1^2/5+\sigma_2^2/5\bigr)^2}{\frac{(\sigma_1^2/5)^2}{4}+\frac{(\sigma_2^2/5)^2}{4}}
	\end{equation}
	Table~\ref{tab:statistical_significance} reports the results for the SAC algorithm across two ablation environments:
	
	Both confidence intervals lie entirely in the positive region, confirming that the SAC+H-EARS improvement over SAC-Vanilla is statistically significant at the $1\%$ level in both ablation environments. Statistical testing is centered on these two environments, which provide the controlled experimental conditions (full 8-variant ablation design, both algorithms fully trained to convergence)~\cite{henderson2018deep}.
	
	\subsection{Hyperparameter Sensitivity Analysis}
	\label{sec:sensitivity}
	
	\begin{figure*}[htbp]
		\centering
		\captionsetup{font=footnotesize, labelfont=footnotesize}
		\subfigure[Task in Ant-v5]{
			\includegraphics[width=0.23\textwidth]{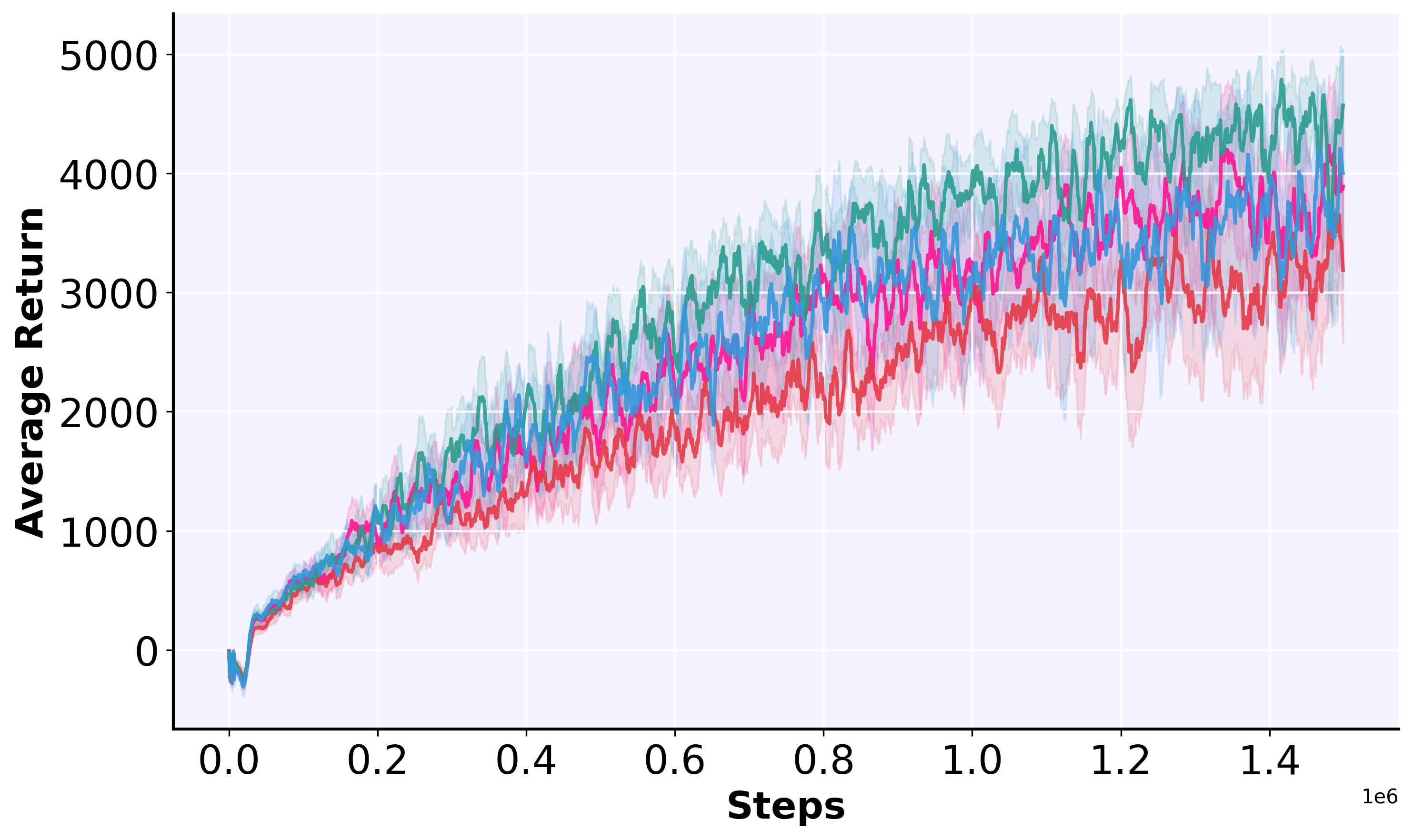}
			\label{fig:task_ant}
		}
		\hspace{-0.2cm}
		\subfigure[Task in Humanoid-v5]{
			\includegraphics[width=0.23\textwidth]{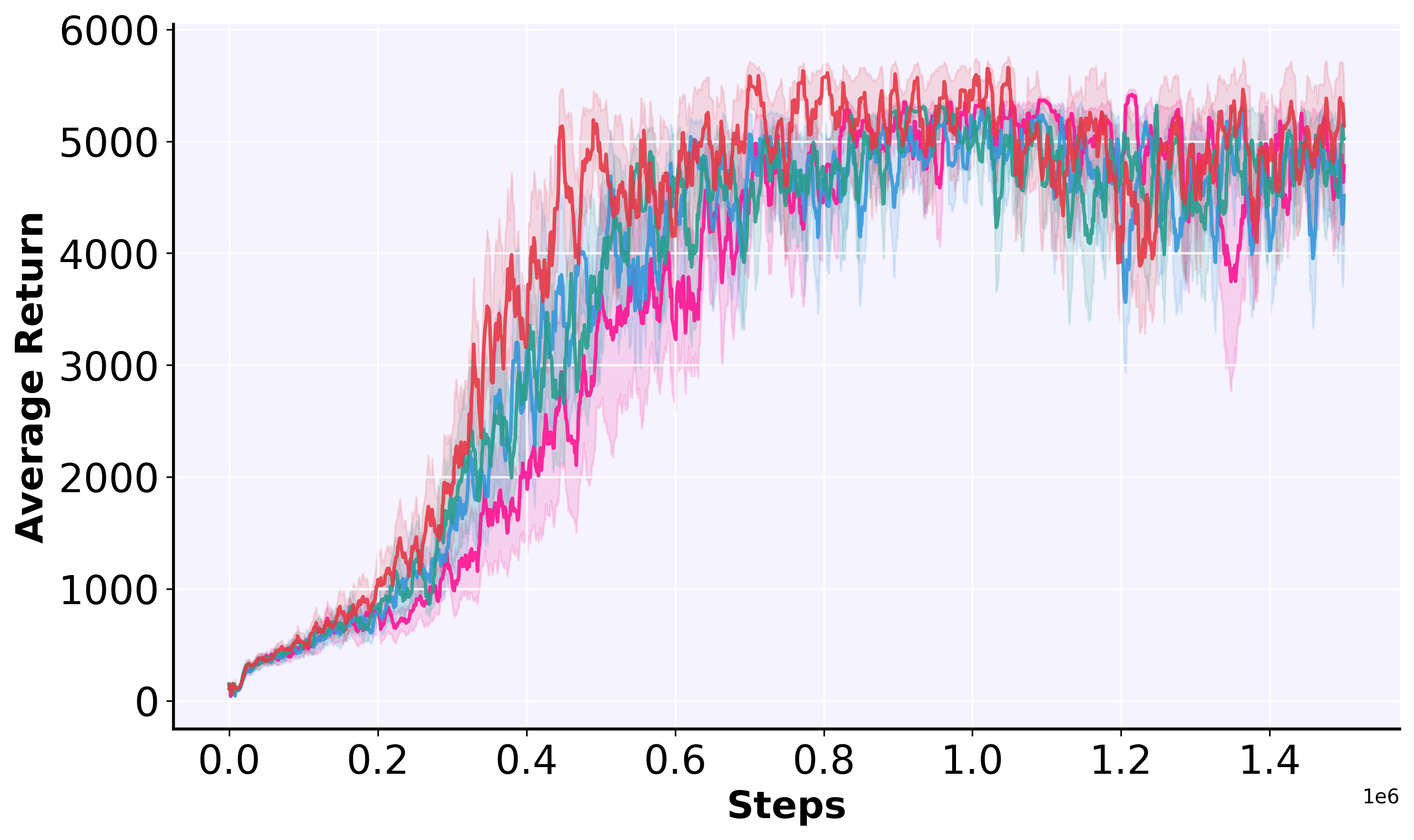}
			\label{fig:task_humanoid}
		}
		\hspace{-0.2cm}
		\subfigure[Task in Hopper-v5]{
			\includegraphics[width=0.23\textwidth]{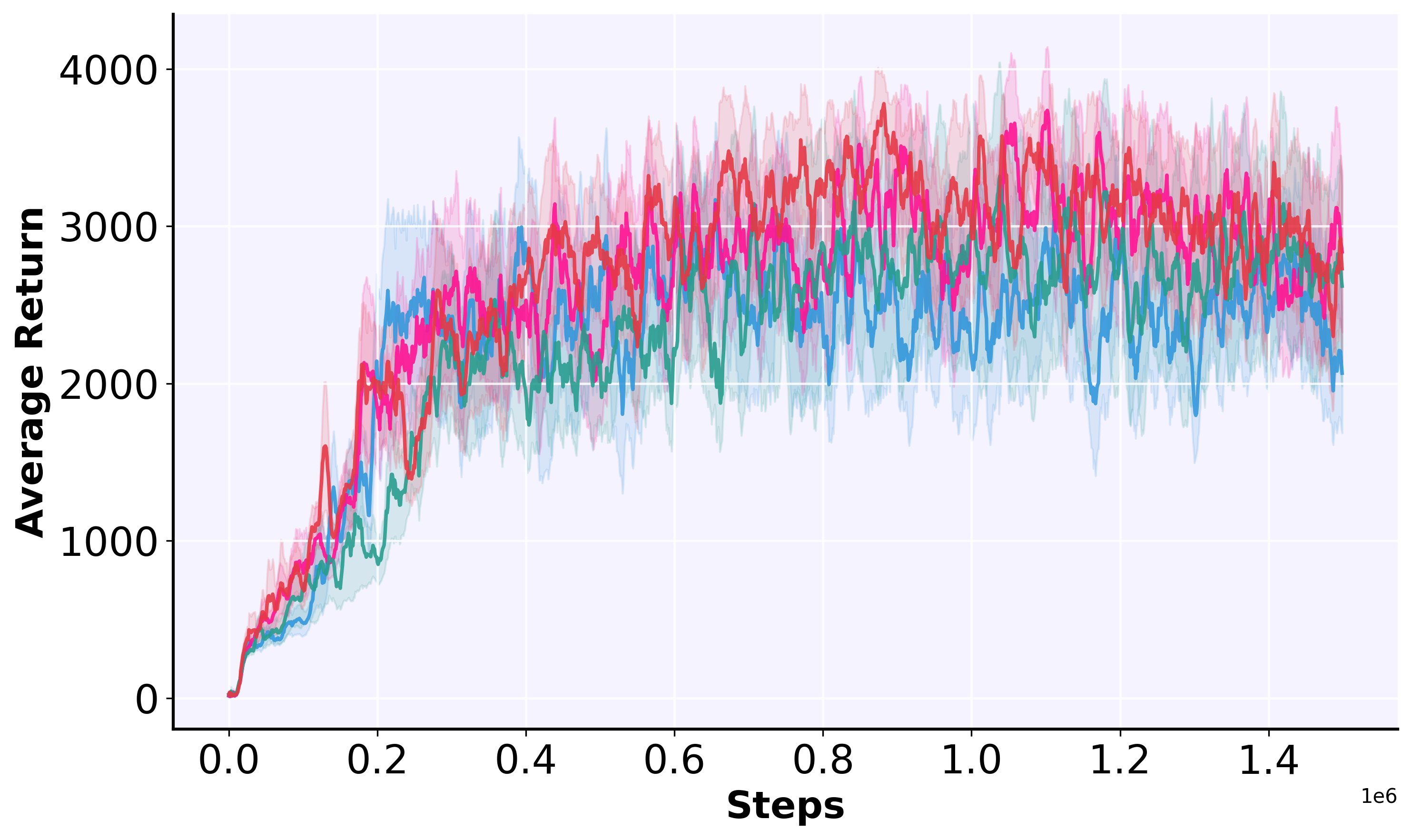}
			\label{fig:task_hopper}
		}
		\hspace{-0.2cm}
		\subfigure[Task in Lunarlander-v3]{
			\includegraphics[width=0.23\textwidth]{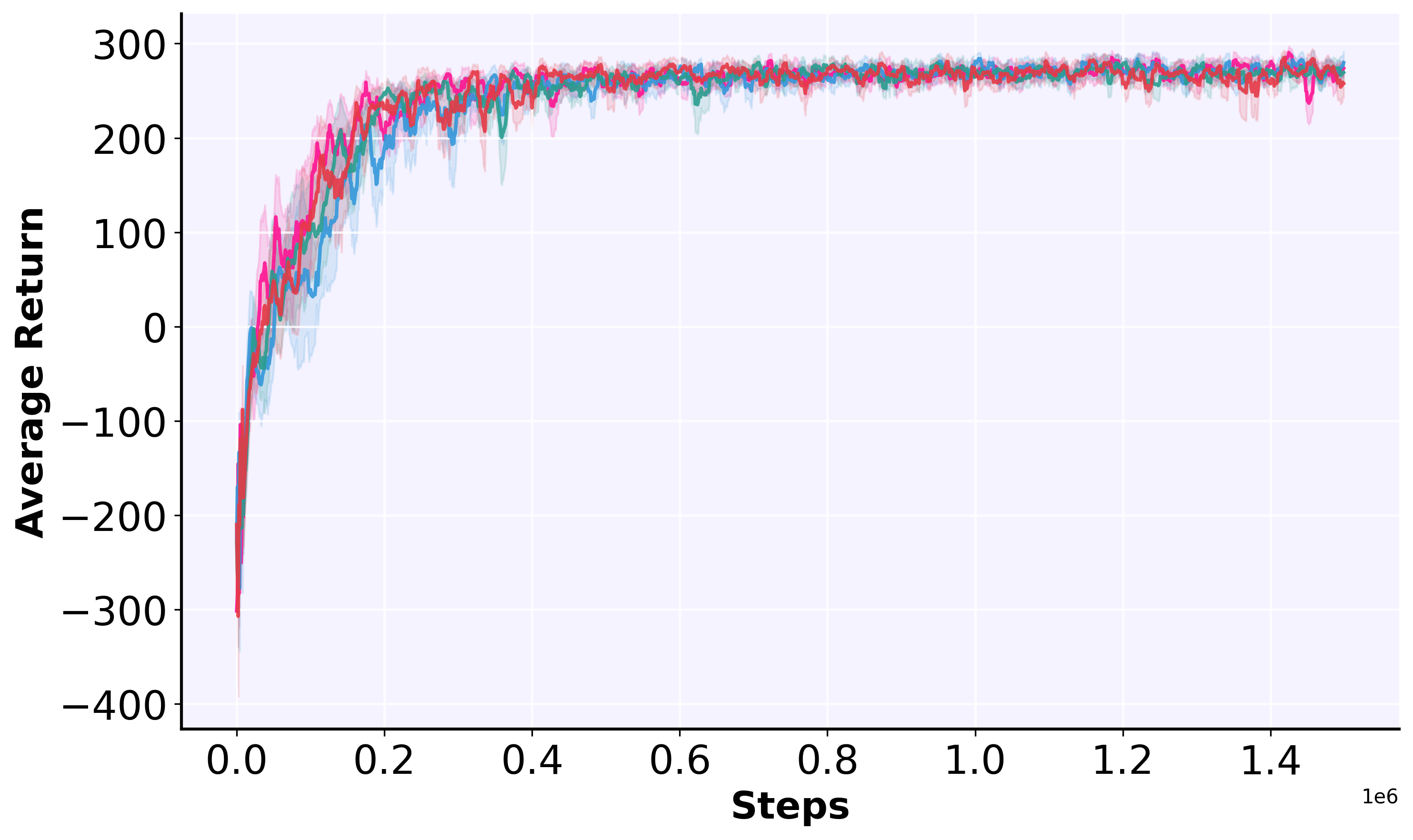}
			\label{fig:task_lunarlander}
		}
		
		\subfigure[Energy in Ant-v5]{
			\includegraphics[width=0.23\textwidth]{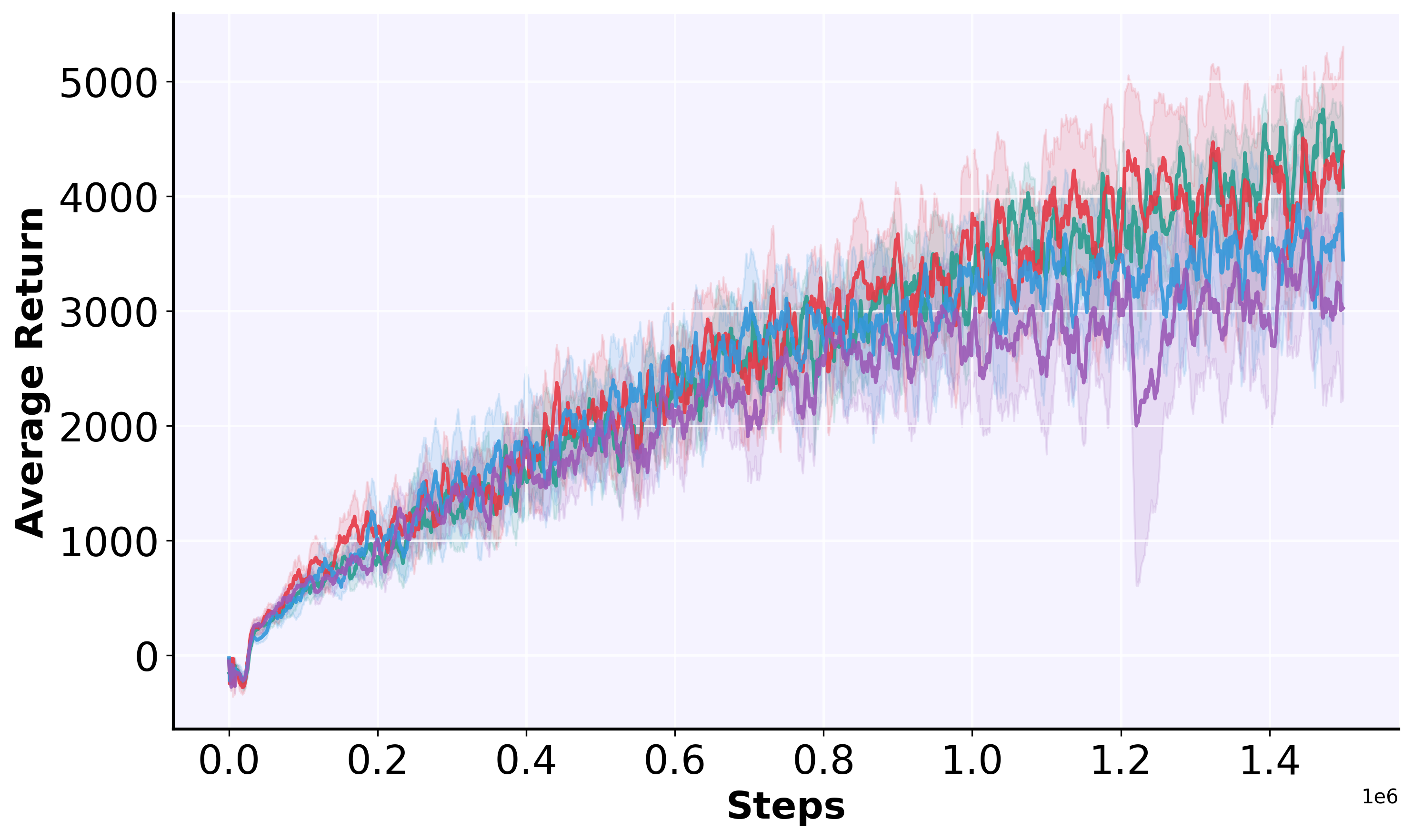}
			\label{fig:energy_ant}
		}
		\hspace{-0.2cm}
		\subfigure[Energy in Humanoid-v5]{
			\includegraphics[width=0.23\textwidth]{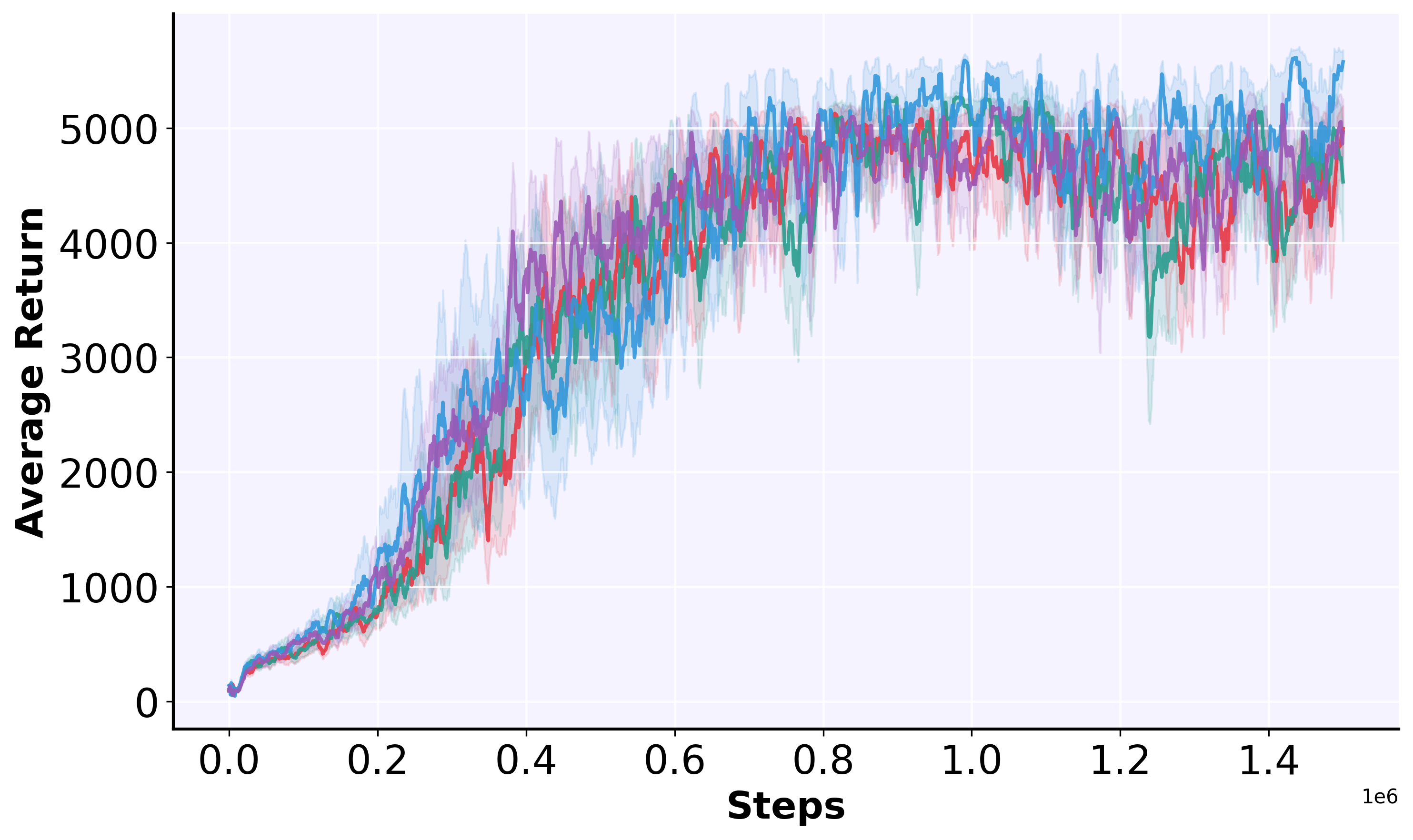}
			\label{fig:energy_humanoid}
		}
		\hspace{-0.2cm}
		\subfigure[Energy in Hopper-v5]{
			\includegraphics[width=0.23\textwidth]{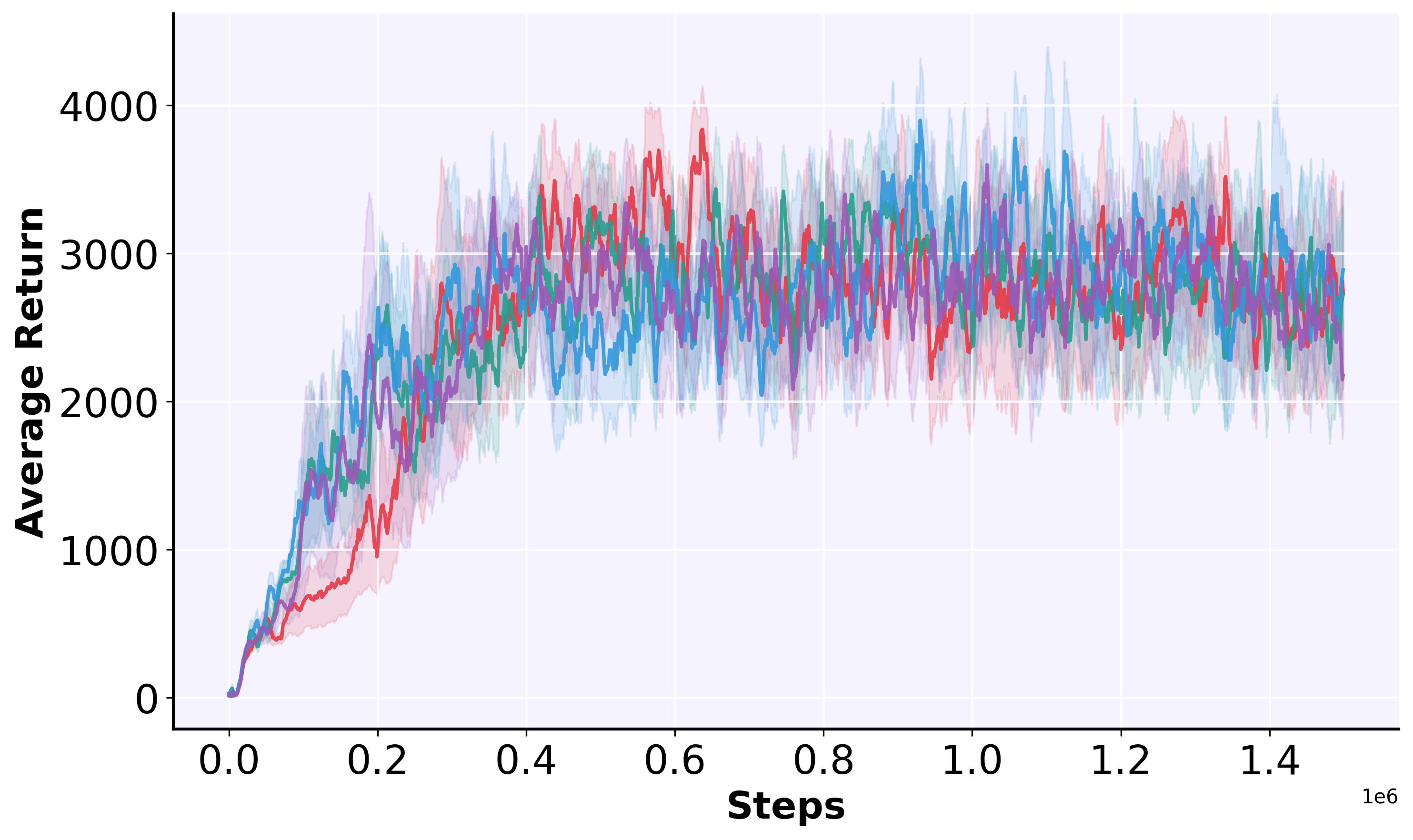}
			\label{fig:energy_hopper}
		}
		\hspace{-0.2cm}
		\subfigure[Energy in Lunarlander-v3]{
			\includegraphics[width=0.23\textwidth]{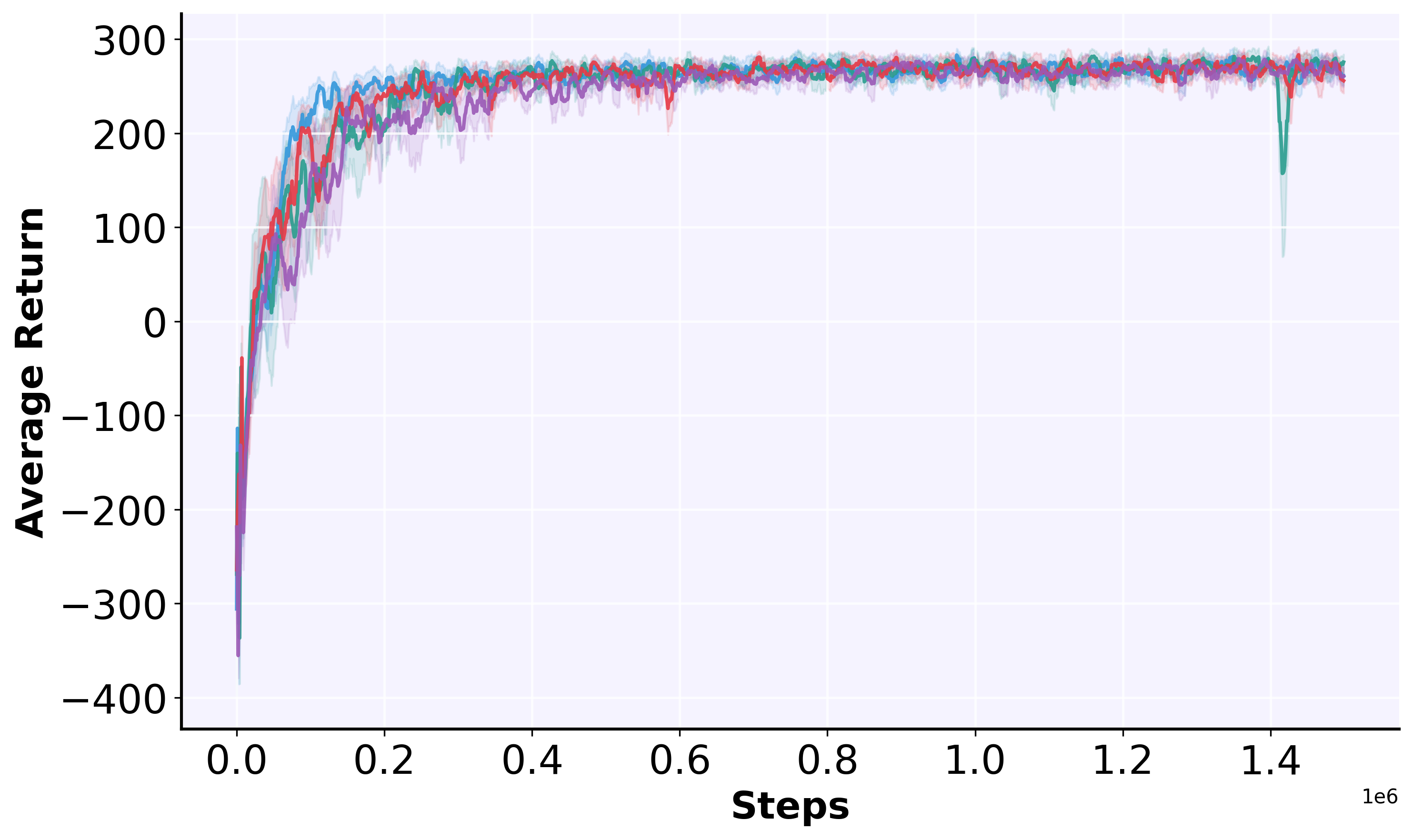}
			\label{fig:energy_lunarlander}
		}
		
		\subfigure[Regularization in Ant-v5]{
			\includegraphics[width=0.23\textwidth]{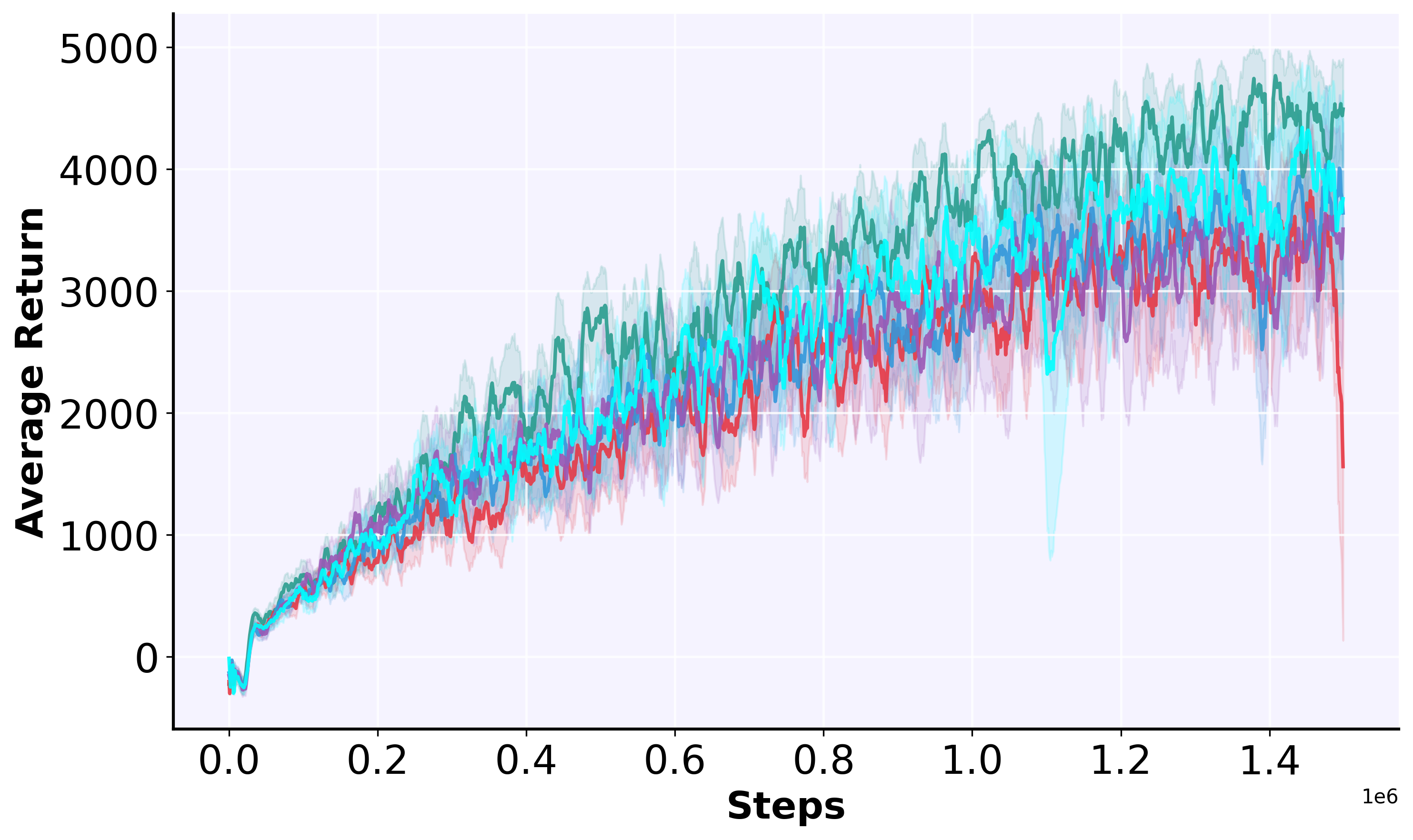}
			\label{fig:regularization_ant}
		}
		\hspace{-0.2cm}
		\subfigure[Regularization in Humanoid-v5]{
			\includegraphics[width=0.23\textwidth]{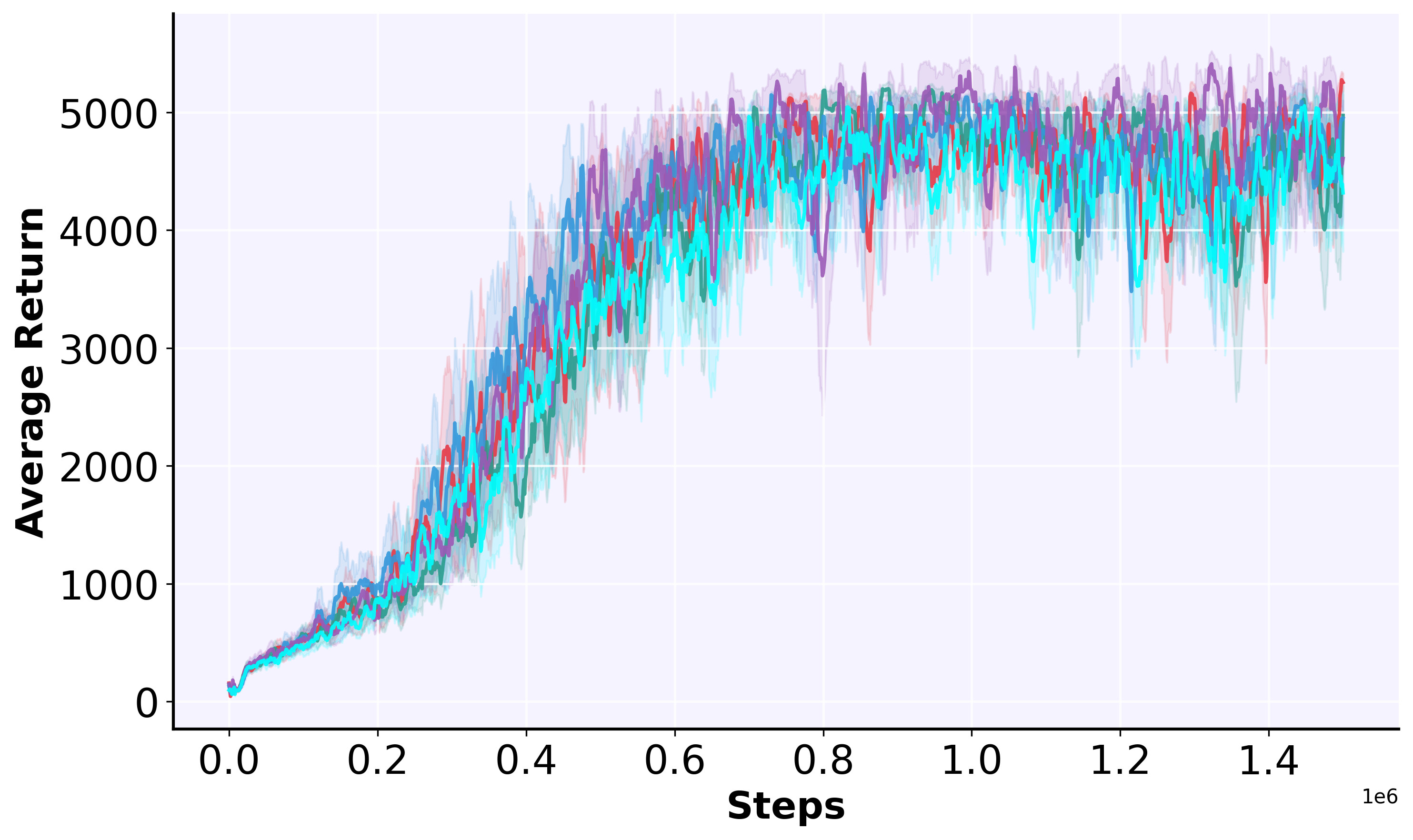}
			\label{fig:regularization_humanoid}
		}
		\hspace{-0.2cm}
		\subfigure[Regularization in Hopper-v5]{
			\includegraphics[width=0.23\textwidth]{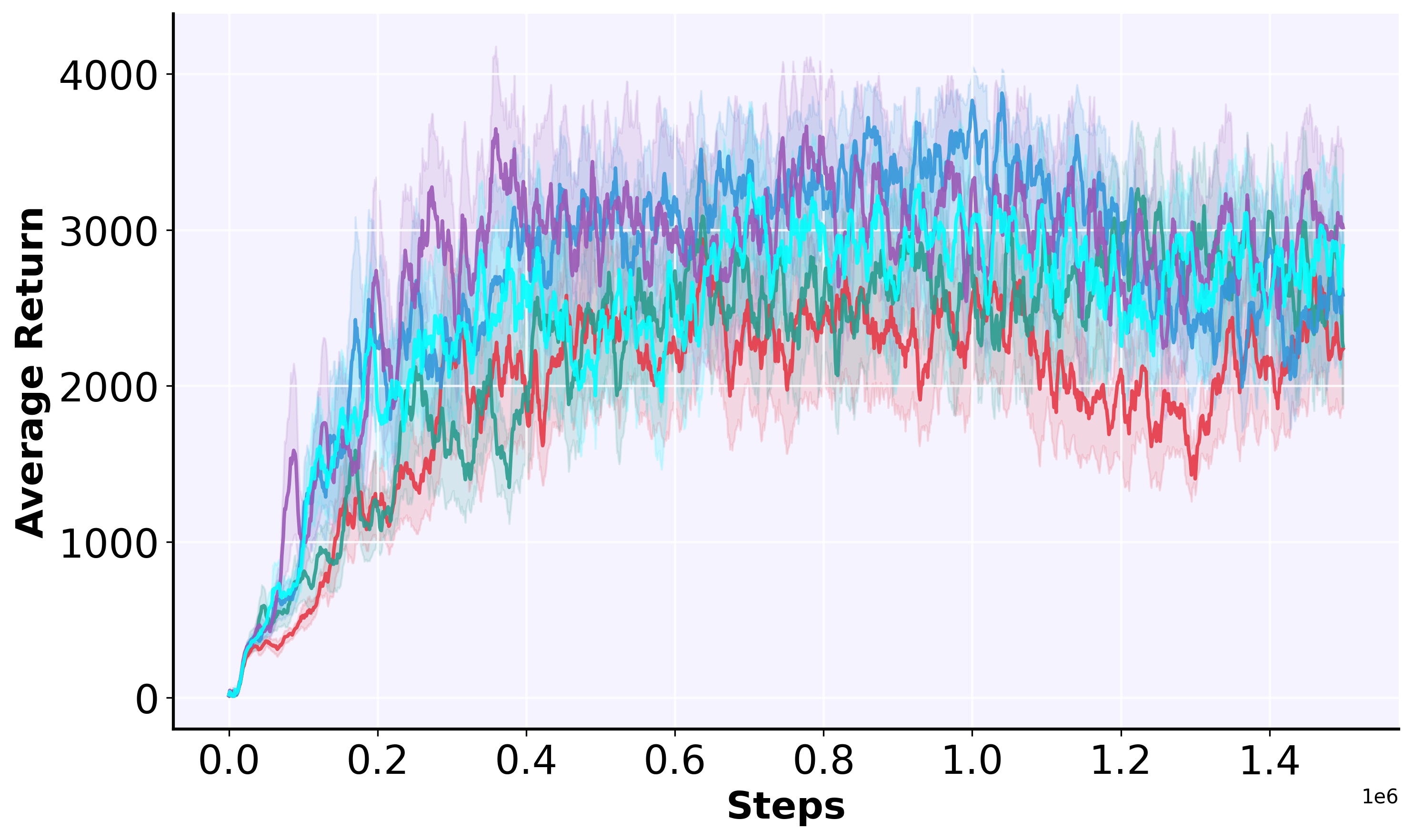}
			\label{fig:regularization_hopper}
		}
		\hspace{-0.2cm}
		\subfigure[Regularization in Lunarlander-v3]{
			\includegraphics[width=0.23\textwidth]{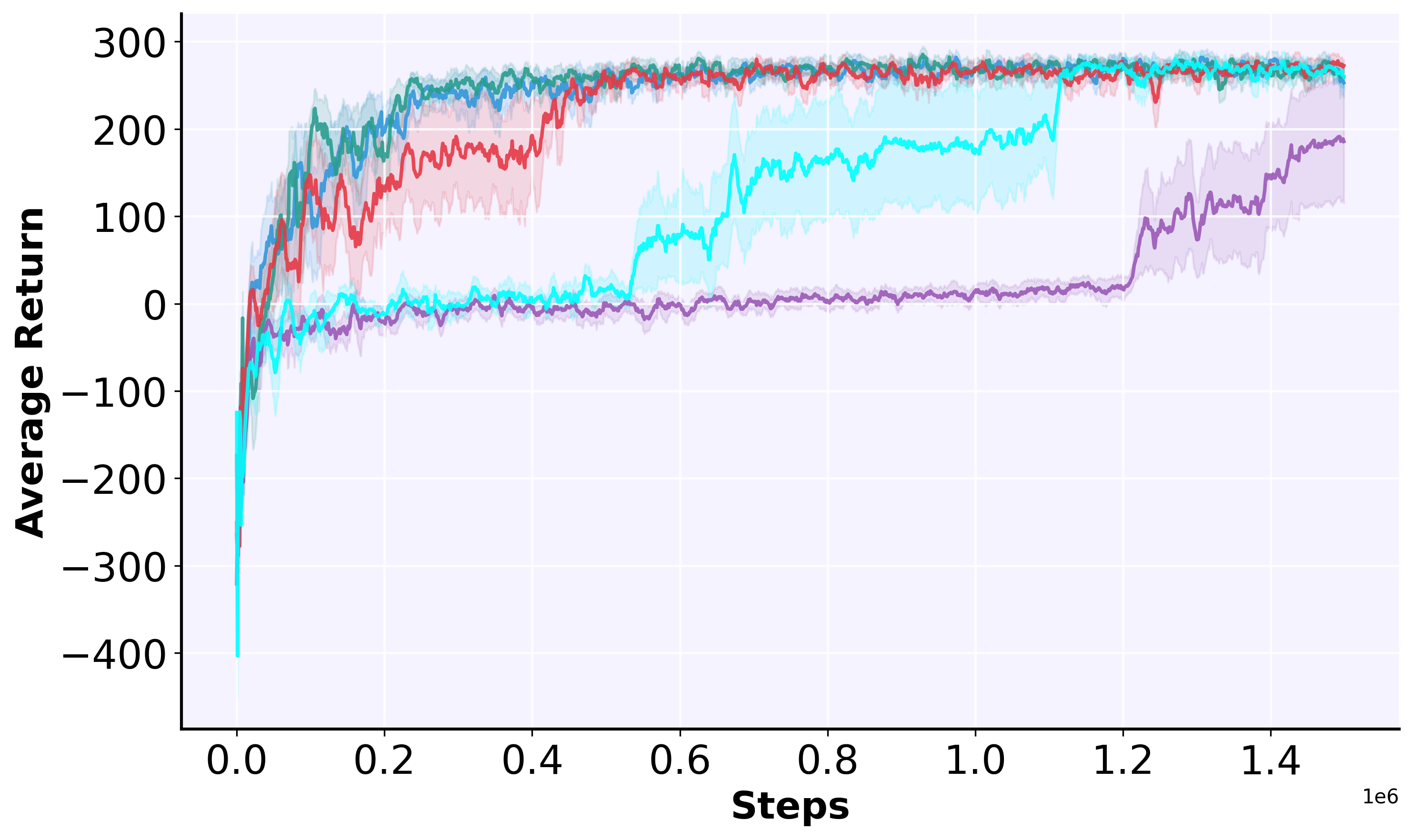}
			\label{fig:regularization_lunarlander}
		}
		
		\vspace{0.15cm}
		\centering
		\begin{tabular}{@{}l@{\hspace{1.5em}}l@{\hspace{1.5em}}l@{\hspace{1.5em}}l@{\hspace{1.5em}}l@{\hspace{1.5em}}l@{}}
			\colorbox{1e0}{\rule{0pt}{1pt}\rule{8pt}{0pt}} \raisebox{-2.0pt}{\scriptsize 1e0} &
			\colorbox{1e-1}{\rule{0pt}{1pt}\rule{8pt}{0pt}} \raisebox{-2.0pt}{\scriptsize 1e-1} &
			\colorbox{1e-2}{\rule{0pt}{1pt}\rule{8pt}{0pt}} \raisebox{-2.0pt}{\scriptsize 1e-2} &
			\colorbox{1e-3}{\rule{0pt}{1pt}\rule{8pt}{0pt}} \raisebox{-2.0pt}{\scriptsize 1e-3} &
			\colorbox{1e-4}{\rule{0pt}{1pt}\rule{8pt}{0pt}} \raisebox{-2.0pt}{\scriptsize 1e-4} &
			\colorbox{1e-5}{\rule{0pt}{1pt}\rule{8pt}{0pt}} \raisebox{-2.0pt}{\scriptsize 1e-5}
		\end{tabular}
		
		\caption{Hyperparameter sensitivity analysis (SAC backbone). Each row tests one parameter ($\alpha_{\text{task}}$, $\alpha_{\text{energy}}$, $\lambda$ from top to bottom); each column corresponds to one environment; each colored line corresponds to one tested value (1e0--1e-5 as shown in the legend). Shaded bands show $\pm$1\,s.d. over 5 seeds. Optimal values for each environment are listed in Table~\ref{tab:hyperparams_standard}. The smooth, bounded response in panels~a--k and the complete failure of $\lambda{=}$1e-5 in panel~l are both consistent with Theorem~\ref{thm:parameter_continuity} and Proposition~\ref{prop:regularization_necessity} respectively.}
		\label{fig:sensitivity}
	\end{figure*}
	
	Theorem~\ref{thm:parameter_continuity} establishes that $\|J(\pi^*_{\lambda+\delta\lambda}) - J(\pi^*_\lambda)\| \leq C|\delta\lambda|\,\mathbb{E}[\mathcal{E}(a)]$, bounding sensitivity to $\lambda$ perturbations. The following experiments extend this to all three H-EARS hyperparameters ($\alpha_{\text{task}}$, $\alpha_{\text{energy}}$, $\lambda$) empirically via a one-factor-at-a-time (OFAT) design: each parameter is swept over six values (1e0--1e-5) while the remaining two are held at their optimal values from Table~\ref{tab:hyperparams_standard}. All experiments use the SAC backbone across all four standard environments (5 seeds each). Figure~\ref{fig:sensitivity} presents the results; three observations emerge.
	
	(1) $\alpha_{\text{task}}$ sensitivity (Fig.~\ref{fig:sensitivity}, panels a--d). All four environments exhibit robust response across the full tested range: curves cluster tightly in Ant-v5, Humanoid-v5, and LunarLander-v3, with Hopper-v5 showing the widest spread yet no catastrophic failure at any tested value. This confirms the initialization guideline: $\alpha_{\text{task}}$ can be set to match the task reward scale without fine-grained tuning, as the energy potential provides a complementary gradient signal that buffers the effect of mild $\alpha_{\text{task}}$ miscalibration.
	
	(2) $\alpha_{\text{energy}}$ sensitivity (Fig.~\ref{fig:sensitivity}, panels e--h). A unimodal response is confirmed across all environments. Hopper-v5 (panel~g) shows the sharpest peak: values above 1e-2 produce measurable degradation consistent with Proposition~\ref{prop:regularization_necessity}'s Case~2, where excessive energy weighting suppresses task-directed gradient signal. LunarLander-v3 (panel~h) shows that large $\alpha_{\text{energy}}$ (1e0) induces early training instability before eventual partial convergence, reflecting the non-mechanical nature of its energy proxy ($\rho_{\text{info}}\approx 1$, Table~\ref{tab:hessian_validity}). Ant-v5 and Humanoid-v5 display a wider plateau, consistent with their globally valid Hessian conditions permitting stronger energy guidance without gradient competition.
	
	(3) $\lambda$ sensitivity (Fig.~\ref{fig:sensitivity}, panels i--l). This is the most environment-dependent parameter. Ant-v5 and Humanoid-v5 (panels~i,~j) are robust across three decades; only $\lambda=\text{1e0}$ produces slight degradation in Ant-v5, where large regularization begins to compete with high-dimensional task gradients. Hopper-v5 (panel~k) shows a well-defined optimal region with degradation at both extremes, consistent with its mixed-alignment character. The most decisive result is LunarLander-v3 (panel~l): $\lambda=\text{1e-5}$ fails completely---training curves remain near $-400$ throughout the run without any meaningful convergence---while all larger values ($\lambda\geq\text{1e-4}$) converge normally to $\approx250$. This directly instantiates Proposition~\ref{prop:regularization_necessity}'s Case~2: in a $\kappa(s)>0$ environment, insufficient regularization leaves oscillatory action sequences unpenalized, preventing policy improvement entirely.
	
	Across all three parameters, the training curves exhibit the bounded, smooth behavior predicted by Theorem~\ref{thm:parameter_continuity}. The sensitivity profiles confirm the scheduling heuristics of Section~\ref{sec:theoretical_framework}: $\alpha_{\text{task}}$ is set first from reward scale, $\alpha_{\text{energy}}$ tuned within the unimodal plateau, and $\lambda$ selected per the $\kappa(s)$ criterion. Since the three parameters act on structurally disjoint channels of the shaped reward (Lemma~\ref{lem:functional_independence}), this sequential tuning procedure does not invalidate earlier choices.
	
	\section{Simulation Validation}
	\label{sec:vehicle}
	
	\subsection{Problem Background}
	
	This section focuses on a four-wheel distributed electric motor independent drive, front-axle steering compact multipurpose vehicle (MPV) as shown in Figure~\ref{fig:Truck}, validating two core theoretical contributions through high-fidelity simulation:
	
	\begin{figure}[htbp]
		\centering
		\includegraphics[width=0.35\textwidth]{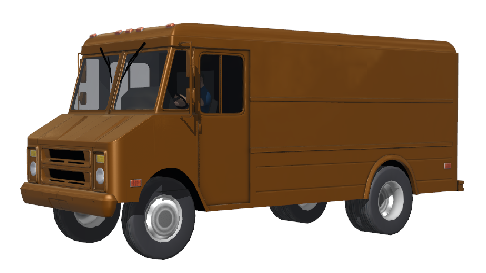}
		\caption{Vehicle simulation model in Trucksim}
		\label{fig:Truck}
	\end{figure}
	
	(i) \textbf{Lyapunov Stability Empiricism (Theorem~\ref{prop:lyapunov_heuristic})}: Validate whether energy potential $\Phi_{\text{energy}}$ achieves implicit Lyapunov stability constraints under extreme conditions (low adhesion + compound slopes), manifested through sideslip angle $\beta$ and yaw rate $r$ convergence characteristics.
	
	(ii) \textbf{Approximate Potential Boundary Test (Lemma~\ref{lem:approx_potential})}: Assess performance retention of simplified energy models (modeling only dominant terms) under out-of-distribution perturbations, validating practical applicability of approximate potential error bound $\delta$.
	
	Vehicle systems exhibit strong nonlinear tire dynamics, multi-actuator coupling, and constraint-intensive characteristics that fully expose theoretical framework adaptation boundaries under complex conditions, which single degree-of-freedom benchmark tasks cannot provide. Additionally, vehicle control domain possesses deep accumulation in physics modeling—two-degree-of-freedom single-track models and Pacejka magic formula tire models enable model-based controllers like MPC to achieve satisfactory performance under standard conditions. However, scenarios of interest present three challenges to traditional simplified models:
	
	(i) Model precision degrades significantly under extreme conditions—when road adhesion coefficient drops below 0.3 or compound slopes exist, linearization assumptions fail and tire cornering stiffness varies dramatically.
	
	(ii) High degree-of-freedom distributed drive control—multi-wheel independent torque allocation introduces multiple control degrees of freedom, with traditional simplified models struggling to capture inter-wheel coupling dynamics.
	
	(iii) Data-driven RL adaptive capability may exploit high-order nonlinear characteristics that explicit models cannot capture.
	
	\subsection{Simulation Configuration}
	
	The study employs a four-wheel distributed-drive electric MPV with parameters detailed in Table~\ref{tab:vehicle_params}. Systematic comparison experiments are conducted through TruckSim. Controller training simulations use 1000m straight road segments containing varying longitudinal/lateral slope combinations and randomly distributed road adhesion coefficients. Test road segment spans 300m straight with randomly distributed low-adhesion areas ($\mu \in [0.1, 1.0]$) and compound slopes (lateral slope $\leq$ 15°, longitudinal slope $\leq$ 20°). Vehicle initial speed is 0 m/s with target speed set at 15 m/s. Training and test road parameter settings are shown in Figures~\ref{fig:training road} and \ref{fig:test road}.
	
	\begin{table}[!t]
		\renewcommand{\arraystretch}{1.3}
		\captionsetup{font=footnotesize, labelfont=footnotesize}
		\caption{Vehicle Body Parameter Settings}
		\label{tab:vehicle_params}
		\centering
		\footnotesize
		\begin{tabular}{ccc}
			\toprule
			\textbf{Parameter} & \textbf{Symbol/Unit} & \textbf{Value} \\
			\midrule
			Total Vehicle Mass & $m$/kg & 2100 \\
			Yaw Moment of Inertia & $I_z$/kg·m$^2$ & 4116 \\
			Center of Mass Height & $h$/mm & 710 \\
			Distance from CoG to Front Axle & $a$/mm & 1350 \\
			Distance from CoG to Rear Axle & $b$/mm & 1450 \\
			Front Track Width & $d_f$/mm & 1800 \\
			Rear Track Width & $d_r$/mm & 1800 \\
			\bottomrule
		\end{tabular}
	\end{table}
	
	\begin{figure}[!t]
		\centering
		\subfigure[Training road height variation with position]{
			\includegraphics[width=0.22\textwidth]{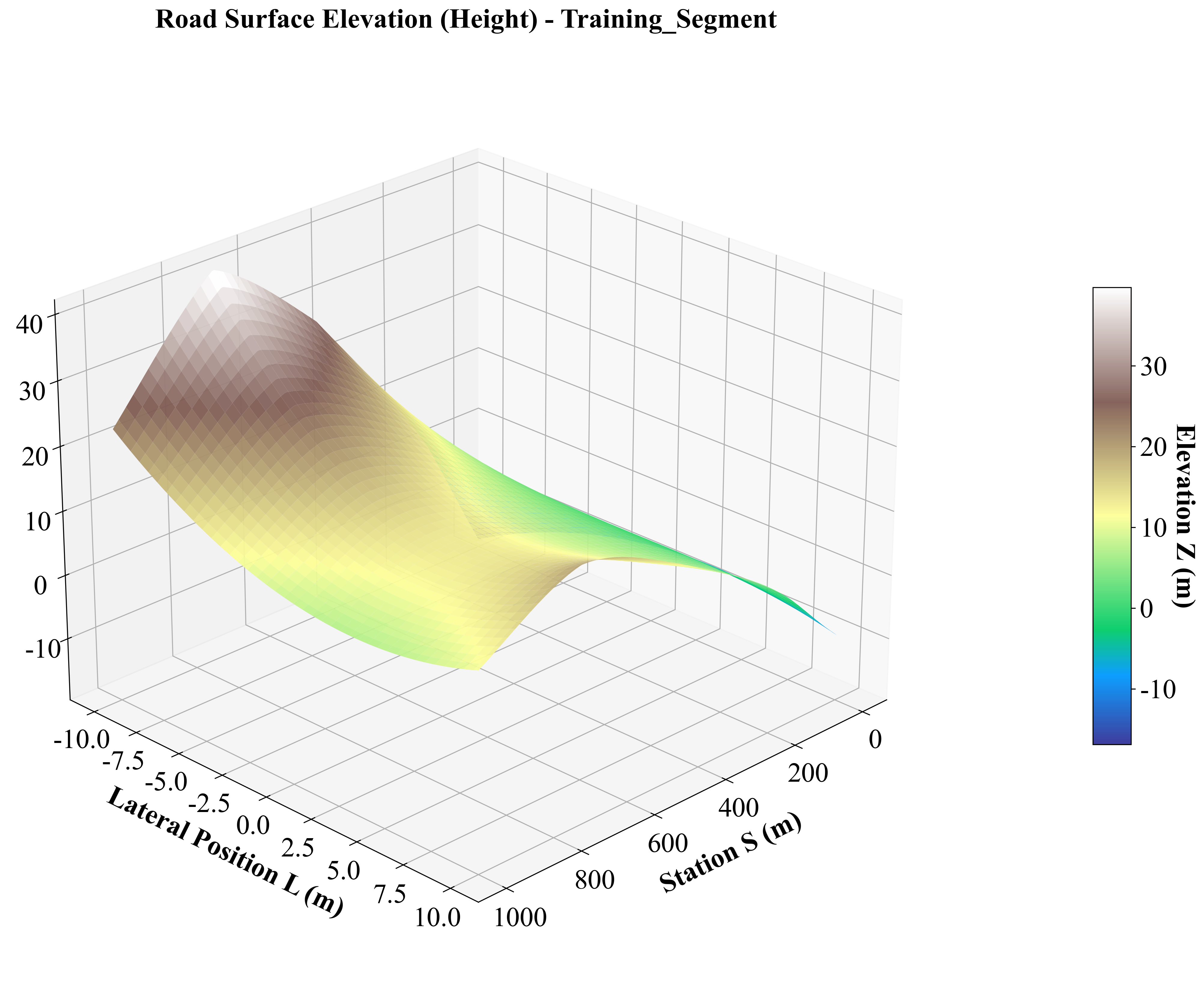}
			\label{fig:train_height}
		}
		\hfill
		\subfigure[Training road adhesion coefficient variation with position]{
			\includegraphics[width=0.22\textwidth]{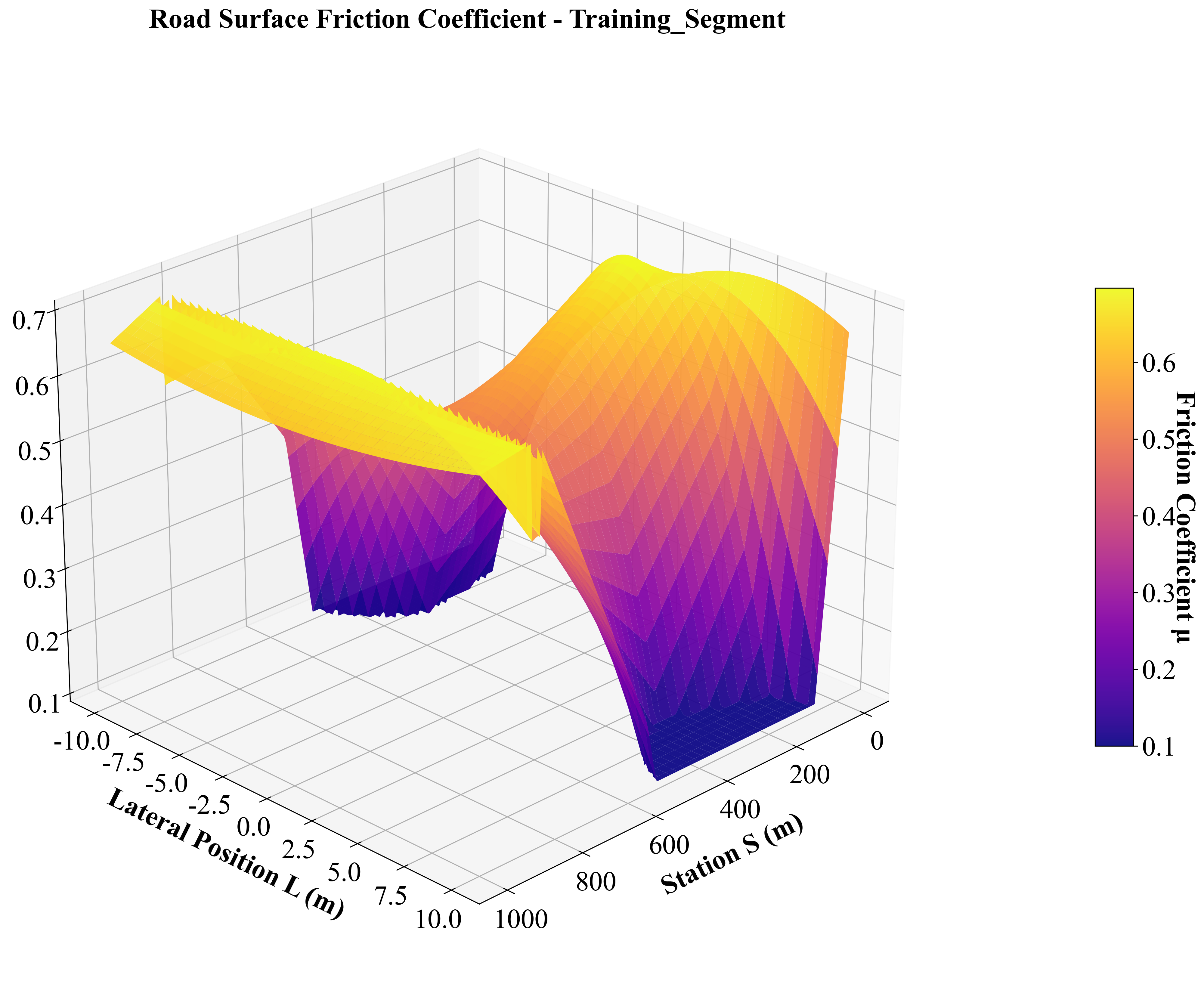}
			\label{fig:train_coefficient}
		}
		\caption{Training road height and adhesion coefficient variation parameter settings}
		\label{fig:training road}
	\end{figure}
	
	\begin{figure}[!t]
		\centering
		\subfigure[Test road height variation with position]{
			\includegraphics[width=0.22\textwidth]{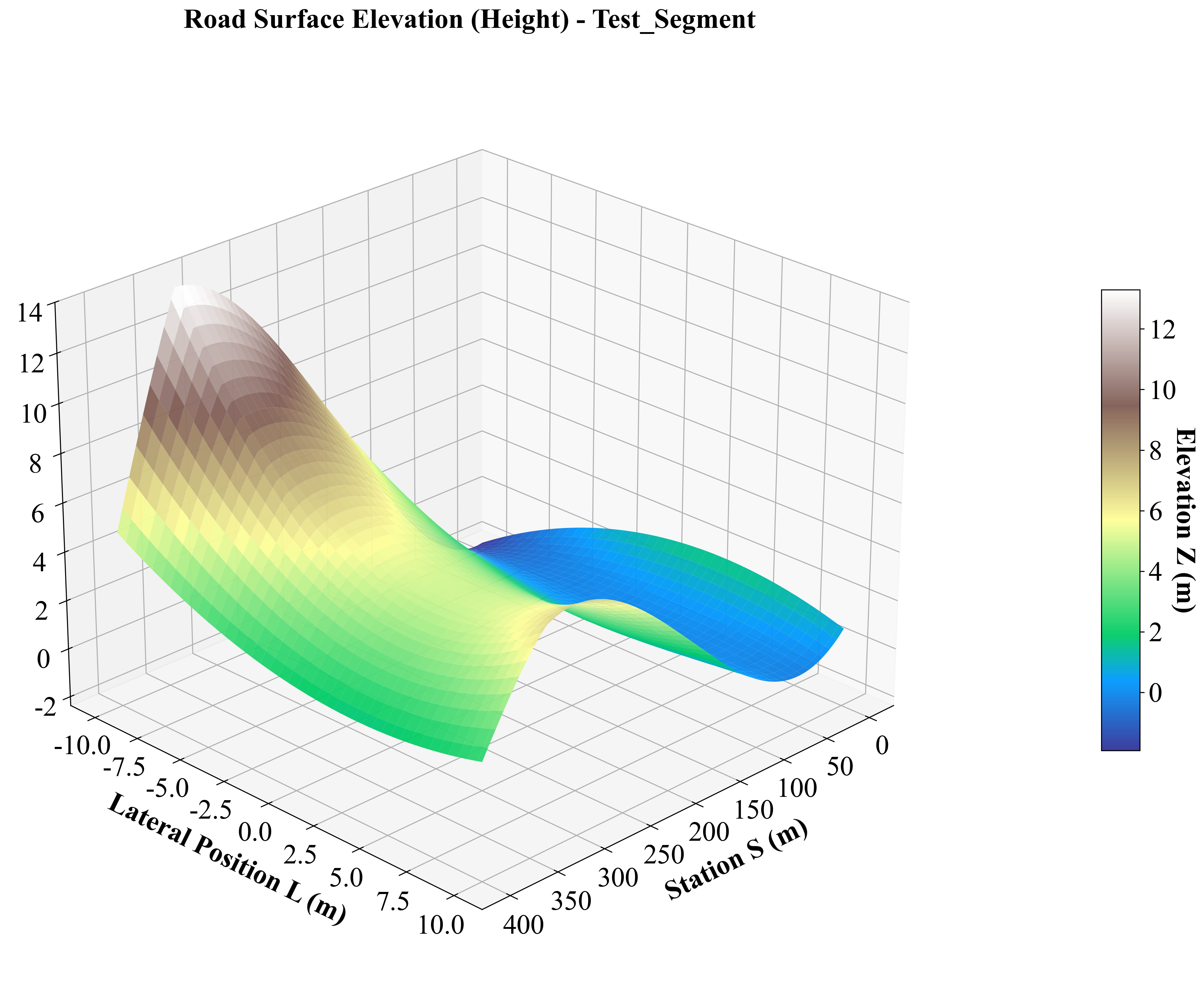}
			\label{fig:test_height}
		}
		\hfill
		\subfigure[Test road adhesion coefficient variation with position]{
			\includegraphics[width=0.22\textwidth]{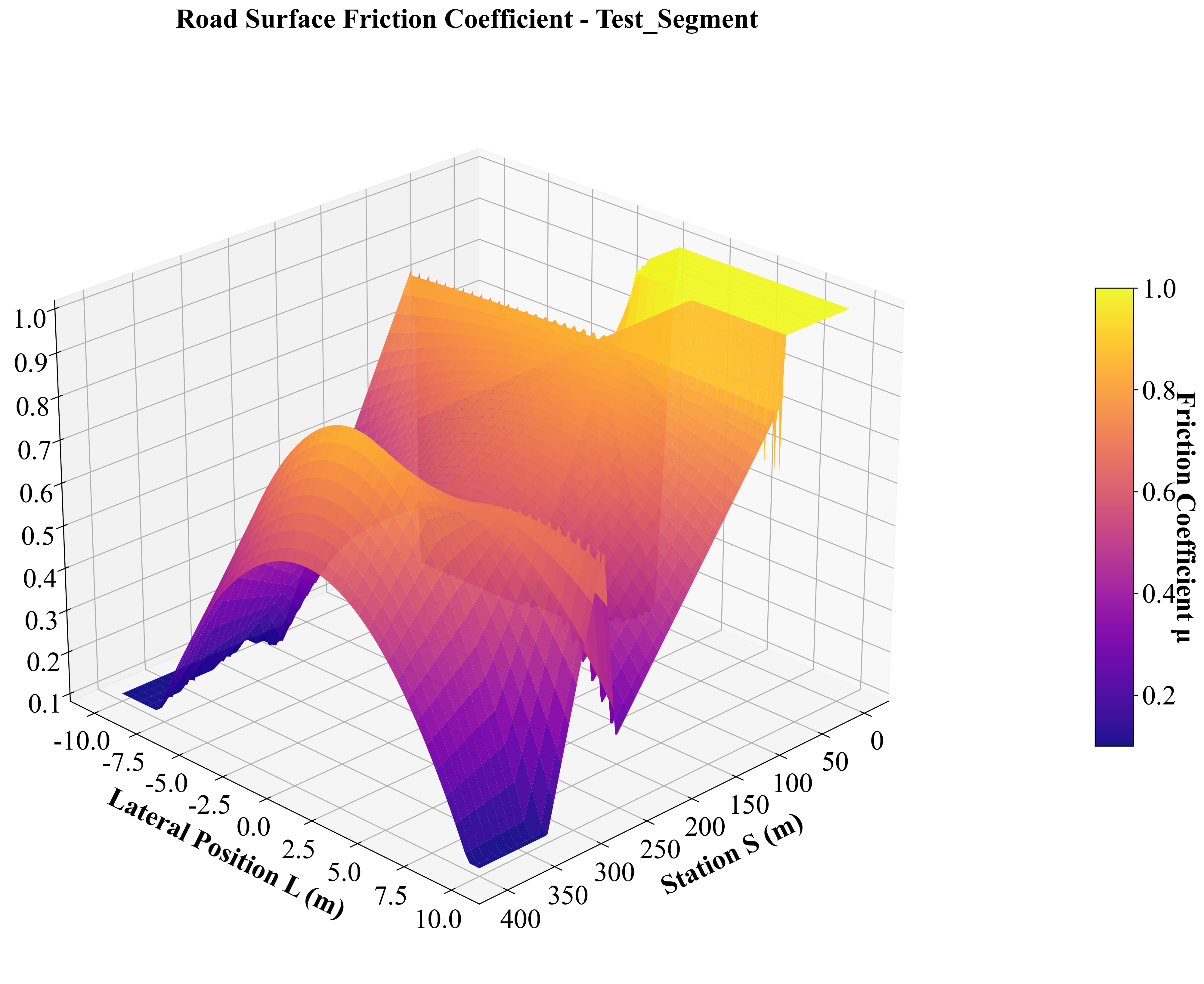}
			\label{fig:test_coefficient}
		}
		\caption{Test road height and adhesion coefficient variation parameter settings}
		\label{fig:test road}
	\end{figure}
	
	\subsection{Simulation Platform and Control Framework}
	
	TruckSim-Python co-simulation architecture is adopted, implementing 50\,Hz state information exchange and control command transmission through TruckSim's API interface. The vehicle experiments serve a dual validation purpose beyond standard performance measurement: (i) they validate the approximate energy modeling framework (Lemma~\ref{lem:approx_potential}) by testing H-EARS under intentionally degraded energy models with approximation errors $\epsilon_{\text{approx}} \in \{10\%, 30\%\}$, directly quantifying the bound's tightness in a safety-critical domain; and (ii) the extreme road adhesion conditions ($\mu \in \{0.8, 0.4, 0.2\}$) constitute controlled disturbance tests for the stability-oriented guidance of Proposition~\ref{prop:lyapunov_heuristic}, where low-$\mu$ surfaces act as external perturbations to the dissipation property $\dot{E} \leq 0$. This framing connects the vehicle experiments directly to two open theoretical questions from the standard environment analysis: the robustness of energy guidance under model approximation, and the stability behavior when the Lyapunov-inspired dissipation condition is challenged by external disturbances.
	
	\subsubsection{Safety Integration Framework}
	\label{sec:safety_framework}
	
	H-EARS shapes the reward landscape to promote physically consistent behavior; formal constraint guarantees require a complementary safety layer. Two integration patterns are identified.
	
	Pattern 1---RL+MPC (implemented). The upper H-EARS policy generates reference states through energy-aware reward shaping. The lower MPC (Section~\ref{sec:mpc_arch}) converts these references to actuator commands while solving the constrained optimisation problem at every control step, enforcing $u_{\min}\leq u_k\leq u_{\max}$ and $\|\Delta u_k\|\leq\Delta u_{\max}$ regardless of the RL policy's output. Crucially, hard-constraint satisfaction is guaranteed by the MPC layer independently of policy quality---H-EARS and the safety layer have non-overlapping responsibilities.
	
	Pattern 2---Shielded RL (for constraint-critical deployments without a lower-level controller). A safety filter monitors the RL policy's action at each step and projects any constraint-violating action to the boundary of the safe set before execution~\cite{hewing2020learning}. H-EARS shaping reduces the frequency and magnitude of such projections by guiding the policy toward physically consistent regions; the filter provides the formal guarantee.
	
	\subsubsection{$\lambda$ Calibration Under Varying Safety Requirements}
	\label{sec:lambda_calib}
	
	Proposition~\ref{prop:regularization_necessity} provides the theoretical basis for $\lambda$ selection: $\lambda$ should be increased when $\kappa(s)>0$ becomes prevalent in task-critical regions, i.e.\ when energy-minimisation and task gradients conflict. In the vehicle domain, this conflict intensifies as road adhesion $\mu$ decreases: at low $\mu$, tyre forces saturate and yaw control demands high-rate torque inputs that conflict with kinetic-energy minimisation, driving $\kappa(s)$ toward positive values and increasing the risk of oscillatory yaw artifacts (Proposition~\ref{prop:regularization_necessity} Case~2). The calibration procedure follows three steps:
	
	(i)Baseline: Set $\lambda$ to the value calibrated for normal operating conditions. In the vehicle experiments, $\lambda=0.20$ is identified by grid search as the value that balances yaw stability and speed-tracking accuracy at $\mu\geq0.4$.
	
	(ii)Safety-budget adjustment: Under stricter requirements (lower $\mu$, tighter yaw-rate limits), increase $\lambda$ in increments $\delta\lambda$. The maximum performance cost per increment is bounded a priori by Theorem~\ref{thm:parameter_continuity}:
	\begin{equation}
		\|J(\pi^*_{\lambda+\delta\lambda}) - J(\pi^*_\lambda)\| \leq C\,|\delta\lambda|\,\mathbb{E}_{\pi^*_\lambda}[\mathcal{E}(a)]
		\label{eq:lambda_tradeoff}
	\end{equation}
	where $C$ depends on the MDP structure and $\mathbb{E}[\mathcal{E}(a)]$ can be estimated from a short rollout of the current policy. This bound makes the speed-recovery cost of each $\lambda$ increment predictable before committing to retraining.
	
	(iii)Stopping criterion: Stop increasing $\lambda$ when the observed yaw-rate oscillation amplitude falls below the domain-specific safety threshold. Because the $\lambda$-to-performance mapping is Lipschitz continuous (Theorem~\ref{thm:parameter_continuity}), this procedure converges to a locally optimal safety-performance operating point without exhaustive search.
	
	\subsection{Vehicle Dynamics Modeling}
	
	The two-degree-of-freedom vehicle dynamics model is employed, applicable to small-angle cornering and medium-low speed driving conditions. Vehicle motion in the horizontal plane is described by:
	\begin{equation}
		\begin{aligned}
			m(\dot{v}_y+v_xr) &= F_{yf} + F_{yr} \\
			I_z\dot{r} &= aF_{yf} - bF_{yr} + M_z
		\end{aligned}
		\label{eq:vehicle_dynamics}
	\end{equation}
	
	Tire longitudinal and lateral forces are calculated through Pacejka magic formula. Sideslip angle $\alpha_i$ is calculated based on vehicle motion state and steering angle.
	
	\subsection{Lower-Level MPC Controller Architecture}
	\label{sec:mpc_arch}
	
	The MPC controller converts RL-generated reference values to specific actuator commands while ensuring physical constraint satisfaction. At each control period, MPC solves the finite-horizon optimization problem:
	\begin{equation}
		\begin{aligned}
			\min_{u_{0:N_c-1}} \quad & \sum_{k=0}^{N_p-1} \|y_k-y_{\text{ref},k}\|_{Q_k}^2 + \sum_{k=0}^{N_c-1} \|\Delta u_k\|_{R_k}^2 \\
			\text{s.t.} \quad & x_{k+1} = f(x_k, u_k), \quad y_k = g(x_k) \\
			& u_{\min} \leq u_k \leq u_{\max}, \quad \|\Delta u_k\| \leq \Delta u_{\max}
		\end{aligned}
	\end{equation}
	
	\subsection{Simulation Results and Analysis}
	
	\subsubsection{Algorithm Convergence Analysis}
	Figure~\ref{fig:Sim_Reward} shows cumulative reward evolution trends of different algorithms during training process.
	
	\begin{figure}[!t]
		\centering
		\includegraphics[width=0.42\textwidth]{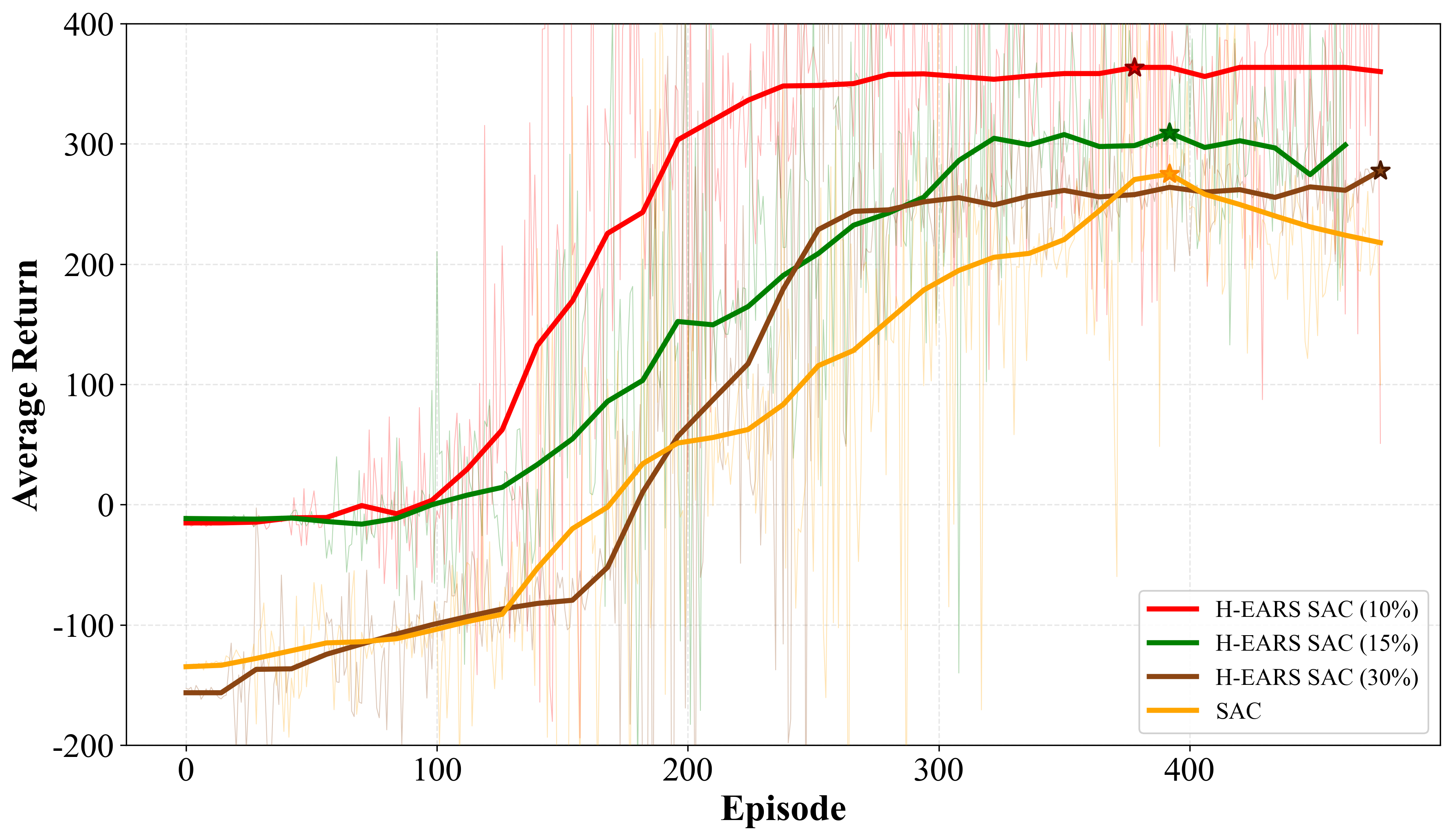}
		\caption{Training return across four variants}
		\label{fig:Sim_Reward}
	\end{figure}
	
	Figure~\ref{fig:Sim_Reward} compares training return for H-EARS variants across three approximation levels and vanilla SAC; quantitative convergence metrics are in Table~\ref{tab:convergence}.
	
	Monotone degradation with approximation error. Final stable returns satisfy the ordering predicted by Lemma~\ref{lem:approx_potential}: $\varepsilon{=}10\%$ ($357{\pm}14$) $>$ $\varepsilon{=}15\%$ ($311{\pm}12$) $>$ $\varepsilon{=}30\%$ ($280{\pm}22$) $>$ SAC ($275{\pm}18$). All three H-EARS variants outperform SAC by 29.8\%, 13.1\%, and 1.8\% respectively, confirming graceful degradation rather than performance collapse at elevated approximation error; even the coarsest 30\% model retains substantial physics-guided advantage.
	
	Convergence acceleration and initialization. Energy-aware potentials raise initial returns from ${\approx}{-130}$ (SAC) to ${\approx}{-40}$ across all H-EARS variants, reflecting the gradient informativeness of $\Phi_{\text{energy}}$ from the first episode (Theorem~\ref{thm:energy_convergence}), irrespective of approximation level. Threshold crossings (265.5, defined as 88.5\% of the $\varepsilon{=}15\%$ stable return) occur at episodes 179, 281, 348, and 376 for the 10\%, 15\%, 30\%, and SAC configurations respectively (Table~\ref{tab:convergence}). The 10\% model reaches threshold 52.4\% earlier than SAC; the 15\% model 25.3\% earlier, consistent with Theorem~\ref{thm:energy_convergence}'s gradient-enrichment acceleration mechanism.
	
	\begin{table}[!t]
		\renewcommand{\arraystretch}{1.3}
		\captionsetup{font=footnotesize, labelfont=footnotesize}
		\caption{Training Performance Comparison Across Approximation Levels}
		\label{tab:convergence}
		\centering
		\footnotesize
		\resizebox{1.0\columnwidth}{!}{%
			\begin{tabular}{lccc}
				\toprule
				\textbf{Algorithm} & \makecell{\textbf{Episodes to}\\\textbf{Threshold}} & \makecell{\textbf{Avg.\ Return}\\\textbf{(Stable)}} & \textbf{CV (\%)} \\
				\midrule
				H-EARS SAC ($\varepsilon_{\text{approx}}{=}10\%$) & \textbf{179} & \textbf{357.0$\pm$14.0} & \textbf{3.8} \\
				H-EARS SAC ($\varepsilon_{\text{approx}}{=}15\%$) & 281          & 311.0$\pm$12.0          & 4.0          \\
				H-EARS SAC ($\varepsilon_{\text{approx}}{=}30\%$) & 348          & 280.0$\pm$22.0          & 5.8          \\
				SAC                                               & 376          & 275.0$\pm$18.0          & 8.3          \\
				\bottomrule
			\end{tabular}
		}
		\vspace{2mm}
		
		\footnotesize
		\textit{Note:} Threshold $= 265.5$ (88.5\% of $\varepsilon{=}15\%$ stable return). Bold: best value per column.
	\end{table}
	
	\subsubsection{Quantitative Performance Comparison}
	
	Table~\ref{tab:trucksim_results_new} presents vehicle control performance for all three H-EARS approximation levels against vanilla SAC, directly validating Lemma~\ref{lem:approx_potential} across the full $\varepsilon_{\text{approx}}$ range.
	
	\begin{table}[!t]
		\renewcommand{\arraystretch}{1.2}
		\captionsetup{font=footnotesize, labelfont=footnotesize}
		\caption{Vehicle Control Performance Under Extreme Conditions: Cross-Approximation Comparison}
		\label{tab:trucksim_results_new}
		\centering
		\footnotesize
		\resizebox{1.0\columnwidth}{!}{%
			\begin{tabular}{lcccc}
				\toprule
				\textbf{Metric} & \makecell{\textbf{H-EARS}\\\textbf{($\varepsilon{=}10\%$)}} & \makecell{\textbf{H-EARS}\\\textbf{($\varepsilon{=}15\%$)}} & \makecell{\textbf{H-EARS}\\\textbf{($\varepsilon{=}30\%$)}} & \textbf{SAC} \\
				\midrule
				Avg Speed Error (m/s)  & $2.51\pm0.32$   & $\mathbf{0.23\pm0.12}$ & $0.35\pm0.20$   & $2.9\pm0.19$ \\
				Max Sideslip (°)       & $\mathbf{0.30}$ & $0.52$                 & $0.82$          & $1.03$        \\
				Max $|r|$ (°/s)        & $\mathbf{0.51}$ & $1.18$                 & $2.71$          & $3.24$        \\
				CV of Speed (\%)       & $\mathbf{2.8}$  & $5.2$                  & $7.1$           & $7.8$         \\
				\bottomrule
			\end{tabular}%
		}
		\vspace{2mm}
		
		\footnotesize
		\textit{Note:} Metrics over 300\,m test segment ($\mu\in[0.1,1.0]$; lateral slope ${\leq}15^\circ$, longitudinal ${\leq}20^\circ$). Speed tracking: the 15\% model achieves best accuracy; the 10\% model is over-constrained due to task--energy gradient competition (see text). Stability metrics (Max Sideslip, Max\,$|r|$): monotonically improve with decreasing $\varepsilon_{\text{approx}}$. Bold: best value per metric.
	\end{table}
	
	\textbf{Speed Tracking Accuracy.} Among the three H-EARS variants, the $\varepsilon{=}15\%$ reference model achieves the best speed tracking accuracy (0.23$\pm$0.12\,m/s, CV\,5.2\%), followed by $\varepsilon{=}30\%$ (0.35$\pm$0.20\,m/s). The $\varepsilon{=}10\%$ model exhibits a systematic positive speed bias (2.51$\pm$0.32\,m/s) despite its lowest CV (2.8\%), yet still outperforms SAC (2.9$\pm$0.19\,m/s, CV\,7.8\%) by 13.4\%. This non-monotone pattern among H-EARS variants arises from Proposition~\ref{prop:regularization_necessity} Case\,2: the additional energy terms in the 10\% model (pitch/roll kinetic energy, tyre deformation potential) have gradients partially misaligned with the speed-tracking objective---i.e.\ $\kappa(s){>}0$ in speed-critical operating regions---creating a systematic pull toward lower-velocity, energy-minimising configurations. The $\lambda{=}0.01$ value, calibrated for the 15\% reference model, is insufficient to suppress this gradient competition at 10\% completeness; Theorem~\ref{thm:parameter_continuity} bounds the resulting performance shift as $\|J(\pi^*_{\lambda+\delta\lambda}) - J(\pi^*_\lambda)\|\leq C|\delta\lambda|\,\mathbb{E}[\mathcal{E}(a)]$, suggesting a larger $\lambda$ would restore tracking accuracy at the cost of reduced stability margins. Conversely, on stability metrics (below), the 10\% model's extra energy terms align with the lateral-stability objective ($\kappa(s){<}0$) and yield the best performance---demonstrating that the optimal approximation level depends on task--energy gradient alignment, not model completeness alone.
	
	\begin{figure}[!t]
		\centering
		\includegraphics[width=0.48\textwidth]{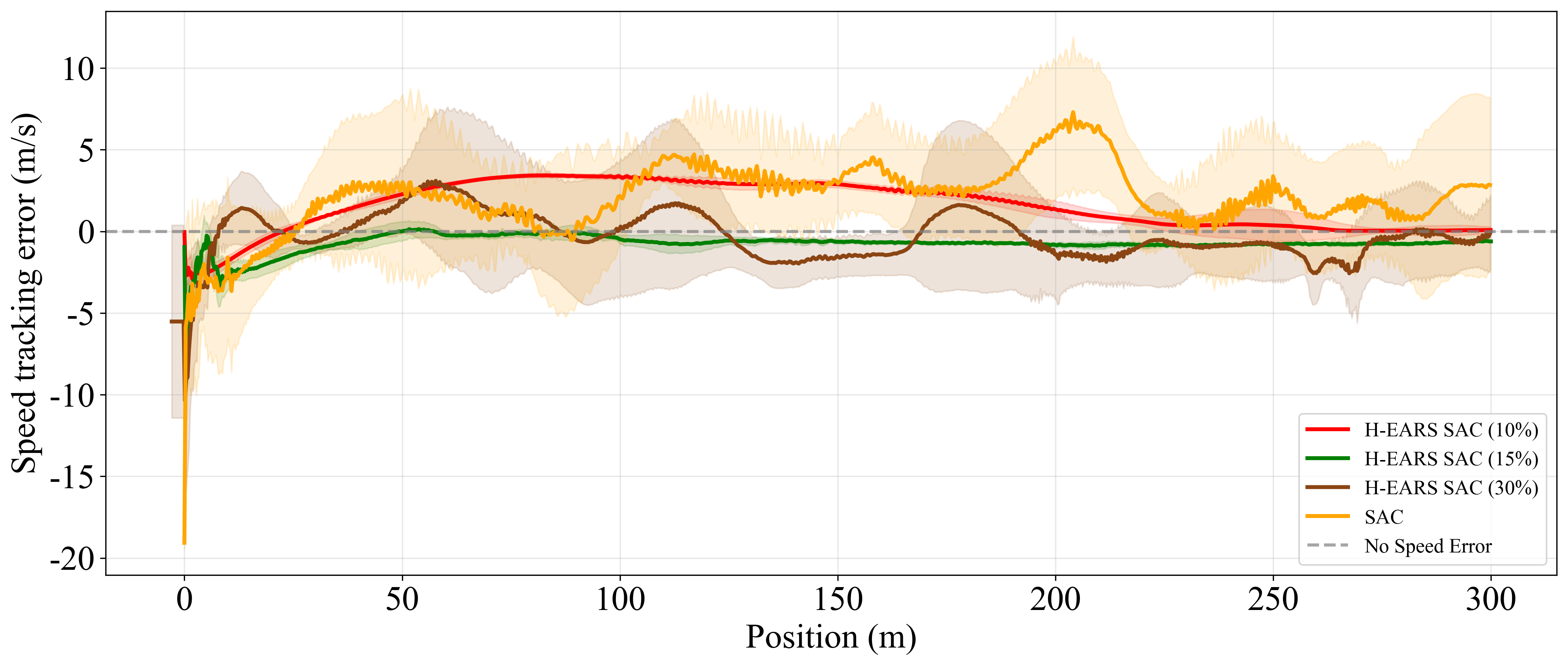}
		\caption{Longitudinal speed tracking error variation with position for each controller}
		\label{fig:velocity_error}
	\end{figure}
	
	\textbf{Lateral Stability Enhancement.} Stability metrics improve monotonically with decreasing $\varepsilon_{\text{approx}}$, directly validating Lemma~\ref{lem:approx_potential}'s graceful-degradation prediction: Max Sideslip progresses from 0.30$^\circ$ ($\varepsilon{=}10\%$) to 0.52$^\circ$ ($\varepsilon{=}15\%$) to 0.82$^\circ$ ($\varepsilon{=}30\%$) to 1.03$^\circ$ (SAC); Max\,$|r|$ progresses from 0.51 to 1.18 to 2.71 to 3.24\,$^\circ$/s over the same sequence. All three H-EARS variants substantially outperform SAC: the 30\% model alone reduces Max Sideslip by 20.4\% and Max\,$|r|$ by 16.4\% relative to SAC, confirming meaningful physics-guided stabilisation even under 30\% approximation error. The monotone stability ordering is consistent with the gradient enrichment mechanism of Theorem~\ref{thm:energy_convergence}: richer energy models supply stronger implicit Lyapunov regularisation (Proposition~\ref{prop:lyapunov_heuristic}), as lateral kinetic energy terms are structurally aligned with the stability objective across all approximation levels.
	
	\begin{figure}[!t]
		\centering
		\includegraphics[width=0.48\textwidth]{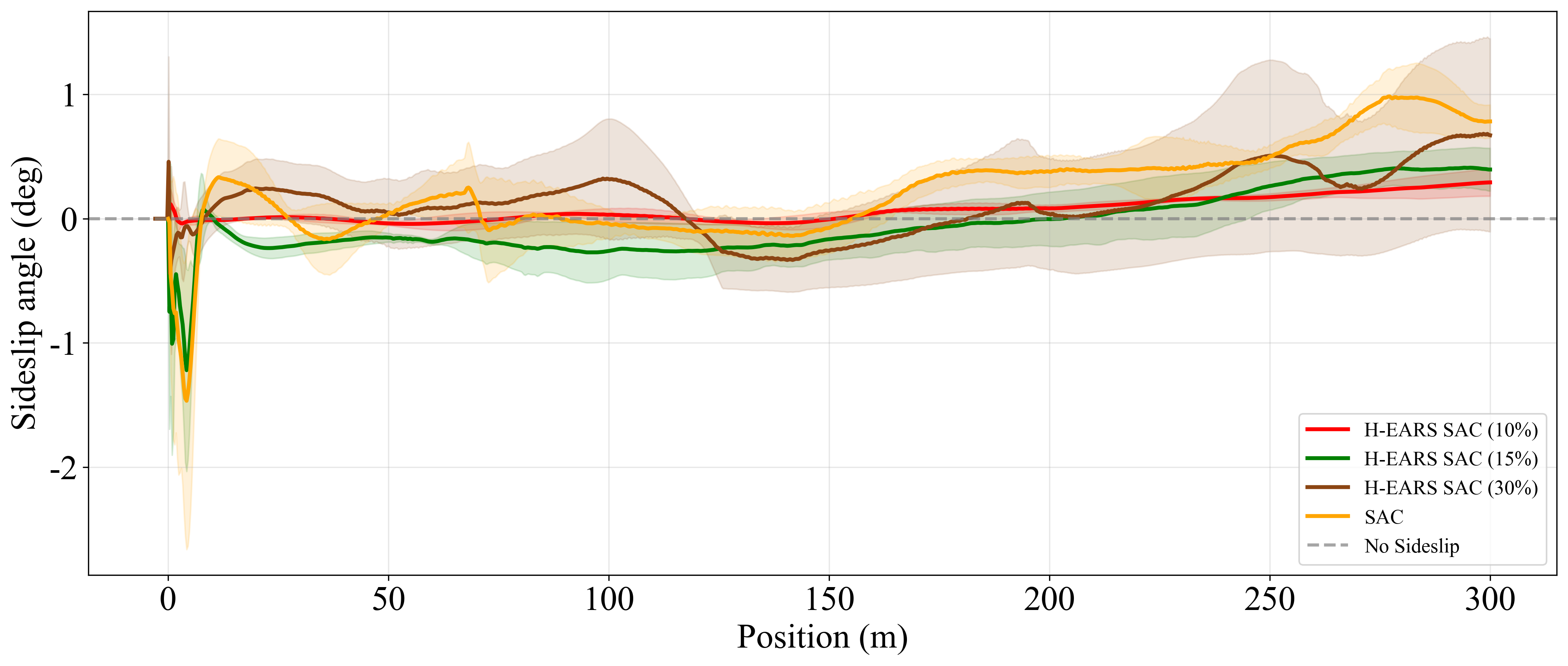}\\
		{\fontsize{7}{6}\selectfont (a) Sideslip angle variation with position}
		\vspace{0.2em}
		\includegraphics[width=0.48\textwidth]{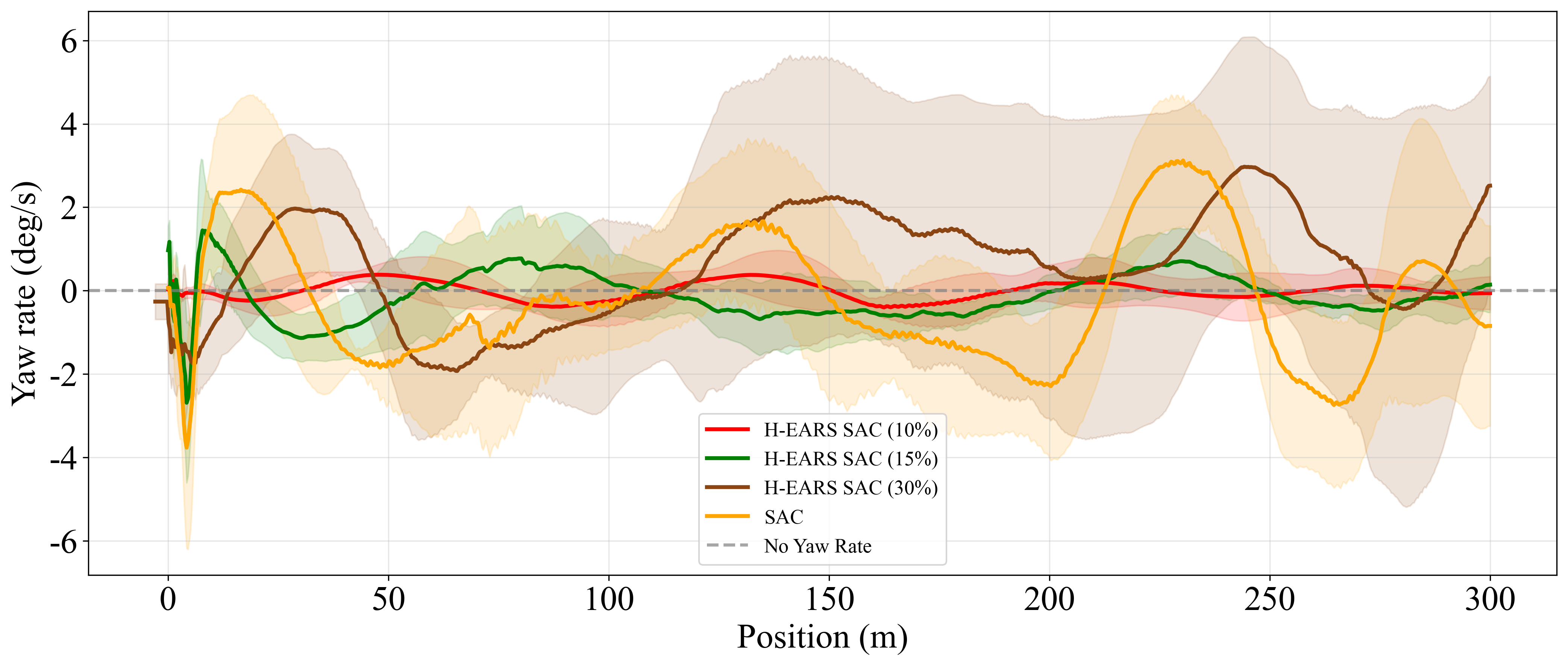}\\
		{\fontsize{7}{8}\selectfont (b) Yaw rate variation with position}
		\caption{Sideslip angle and yaw rate variation with position for each controller}
		\label{fig:slide_Yaw}
	\end{figure}
	
	\subsubsection{Lyapunov Stability Empirical Validation}
	
	Figure~\ref{fig:slide_Yaw} illustrates sideslip angle and yaw rate evolution under extreme conditions for all three H-EARS approximation levels and vanilla SAC, providing direct empirical evidence for Proposition~\ref{prop:lyapunov_heuristic}. The monotone performance ordering across approximation levels is quantified in Table~\ref{tab:trucksim_results_new}; the analysis below focuses on the $\varepsilon{=}15\%$ reference model as representative of shared characteristics across all H-EARS variants.
	
	\textbf{Monotonic Convergence Characteristics}: Defining Lyapunov candidate function $L = \frac{1}{2}I_z r^2 + k_\beta \beta^2$ (corresponding to energy potential $\Phi_{\text{energy}} = -L$), the $\varepsilon{=}15\%$ reference model exhibits representative monotonic convergence behaviour shared across all H-EARS variants. In the 0-50m initial segment, sideslip angle $\beta$ rapidly decays from peak 0.25° to within ±0.2°, with $L(t)$ presenting monotonic decreasing trend satisfying $\dot{L}<0$. Throughout the 300m test segment, H-EARS maintains absolute sideslip peak below 0.52° with overall fluctuation within ±0.4°. In contrast, SAC exhibits periodic instability with sideslip peaks breaking 1.0° in the 200-300m segment, particularly at road condition mutation points (100m, 200m, 250m positions), indicating failure to maintain $\dot{L}<0$ energy dissipation property.
	
	\textbf{Energy Dissipation Mechanism}: Proposition~\ref{prop:lyapunov_heuristic} establishes that maximizing $\Phi_{\text{energy}}$ is equivalent to controlling along $-\nabla L$ direction. Empirical observations confirm this prediction: the $\varepsilon{=}15\%$ model's yaw rate adjustment amplitude at road condition mutations (150\,m, 240\,m positions) remains $<$1\,$^\circ$/s, significantly smaller than SAC's 3--5\,$^\circ$/s mutations. This validates energy potential's predictive constraint effect—the system automatically avoids aggressive actions generating large-amplitude yaw energy $\frac{1}{2}I_z r^2$. SAC yaw oscillation amplitude reaches ±4-5°/s in the 200-250m compound slope segment, corresponding to non-monotonic growth of $L$ function, violating Lyapunov stability principles.
	
	\textbf{Approximate Model Robustness.} Experiments across $\varepsilon_{\text{approx}}\in\{10\%,15\%,30\%\}$ directly validate Lemma~\ref{lem:approx_potential}'s graceful-degradation prediction. On lateral stability metrics, performance degrades monotonically with increasing $\varepsilon_{\text{approx}}$ (Table~\ref{tab:trucksim_results_new}): Max Sideslip progresses from 0.30$^\circ$ to 0.52$^\circ$ to 0.82$^\circ$, while all three levels remain substantially below SAC's 1.03$^\circ$. The relative performance degradations in stable return from the 10\% model are $-12.9\%$ ($\varepsilon{=}15\%$) and $-21.6\%$ ($\varepsilon{=}30\%$), both well-bounded even at 30\% approximation error. Speed tracking exhibits a non-monotone pattern arising from task--energy gradient competition at 10\% completeness. These results confirm that dominant-term energy models capture sufficient physical structure for effective guidance across the full tested range, without requiring complete dynamics identification.
	
	\textbf{Physics Mechanism Interpretation}: The equivalence relation $\max \Phi_{\text{energy}} \Leftrightarrow \min \dot{L}$ established by Theorem~\ref{prop:lyapunov_heuristic} demonstrates that RL policies naturally learn passive stability through energy minimization principles implicitly encoded by PBRS. Without explicit complete vehicle dynamics modeling, relying solely on gradient information of dominant energy terms maintains Lyapunov control law-like asymptotic stability characteristics under extreme conditions. Yaw rate's smooth evolution (low-frequency characteristics in Figure~\ref{fig:slide_Yaw}(b)) further validates approximate validity conditions: under settings $\gamma=0.99$, $\Delta t=0.02$s, dominant gradient term $\gamma\nabla_s\Phi \cdot f\Delta t$ of PBRS far exceeds constant correction term $(1-\gamma)|\Phi|$ and high-order error $O(\Delta t^2)$, making discrete-time approximation bias $<$3\%. This explains why simplified energy models realize theoretically predicted stability guarantees in high-fidelity TruckSim environments.
	
	\section{Limitations and Future Work}
	
	Scope of the Hessian condition. The gradient enrichment guarantee of Theorem~\ref{thm:energy_convergence} is predicated on $\frac{\partial^2 E}{\partial q^2}\succ 0$ over $\mathcal{S}_{\text{op}}$. For mechanical systems governed by Lagrangian dynamics this condition is structurally satisfied by the positive-definite inertia matrix $M(q)$, covering the full class of locomotion, manipulation, and vehicle control tasks studied here. Where the condition holds only locally---e.g., large-angle pendulums or gaits with ground-contact discontinuities---the enrichment guarantee applies within that sub-region and the framework degrades gracefully to TaskOnly-level performance elsewhere (ablation, Figure~\ref{fig:ablation}), with the $O(1/\sqrt{N})$ convergence guarantee of Theorem~\ref{thm:convergence} remaining intact throughout.
	
	Extension to non-mechanical domains. H-EARS assumes that dominant energy terms are identifiable a priori from physical knowledge. Extending the framework to domains without natural mechanical structure requires constructing a scalar proxy $E_{\text{proxy}}$ satisfying $\partial^2 E_{\text{proxy}}/\partial q^2\succ 0$ over the task-relevant region, after which Theorem~\ref{thm:energy_convergence} and Lemma~\ref{lem:approx_potential} apply without modification. Standard constructions exist for several engineering domains---reservoir-level control ($E_{\text{proxy}}(h)=\tfrac{1}{2}k(h-h^*)^2$, $\partial^2/\partial h^2 = k>0$) and power-grid frequency regulation ($E_{\text{proxy}}(\Delta f)=\tfrac{1}{2}(\Delta f)^2$, $\partial^2/\partial(\Delta f)^2=1>0$) being representative examples---and we identify systematic proxy design for cyber-physical and economic systems as a productive direction for future work.
	
	\section{Conclusion}
	
	This paper presents H-EARS, a framework that encodes dominant energy terms---assumed known a priori---directly as reward potentials, enabling physics-guided reinforcement learning without modifying the underlying algorithm or requiring complete system identification.
	
	Three theoretical contributions are established. First, when the energy Hessian is locally positive definite over the task-relevant operating region, the energy potential provides gradient enrichment throughout the state space, accelerating convergence in domains where task rewards are sparse (Theorem~\ref{thm:energy_convergence}). Second, dual-potential decomposition into task-oriented and energy-based components is shown to be mathematically necessary: a single potential cannot simultaneously satisfy task directivity and energy awareness when objectives are misaligned (Proposition~\ref{prop:dual_necessity}). Third, performance bounds for approximate energy models quantify the trade-off between modeling effort and guidance quality---relative approximation errors up to 20\% yield less than 5\% performance loss under typical hyperparameters (Lemma~\ref{lem:approx_potential}).
	
	Empirical evaluation across four environments and four baseline algorithms confirms consistent improvements in convergence speed, policy stability, and energy efficiency. In practice, H-EARS provides a pathway from academic research to industrial deployment by showing that minimal physics priors can enhance model-free RL without expert modeling, complete dynamics, or algorithm changes. The systematic integration of convergence guarantees, error bounds, mechanical principles, cross-algorithm validation, and simplified modeling bridges theoretical rigor with deployment feasibility.
	
	\section{Acknowledgment}
	
	Supported by the High-Speed Autonomous Engineering Scooter Project (Revealing List Program, Science and Technology for China Platform in Xuzhou City).
	The authors thank the School of Mechanical Engineering, University of Science and Technology Beijing, and Jiangsu XCMG Construction Machinery Research Institute Co., Ltd.
	
	\appendix
	
	\subsection{Training Details on benchmark}
	\label{benchmark_training_details}
	
	Table~\ref{tab:hyperparams_standard} shows the detailed hyperparameters used in benchmark experiments.
	
	\begin{table}[htbp]
		\renewcommand{\arraystretch}{1.2}
		\captionsetup{font=footnotesize, labelfont=footnotesize}
		\caption{Detailed hyperparameters}
		\label{tab:hyperparams_standard}
		\centering
		\footnotesize
		\begin{tabular}{lc}
			\toprule
			\textbf{Hyperparameters} & \textbf{Value} \\
			\midrule
			\textit{Shared} & \\
			\quad Optimizer & Adam ($\beta_1=0.9, \beta_2=0.999$) \\
			\quad Actor learning rate & $3\mathrm{e}{-}4$ \\
			\quad Critic learning rate & $3\mathrm{e}{-}4$ \\
			\quad Discount factor ($\gamma$) & 0.99 \\
			\quad Policy update interval & 5000 \\
			\quad Target smoothing coefficient ($\tau$) & 0.005 \\
			\quad Reward scale & 1 \\
			\quad Random seed set & [12345,22345,32345,42345,52345] \\
			\midrule
			\textit{Maximum-entropy framework} & \\
			\quad Learning rate of $\alpha$ & $3\mathrm{e}{-}4$ \\
			\quad Expected entropy ($\overline{H}$) & $-\dim(\mathcal{A})$ \\
			\midrule
			\textit{Off-policy} & \\
			\quad Replay buffer size & $1 \times 10^6$ \\
			\quad Samples collected per iteration & 256 \\
			\midrule
			\textit{On-policy} & \\
			\quad Sample batch size & 2048 \\
			\quad Train batch size & 2048 \\
			\quad GAE factor ($\lambda$) & 0.95 \\
			\midrule
			\textit{H-EARS ($\alpha_{\text{task}}$, $\alpha_{\text{energy}}$, $\lambda$)} & \\
			\quad Ant-v5 & (0.005, 0.03, 0.01) \\
			\quad Hopper-v5 & (0.5, 0.001, 0.0005) \\
			\quad LunarLander-v3 & (0.5, 0.001, 0.01) \\
			\quad Humanoid-v5 & (0.1, 0.001, 0.0001) \\
			\bottomrule
		\end{tabular}
	\end{table}
	
	\subsection{H-EARS Module Design for Standard Environments}
	\label{app:standard_envs}
	
	This section details the H-EARS module design for the four standard environments in the experiments of Section IV, supporting reproducibility of the experimental results.
	
	\subsubsection{Ant-v5 Environment}
	
	Internal energy function:
	\begin{equation}
		\begin{aligned}
			\mathcal{E}_{\text{int}}(s) &= \frac{1}{2}m\sum_{i=1}^{3}v_i^2 + \frac{1}{2}I\sum_{j=1}^{3}\omega_j^2 + mgh \\
			\mathbf{v} &= s[13:16], \quad \boldsymbol{\omega} = s[16:19], \quad h = s[0]
		\end{aligned}
	\end{equation}
	
	Task potential function:
	\begin{equation}
		\Phi_{\text{task}}(s) = \alpha_{\text{task}} \cdot x, \quad x = s[1].
	\end{equation}
	
	Hybrid potential function:
	\begin{equation}
		\Phi(s) = \alpha_{\text{task}} \cdot x - \alpha_{\text{energy}}\mathcal{E}_{\text{int}}(s).
	\end{equation}
	
	Control energy:
	\begin{equation}
		\mathcal{E}(a) = \frac{1}{2}\sum_{k=1}^{8}a_k^2.
	\end{equation}
	
	Hyperparameters: $\alpha_{\text{task}}=0.005$, $\alpha_{\text{energy}}=0.03$, $\lambda=0.01$.
	
	\subsubsection{Hopper-v5 Environment}
	
	Internal energy function:
	\begin{equation}
		\begin{aligned}
			\mathcal{E}_{\text{int}}(s) &= \frac{1}{2}m(v_x^2 + v_z^2) + \frac{1}{2}I\sum_{j=1}^{3}\omega_j^2 + mgh + 0.1\sum_{i=1}^{3}\theta_i^2 \\
			\mathbf{v} &= s[5:7], \quad \boldsymbol{\omega} = s[7:10], \quad h = s[0], \quad \boldsymbol{\theta} = s[2:5]
		\end{aligned}
	\end{equation}
	where $0.1\sum\theta_i^2$ represents the posture deviation energy, encouraging upright posture.
	
	Task potential function:
	\begin{equation}
		\Phi_{\text{task}}(s) = \alpha_{\text{task}}\sqrt{\max(0, x)}, \quad x = s[1].
	\end{equation}
	
	Hybrid potential function:
	\begin{equation}
		\Phi(s) = \alpha_{\text{task}}\sqrt{\max(0, x)} - \alpha_{\text{energy}}\mathcal{E}_{\text{int}}(s).
	\end{equation}
	
	Control energy:
	\begin{equation}
		\mathcal{E}(a) = \frac{1}{2}\sum_{k=1}^{3}a_k^2.
	\end{equation}
	
	Hyperparameters: $\alpha_{\text{task}}=0.5$, $\alpha_{\text{energy}}=0.001$, $\lambda=0.0005$ (dynamic balancing task).
	
	\subsubsection{LunarLander-v3 Environment}
	
	Internal energy function:
	\begin{equation}
		\begin{aligned}
			\mathcal{E}_{\text{int}}(s) &= \frac{1}{2}m(v_x^2 + v_y^2) + \frac{1}{2}I\omega^2 + mgh \\
			\mathbf{v} &= (v_x, v_y) = s[2:4], \quad \omega = s[5], \quad h = s[1]
		\end{aligned}
	\end{equation}
	
	Task potential function:
	\begin{equation}
		\Phi_{\text{task}}(s) = -\alpha_{\text{task}}\left(\sqrt{x^2+y^2} + 0.5|\theta|\right),
	\end{equation}
	where $x=s[0]$, $y=s[1]$, $\theta=s[4]$. The negative sign ensures the potential function increases as the target is approached.
	
	Hybrid potential function:
	\begin{equation}
		\Phi(s) = -\alpha_{\text{task}}\left(\sqrt{x^2+y^2} + 0.5|\theta|\right) - \alpha_{\text{energy}}\mathcal{E}_{\text{int}}(s).
	\end{equation}
	
	Control energy:
	\begin{equation}
		\mathcal{E}(a) = \frac{1}{2}\sum_{k=1}^{2}a_k^2.
	\end{equation}
	
	Hyperparameters: $\alpha_{\text{task}}=0.5$, $\alpha_{\text{energy}}=0.001$, $\lambda=0.01$ (precision control task).
	
	\subsubsection{Humanoid-v5 Environment}
	
	Internal energy function:
	\begin{equation}
		\begin{aligned}
			\mathcal{E}_{\text{int}}(s) &= \frac{1}{2}m\sum_{i=1}^{3}v_i^2 + \frac{1}{2}I\sum_{j=1}^{3}\omega_j^2 + mgh \\
			\mathbf{v} &= s[185:188], \quad \boldsymbol{\omega} = s[188:191], \quad h = s[0]
		\end{aligned}
	\end{equation}
	
	Task potential function:
	\begin{equation}
		\Phi_{\text{task}}(s) = \alpha_{\text{task}} \cdot x, \quad x = s[1].
	\end{equation}
	
	Hybrid potential function:
	\begin{equation}
		\Phi(s) = \alpha_{\text{task}} \cdot x - \alpha_{\text{energy}}\mathcal{E}_{\text{int}}(s).
	\end{equation}
	
	Control energy:
	\begin{equation}
		\mathcal{E}(a) = \frac{1}{2}\sum_{k=1}^{N_a}a_k^2.
	\end{equation}
	
	Hyperparameters: $\alpha_{\text{task}}=0.1$, $\alpha_{\text{energy}}=0.001$, $\lambda=0.0001$.
	
	\subsection{Vehicle Simulation Environment Configuration}
	\label{app:vehicle}
	
	This section details the reward function design and H-EARS configuration for the vehicle control experiments in Section V of the main text.
	
	\subsubsection{Base Reward Function}
	
	The original reward function consists of the following components:
	\begin{equation}
		\begin{aligned}
			R_{\text{base}} &= w_{\text{coop}}R_{\text{coop}} + w_{\text{speed}}R_{\text{speed}} + w_{\text{path}}R_{\text{path}} \\
			&\quad + w_{\text{look}}R_{\text{look}} + w_{\text{head}}R_{\text{head}} + w_{\text{stab}}R_{\text{stab}} + P_{\text{term}}
		\end{aligned}
	\end{equation}
	
	Cooperation reward (MPC-RL cooperation):
	\begin{equation}
		R_{\text{coop}} = 3.0\sigma(f) - 0.5(1-\sigma(f)) + R_{\text{ref-exec}} + R_{\text{state-dep}},
	\end{equation}
	where $\sigma(f) = 1/(1+e^{-15(f-0.85)})$ and $f$ is the MPC feasibility ratio.
	
	Reference-execution consistency:
	\begin{equation}
		R_{\text{ref-exec}} = 2.0\exp\left(-2.0\sum_{i}\left|\frac{y_{\text{ref},i} - y_{\text{exec},i}}{y_{\text{max},i}}\right|\right).
	\end{equation}
	
	Speed tracking:
	\begin{equation}
		R_{\text{speed}} = \frac{v_{\text{target}}^2 - (v_x-v_{\text{target}})^2}{v_{\text{target}}^2}.
	\end{equation}
	
	Path tracking:
	\begin{equation}
		R_{\text{path}} = 
		\begin{cases}
			5.0\left(1-\frac{|e_{\text{lat}}|}{1.0}\right) & \text{if } |e_{\text{lat}}| < 1.0\text{m} \\
			-2.0(|e_{\text{lat}}|-1.0)^{1.5} & \text{otherwise}
		\end{cases}.
	\end{equation}
	
	Look-ahead point tracking:
	\begin{equation}
		R_{\text{look}} = 0.7\exp\left(-\frac{\theta_{e,1}^2}{2(0.3)^2}\right) + 0.3\exp\left(-\frac{\theta_{e,2}^2}{2(0.3)^2}\right).
	\end{equation}
	
	Heading angle:
	\begin{equation}
		R_{\text{head}} = \exp(-3.0|\psi|).
	\end{equation}
	
	Stability:
	\begin{equation}
		R_{\text{stab}} = \frac{1}{3}[\exp(-3.0|r|) + \exp(-5.0|\beta|) + \exp(-5.0\text{LTR}^2)].
	\end{equation}
	
	Weights: $w_{\text{coop}}=1.0$, $w_{\text{speed}}=1.5$, $w_{\text{path}}=1.0$, $w_{\text{look}}=1.2$, $w_{\text{head}}=1.2$, $w_{\text{stab}}=0.8$.
	
	Final training reward:
	\begin{equation}
		R_{\text{train}} = 0.05 \cdot R_{\text{base}}.
	\end{equation}
	
	\subsubsection{H-EARS Enhancement}
	
	Task potential function:
	\begin{equation}
		\begin{aligned}
			\Phi_{\text{task}}(s) &= 0.35\Phi_{\text{track}} + 0.25\Phi_{\text{stab}} \\
			&\quad + 0.15\Phi_{\text{head}} + 0.15\Phi_{\text{speed}} + 0.10\Phi_{\text{prog}} \\
			\Phi_{\text{track}} &= 10.0\exp(-0.5|v_y|^2) \\
			\Phi_{\text{stab}} &= 5.0\exp(-3.0(\beta^2 + 0.5r^2)) \\
			\Phi_{\text{head}} &= 3.0\exp(-5.0\psi^2) \\
			\Phi_{\text{speed}} &= 4.0\exp(-0.05(v_x-v_{\text{target}})^2) \\
			\Phi_{\text{prog}} &= 10.0\min(x/x_{\text{max}}, 1.0).
		\end{aligned}
	\end{equation}
	
	Internal energy function:
	\begin{equation}
		\begin{aligned}
			\mathcal{E}_{\text{int}}(s) &= \hat{E}_{\text{lin}} + \hat{E}_{\text{ang}} + E_{\text{slip}} + E_{\Delta r} + E_{\Delta v} \\
			\hat{E}_{\text{lin}} &= \frac{m(v_x^2+v_y^2)}{mv_{\text{target}}^2} \\
			\hat{E}_{\text{ang}} &= \frac{I_zr^2}{I_zr_{\text{typical}}^2} \\
			E_{\text{slip}} &= 2.0\beta^2 \\
			E_{\Delta r} &= 1.0\frac{(r-r_{\text{prev}})^2}{r_{\text{ref}}^2} \\
			E_{\Delta v} &= 0.5\left(\frac{v_x-v_{\text{ideal}}}{v_{\text{ideal}}}\right)^2.
		\end{aligned}
	\end{equation}
	
	Action regularization:
	\begin{equation}
		\begin{aligned}
			\mathcal{E}(a) &= \frac{1}{2}\|a\|^2 + P_{\text{turn}} + P_{\text{slip}} + P_{\text{change}} \\
			P_{\text{turn}} &= 
			\begin{cases}
				2.0(|\hat{r}_{\text{ref}}|-0.3)^2 & \text{if } |\hat{r}_{\text{ref}}| > 0.3 \\
				0 & \text{otherwise}
			\end{cases} \\
			P_{\text{slip}} &= 
			\begin{cases}
				3.0(|\hat{\beta}_{\text{ref}}|-0.3)^2 & \text{if } |\hat{\beta}_{\text{ref}}| > 0.3 \\
				0 & \text{otherwise}
			\end{cases} \\
			P_{\text{change}} &= \|a - a_{\text{prev}}\|^2.
		\end{aligned}
	\end{equation}
	
	H-EARS reward:
	\begin{equation}
		R_{\text{H-EARS}} = R_{\text{train}} + \gamma\Phi(s') - \Phi(s) - \lambda\mathcal{E}(a),
	\end{equation}
	where $\Phi(s) = \alpha_{\text{task}}\Phi_{\text{task}}(s) - \alpha_{\text{energy}}\mathcal{E}_{\text{int}}(s)$.
	
	Hyperparameters: $\alpha_{\text{task}}=0.45$, $\alpha_{\text{energy}}=0.35$, $\lambda=0.20$.
	
	This configuration significantly improves stability under extreme conditions through potential function guidance and regularization constraints, targeting the regularized MDP $\mathcal{M}_\lambda$ objective (Eq.~\eqref{eq:optimal_policy}) that deliberately prioritizes energy-efficient, physically consistent solutions.


\begin{thebibliography}{99}
		
		\bibitem{haarnoja2018soft}
		T. Haarnoja, A. Zhou, P. Abbeel, and S. Levine, ``Soft actor-critic: Off-policy maximum entropy deep reinforcement learning with a stochastic actor,'' in \textit{Proc. Int. Conf. Mach. Learn.}, 2018, pp. 1861--1870.
		
		\bibitem{cobbe2019quantifying}
		K. Cobbe, O. Klimov, C. Hesse, T. Kim, and J. Schulman, ``Quantifying generalization in reinforcement learning,'' in \textit{Proc. 36th Int. Conf. Mach. Learn.}, 2019, pp. 1282--1289.
		
		\bibitem{lutter2019deep}
		M. Lutter, C. Ritter, and J. Peters, ``Deep Lagrangian networks: Using physics as model prior for deep learning,'' in \textit{Proc. Int. Conf. Learn. Represent.}, 2019.
		
		\bibitem{greydanus2019hamiltonian}
		S. Greydanus, M. Dzamba, and J. Yosinski, ``Hamiltonian neural networks,'' in \textit{Proc. Adv. Neural Inf. Process. Syst.}, 2019, pp. 15379--15389.
		
		\bibitem{zhang2021learning}
		A. Zhang, R. McAllister, R. Calandra, Y. Gal, and S. Levine, ``Learning invariant representations for reinforcement learning without reconstruction,'' in \textit{Proc. Int. Conf. Learn. Represent.}, 2021.
		
		\bibitem{ng1999policy}
		A. Y. Ng, D. Harada, and S. Russell, ``Policy invariance under reward transformations: Theory and application to reward shaping,'' in \textit{Proc. 16th Int. Conf. Mach. Learn.}, 1999, pp. 278--287.
		
		\bibitem{ding2023magnetic}
		H. Y. Ding, Y. Z. Tang, Q. Wu, B. Wang, C. L. Chen, and Z. Wang, ``Magnetic field-based reward shaping for goal-conditioned reinforcement learning,'' \textit{IEEE/CAA J. Autom. Sinica}, vol. 10, no. 12, pp. 2233--2247, Dec. 2023.
		
		\bibitem{wiewiora2003principled}
		E. Wiewiora, G. Cottrell, and C. Elkan, ``Principled methods for advising reinforcement learning agents,'' in \textit{Proc. 20th Int. Conf. Mach. Learn.}, 2003, pp. 792--799.
		
		\bibitem{devlin2012dynamic}
		S. Devlin and D. Kudenko, ``Dynamic potential-based reward shaping,'' in \textit{Proc. 11th Int. Conf. Auton. Agents Multiagent Syst.}, 2012, pp. 433--440.
		
		\bibitem{zou2021reward}
		H. Zou, T. Ren, D. Yan, H. Su, and J. Zhu, ``Learning task-distribution reward shaping with meta-learning,'' in \textit{Proc. 35th AAAI Conf. Artif. Intell.}, 2021, pp. 11210--11218.
		
		\bibitem{abbeel2004apprenticeship}
		P. Abbeel and A. Y. Ng, ``Apprenticeship learning via inverse reinforcement learning,'' in \textit{Proc. 21st Int. Conf. Mach. Learn.}, 2004, pp. 1--8.
		
		\bibitem{grzes2010learning}
		M. Grzes and D. Kudenko, ``Online learning of shaping rewards in reinforcement learning,'' \textit{Neural Networks}, vol. 23, no. 4, pp. 541--550, May 2010.
		
		\bibitem{harutyunyan2015expressing}
		A. Harutyunyan, S. Devlin, P. Vrancx, and A. Nowé, ``Expressing arbitrary reward functions as potential-based advice,'' in \textit{Proc. 29th AAAI Conf. Artif. Intell.}, 2015, pp. 2652--2658.
		
		\bibitem{brys2015}
		T. Brys, A. Harutyunyan, M. E. Taylor, and A. Nowé, ``Policy transfer using reward shaping,'' in \textit{Proc. 14th Int. Conf. Auton. Agents Multiagent Syst.}, 2015, pp. 181--188.
		
		\bibitem{cranmer2020lagrangian}
		M. Cranmer, S. Greydanus, S. Hoyer, P. Battaglia, D. Spergel, and S. Ho, ``Lagrangian neural networks,'' in \textit{Proc. ICLR Workshop Deep Learn. Phys. Sci.}, 2020.
		
		\bibitem{battaglia2016interaction}
		P. Battaglia, R. Pascanu, M. Lai, and D. J. Rezende, ``Interaction networks for learning about objects, relations and physics,'' in \textit{Proc. Adv. Neural Inf. Process. Syst.}, 2016, pp. 4502--4510.
		
		\bibitem{sanchez2020learning}
		A. Sanchez-Gonzalez, J. Godwin, T. Pfaff, R. Ying, J. Leskovec, and P. Battaglia, ``Learning to simulate complex physics with graph networks,'' in \textit{Proc. Int. Conf. Mach. Learn.}, 2020, pp. 8459--8468.
		
		\bibitem{chen2018neural}
		R. T. Chen, Y. Rubanova, J. Bettencourt, and D. K. Duvenaud, ``Neural ordinary differential equations,'' in \textit{Proc. Adv. Neural Inf. Process. Syst.}, 2018, pp. 6571--6583.
		
		\bibitem{liu2020parallel}
		T. Liu, B. Tian, Y. Ai, and F.-Y. Wang, ``Parallel reinforcement learning-based energy efficiency improvement for a cyber-physical system,'' \textit{IEEE/CAA J. Autom. Sinica}, vol. 7, no. 2, pp. 617--626, Mar. 2020.
		
		\bibitem{fujimoto2018addressing}
		S. Fujimoto, H. van Hoof, and D. Meger, ``Addressing function approximation error in actor-critic methods,'' in \textit{Proc. 35th Int. Conf. Mach. Learn.}, 2018, pp. 1587--1596.
		
		\bibitem{schulman2017proximal}
		J. Schulman, F. Wolski, P. Dhariwal, A. Radford, and O. Klimov, ``Proximal policy optimization algorithms,'' \textit{arXiv preprint arXiv:1707.06347}, 2017.
		
		\bibitem{haarnoja2018applications}
		T. Haarnoja, A. Zhou, K. Hartikainen, G. Tucker, S. Ha, J. Tan, V. Kumar, H. Zhu, A. Gupta, P. Abbeel, and S. Levine, ``Soft actor-critic algorithms and applications,'' \textit{arXiv preprint arXiv:1812.05905}, 2018.
		
		\bibitem{ma2020artificial}
		Y. Ma, Z. Wang, H. Yang, and L. Yang, ``Artificial intelligence applications in the development of autonomous vehicles: A survey,'' \textit{IEEE/CAA J. Autom. Sinica}, vol. 7, no. 2, pp. 315--329, Mar. 2020.
		
		\bibitem{peng2021end}
		P. Peng, Y. Jiang, Y. Wang, and Y. Sun, ``End-to-end autonomous driving through dueling double deep Q-network,'' \textit{Automot. Innov.}, vol. 4, no. 3, pp. 328--337, Aug. 2021.
		
		\bibitem{hewing2020learning}
		L. Hewing, K. P. Wabersich, M. Menner, and M. N. Zeilinger, ``Learning-based model predictive control: Toward safe learning in control,'' \textit{Annu. Rev. Control Robot. Auton. Syst.}, vol. 3, pp. 269--296, May 2020.
		
		\bibitem{spielberg2019neural}
		Spielberg NA, Brown M, Kapania NR, Kegelman JC, Gerdes JC. ``Neural network vehicle models for high-performance automated driving,'' \textit{Sci Robot}, 2019 Mar 27;4(28):eaaw1975.
		
		\bibitem{rosolia2017learning}
		U. Rosolia and F. Borrelli, ``Learning model predictive control for iterative tasks. A data-driven control framework,'' \textit{IEEE Trans. Autom. Control}, vol. 63, no. 7, pp. 1883--1896, Jul. 2018.
		
		\bibitem{liniger2015optimization}
		A. Liniger, A. Domahidi, and M. Morari, ``Optimization-based autonomous racing of 1:43 scale RC cars,'' \textit{Optim. Control Appl. Methods}, vol. 36, no. 5, pp. 628--647, 2015.
		
		\bibitem{kiran2022deep}
		B. R. Kiran et al., ``Deep reinforcement learning for autonomous driving: A survey,'' \textit{IEEE Trans. Intell. Transp. Syst.}, vol. 23, no. 6, pp. 4909--4926, Jun. 2022.
		
		\bibitem{henderson2018deep} P. Henderson, R. Islam, P. Bachman, J. Pineau, D. Precup, and D. Meger, ``Deep reinforcement learning that matters,'' in \textit{Proc. 32nd AAAI Conf. Artif. Intell.}, 2018, pp. 3207--3214.
		
	\end{thebibliography}
\end{document}